\documentclass[11pt]{article}
\usepackage[utf8]{inputenc}
\usepackage{mystyle}

\def\algo{\texttt{PLAN}}

\begin{document}

\title{\huge Strategic Decision-Making in the Presence of Information Asymmetry: Provably Efficient RL with Algorithmic Instruments}

\author{Mengxin Yu \qquad\qquad\qquad Zhuoran Yang \qquad\qquad\qquad Jianqing Fan\thanks{Email:\texttt{\{mengxiny,jqfan\}@princeton.edu, zhuoran.yang@yale.edu}. This work is supported by NSF grants DMS-2210833, DMS-2053832, DMS-2052926 and ONR grant N00014-22-1-2340.}\\
Princeton University \qquad\qquad Yale University \qquad\qquad Princeton University}


\maketitle

\begin{abstract}
We study offline  reinforcement learning under  a novel model called strategic MDP, 
which characterizes the strategic interactions between a principal and a sequence of myopic  agents with private types. 
Due to the bilevel structure and private types, strategic MDP involves information asymmetry between the principal and the agents. 
We focus on the offline RL problem where 
the goal is to 
learn the optimal policy of the principal concerning a target population of agents, based on a pre-collected dataset that consists of historical interactions. 
The unobserved private types confound such a dataset as they affect both the rewards and observations received by the principal. 
We propose  a novel algorithm,  \underbar{p}essimistic policy \underbar{l}earning with \underbar{a}lgorithmic  i\underbar{n}struments  (\algo), which leverages the ideas of instrumental variable regression and  the pessimism principle to learn a near-optimal principal's  policy in the context of general function approximation.  
Our algorithm is based on the critical observation that the principal's actions serve as valid instrumental variables. 
In particular, under a  partial coverage assumption on the offline dataset, we prove that \algo~ outputs a $1 / \sqrt{ K}$-optimal policy with $K$ being the number of collected trajectories.
We further apply our framework to some special cases of strategic MDP, including strategic regression \citep{stevenwu2021}, strategic bandit, and noncompliance in recommendation systems \citep{robins1998correction}.

\end{abstract}

\section{Introduction}
In multi-agent  decision-making systems \citep{ferber1999multi},
at each step,  each agent chooses an action based on its local  information gathered so far, 
where the local information contains both the public information that is shared by every agent, 
and the private information that is only  known to itself. 
The actions taken by all the agents then determine the  state of the underlying environment  and the observations received by each agent at the next~step.  
The goal of each agent is to maximize its own expected  cumulative rewards 
via taking a sequence of actions that leverage her private information, in the presence of other agents.

Multi-agent systems with private information find wide applications in economics \citep{laffont2009theory}, social science \citep{sabater2005review}, robotics \citep{farinelli2004multirobot, yan2013survey},  and cognitive science \citep{sun2006cognition}. 
The private information of the agent 
often  include hidden states that represent  the agent's own type information  or  part of information about the environment that other agents are ignorant of,  and hidden actions taken or local observation received by the agent that are not observed by others.

In this work, we aim to study reinforcement learning in the presence of private information. Motivated by the principal-agent framework in contract theory \citep{bolton2004contract, laffont2009theory}, we propose a stylized bilevel 
multi-agent decision-making 
model where a principal interacts with a sequence of myopic agents with private information.   
Such a model specifies 
an information asymmetry 
--- the agents have private information unknown to the principal, while the principal's policy is known to the agents. 
In specific, 
such a model can be viewed as a bilevel Stackelberg game \citep{bacsar1998dynamic} where the principal first announces her policy, and the agents then choose their best-response actions --- actions that maximize their immediate rewards --- based on their private information. 
Here the private information includes both the private types and actions of the agent, which are unknown to the principal.
The type of agent reflects its personal preference and is part of the agent's reward function. 
In contrast, the principal aims to maximize  its expected cumulative rewards, assuming the agents always adopt the best response policies. 
Such a model is called the strategic Markov decision process (MDP), as it models the interaction between the principal and a sequence of agents that strategically respond to the principal.

In strategic MDP,  the private types of agents influence the game structure in two ways. 
First, the actions taken by the agents implicitly depend on the private types. 
While the principal does not observe these actions,  the agents generate available observations which affect the principal's immediate rewards.
Second,   the state transitions depend on both the actions of the principal and the private types of agents. 
As a result, the optimal policy of the principal also implicitly depends on the private types of the agents. 
In other words, the principal should adopt different policies targeting at  different types of agents.


Strategic MDP is motivated by the notion of performative prediction \citep{perdomo2020performative}, which refers to the setting where the prediction given by the classifier causes a distributional change in the targeted variables.
Such a performative distributional shift is ubiquitous in strategic machine learning problems where the agents strategically modify the data distribution in response to the announced machine learning model, in order to improve their outcomes. 
As a concrete example, consider the scenario of the college application, the requirements on 
the scores of the standardized tests (e.g., SAT) incentivize the 
applicants to make additional efforts (e.g., taking multiple tests) to improve their scores \citep{goodman2020take}. 
Strategic MDP captures such a strategic/performative setting by assuming the agents strategically choose best-response actions to maximize their rewards, while these actions, in turn, change the distribution of the observations. 
More importantly, strategic MDP 
brings strategic interactions into the context of sequential decision-making by introducing Markovian state transitions, which helps model more complex strategic behaviors involving dynamic structures. 

In this paper, we study the  
offline RL  problem \citep{levine2020offline}  in strategic MDP, 
where the goal is to learn the  optimal policy of  the 
principal based on an offline dataset collected a priori. Specifically, since the principal's optimal policy depends on the agents' private types, we consider learning the optimal policy for a targeted population of agents, where a population refers to a distribution over the private type. 
Moreover, the offline dataset is generated by interacting with an arbitrary population of agents with a  possibly suboptimal behavior policy and only consists of the variables available to the principal. 
In other words, the agents' private types and actions are not recorded in the offline dataset. 

Compared with the standard  MDP model, offline RL in strategic MDP involves the confounding issue brought by the private types, which is absent in MDP, and a more involved challenge of distributional shift. 
Specifically, the private actions taken by the agents affect the state transitions and the principal's observations and rewards stored in the dataset. 
Moreover, these private actions are determined implicitly by the private types of the agents from the best-response policies. 
As a result,  the private types are unobserved confounders and thus  
directly applying standard offline RL methods to learn the transition model and reward functions would incur a considerable bias. 
Furthermore, the offline dataset involves two kinds of distributional shifts. 
First, the dataset is collected by interacting with some population of agents while we are interested in the optimal policy concerning the target population. 
Second,  in data generation, the principal's policy is fixed to some behavior policy which can be very different from our desired optimal policy. 
As a result, any successful learning algorithm needs to handle the challenges due to the offline dataset's unobserved confounder and distributional shift.

To this end, we propose a novel algorithm,  \underbar{p}essimistic policy \underbar{l}earning with \underbar{a}lgorithmic  i\underbar{n}struments  (\algo), which resolves the above two challenges by leveraging the ideas of instrumental variable regression and the pessimism principle \citep{buckman2020importance, jin2021pessimism}. 
Specifically, although the private types are absent in the offline dataset, we prove that the actions taken by the principal serve as instrumental variables that help identify their causal effect.
That is,  
both the reward function and transition kernel can be written as the solutions to some conditional moment equations given the states and actions of the principal. 
Intuitively, the principal's actions directly affect how the agents take the best-response actions in each step. 
Moreover, they indirectly affect the observations presented to the principals via the agents' actions. 
Thus, the principal's actions serve as instrumental variables.  
Such a fundamental property enables us to construct a loss function based on the offline data  via minimax estimation, where the global minimizer of the population loss corresponds to the true model.
Meanwhile, the construction of the minimax loss function can readily incorporate general function approximators.  
Furthermore, to handle the distributional shift, we follow the pessimism principle by (i) constructing a high-confidence region containing the true model based on the level sets of loss functions and (ii) returning the optimal policy of the most pessimistic candidate model in the confidence region. 
Under proper assumptions, we prove that \algo~outputs a  $1 / \sqrt{ K}$-optimal policy when the offline dataset satisfies a mild partial coverage assumption, where $K$ is the number of trajectories in the offline dataset.
We also instantiate our results to specific function classes, including kernel and neural network functions. 
Besides, as concrete examples,  we apply the results for strategic MDP to particular cases, including strategic regression \citep{stevenwu2021}, strategic bandit, and noncompliance in recommendation systems \citep{robins1998correction}.

\subsection*{Main Contributions} Our contribution is several-fold. First, we propose a general framework named strategic MDP for modeling the interaction betweens a principal and a sequence of myopic agents with private information. 
Such a model captures many strategic decision-making problems as special cases, such as strategic regression, strategic bandit, Stackelberg game, and etc. 
Second, for offline RL in strategic MDP, we propose a novel algorithm, \algo, which leverages instrumental variable regression to handle the confounding issue caused by the agents' private information and adopts the idea of pessimism to handle the distributional shift of the offline data. 
Specifically, the instruments are the  state variables and  the actions of the principal sampled from the  behavior policy during data collection. 
Third,  under mild assumptions on the offline dataset and the function classes employed in statistical estimation,  we prove that \algo~outputs a near optimal policy with statistical accuracy. Meanwhile, we also apply the theoretical results to special cases of strategic MDP, including strategic regression and strategic bandits, to showcase the efficacy of \algo.  

\subsection{Related Work}

{\noindent \bf Pessimism in offline RL.} 
There exists a large body of literature on offline RL. Assuming the offline dataset has sufficient coverage over all target policies, many existing works establish the statistical rates of convergence for various offline RL methods with function approximation. See, e.g., \citep{antos2007fitted, munos2008finite,antos2008learning, farahmand2010error,farahmand2016regularized, busoniu2017reinforcement, fan2019theoretical, chen2019information, duan2021risk} and the references therein. 
Our work mainly builds upon the recent line of works that develop offline RL algorithms based on the pessimism principle. 
Leveraging pessimism, these works prove that various RL  algorithms are provably efficient when the  offline  datasets merely  satisfy a partial coverage condition  \citep{liu2020provably, kidambi2020morel, yu2020mopo, buckman2020importance, xie2021bellman, jin2021pessimism, uehara2021pessimistic, shi2022pessimistic, rashidinejad2021bridging,zanette2021provable,  li2022settling, yan2022efficacy, lu2022pessimism}. 
Among these contributions, our work is particularly related to \cite{lu2022pessimism}, 
which studies offline RL in partially observable MDPs, where the offline dataset is collected by a behavior policy that has access to the latent states. 
In \cite{lu2022pessimism}, the observations and actions in the dataset are confounded by the latent states. 
To handle the confounding issue,  \cite{lu2022pessimism} assume the existence of confounding bridge functions and leverage proximal causal inference  \citep{lipsitch2010negative, tchetgen2020introduction} to construct pessimistic estimates of the bridge functions. 
In contrast, strategic MDP is a different problem, and our \algo~algorithm is based on   instrumental variable estimation. Besides,  we consider general function approximators with kernel and neural networks as special cases. Whereas \cite{lu2022pessimism} only focuses on the   linear setting. 

{\noindent \bf RL with confounders.} Our work is also related to the works that propose RL  methods in the presence of confounded data. 
In particular, 
a line of research studies RL  methods for 
dynamic treatment regimes (DTR)  \citep{chakraborty2014dynamic}  under the  sequential ignorability condition  \citep{murphy2001marginal, robins2004optimal, schulte2014q}  ensures that all confounders are measured. 
When there exist unmeasured confounders, 
\cite{chen2021estimating} 
utilize a time-varying instrumental variable to establish a novel estimation framework for DTR. 
Compared with   \cite{chen2021estimating}, although our work also uses instrumental variables to address the confounding issue, they  study  a different model and thus are  not directly comparable. 
In addition, there is also  a line of research on  learning the optimal policies of various confounded MDP models 
\citep{zhang2016markov, liao2021instrumental,lu2021causal,  wang2021provably}. 
Among these works, our work is most relevant to the work of  \cite{liao2021instrumental}, which proposes a value iteration algorithm that leverages instrumental variable regression to learn the value functions. 
In particular, in terms of value function estimation,  the work of  \cite{liao2021instrumental} focuses on the setting of linear function approximation and assumes the dataset has uniform coverage. In contrast, our work considers general function classes and only assumes a partial coverage of the offline dataset. 
Furthermore, a growing body of literature leverages causal inference tools for the  
off-policy evaluation (OPE) problem in partially observable MDPs. See, e.g., 
\cite{tennenholtz2020off, kallus2020confounding, namkoong2020off, shi2021minimax, chen2021instrumental,  bennett2021off, bennett2021proximal, hu2021off, chen2022well} and the references therein.
These works aim to evaluate a history-dependent policy where the offline dataset is collected by a behavior  policy that has access to the latent states. Thus the observations and actions in the dataset are confounded.  
In comparison, we 
focus on learning the optimal policy, which involves evaluating each policy within the policy class as a special case. 
The fact that the offline dataset has merely partial coverage motivates our pessimism  approach, whereas pessimism seems unnecessary in these works on OPE.

\subsection{Notation}
We introduce some useful notation before proceeding. We let $[N] $ denote $ \{1,2,\dots, N \}$ for any positive integer $N$. For any random variable $X\sim \PP$ and $q\ge 0$, we use $\|X\|_{q}$ to represent the  $\ell_q$-norm of $X$ under a probability measure $\PP$, i.e., $\|X\|_{q}=(\EE_{\PP}[X^q])^{1/q}.$ In addition, we let $ \|X\|_{q,K}$ be the sample version of $\|X\|_{q},$ namely, $\|X\|_{q,K}=(1/K\sum_{i=1}^{K}X_i^q)^{1/q}$. For two positive sequences $\{a_n\}_{n\ge 1}$, $\{b_n\}_{n\ge 1}$, we write $a_n=\cO(b_n)$ or $a_n\lesssim b_n$ if there exists a positive constant $C$ such that $a_n\le C\cdot b_n$ and we write $a_n=o(b_n)$ if $a_n/b_n\rightarrow 0$. In addition, we write $a_n=\Omega(b_n)$ or $a_n\gtrsim b_n$ if $a_n/b_n\ge c$ with some constant $c>0$. We use $a_n=\Theta(b_n)$ if $a_n=\cO(b_n)$ and $a_n=\Omega(b_n)$. If a function class $\cF$ is star-shaped, it holds that $\forall f\in \cF$, $\alpha f\in \cF,\forall \alpha\in[0,1].$ Moreover, a function class $\cF$ is assumed to be symmetric when $\forall f\in\cF, $ we have $-f\in\cF.$ The star hull of a function class $\XX$ is defined as the minimal star-shaped function class that contains $\XX$. 

\subsection{Roadmap}
The rest of this paper is organized as follows. We provide some preliminary knowledge 
 in \S\ref{sec:pre}. We then describe the problem setting in  \S\ref{sec:algo} and  introduce  the \algo{} algorithm   in \S\ref{sec:method}.  
 In addition, in \S\ref{sec:theory}, we establish the theoretical guarantees of \algo{} and prove that \algo~is provably efficient.  Finally, we discuss related applications of the strategic MDP model in \S\ref{application}.  

 \section{Preliminaries}\label{sec:pre}
 In this section, we present some preliminary background to set up a theme and to introduce some notation. Specifically, we introduce the episodic Markov decision process and some basic concepts on empirical process in   \S\ref{def:mdp} and \S\ref{intro_local_rademacher}, respectively.
 
 \subsection{Episodic Markov Decision Process (MDP)}\label{def:mdp}
 
 An episodic MDP $(\cS,\cA,H,\PP,r)$ consists of a  state space $\cS$, an  action space $\cA$, a  time horizon $H$, a transition kernel $\PP=\{\PP_h^*\}_{h=1}^{H}$ and a reward function $r=\{R_h^*\}_{h=1}^{H}.$ 
 
 Without loss of generality, we consider bounded reward function with $R_h^*\in[0,1]$ for all $h\in [H]$. For any policy $\pi=\{\pi_h:\cS\rightarrow\Delta(\cA)\}_{h=1}^{H},$  we define the $Q$-function  (action-value function) $Q_h^\pi \colon \cS \times \cA \rightarrow \RR$  and (state-)value function $V_h^{\pi}:\cS\rightarrow \RR$ at each step as follows:
 \begin{align*}
    {Q}_{h}^{\pi}(s,a)&=\EE_{\pi}\Bigg[ \sum_{i=h}^{H}{R}_i^*(s_i,a_i)\given s_i=s,a_i=a\Bigg ],\\ 
 	V_h^{\pi}(s)&=\EE_{\pi}\Bigg[\sum_{i=h}^{H}R_i^*(s_i,a_i)\given s_h=s\Bigg].
 	\end{align*}
 Here the expectation is taken with respective to the randomness induced by $\pi$ and the transition kernel $\PP^*$, which is obtained by taking $a_h\sim \pi_h (\cdot \given s_h)$ at state $s_h$ and obtaining the next state $s_{h+1}\sim \PP^*_h(\cdot\given s_h,a_h)$, $\forall h \in [H]$. Suppose the initial state $s_1$ is sampled from a distribution $\rho_0$, the optimal policy for our cumulative reward is defined as 
 \begin{align*}
 	\pi^*:=\argmax_{\pi\in \Pi}\EE_{s\sim\rho_0}[V_1^{\pi}(s)],
 	\end{align*}
 where $\Pi$ is a class of policies such that $\Pi:=\{\pi=(\pi_1,\cdots,\pi_H),\pi_h:\cS\rightarrow\Delta(\cA),\forall h\in [H]\}.$ 
 

 \subsection{Local Rademacher Complexity and Critical Radius}  \label{intro_local_rademacher}
 
 The (empirical) local Rademacher complexity quantifies the complexity of a bounded function class locally around the ground truth with a given radius \citep{bartlett2005local,wainwright_2019}.
 \begin{definition}[(Empirical) Local Rademacher Complexity]
 The local Rademacher complexity of  a function class $\cF:V\rightarrow [-L,L]$ and its empirical counterpart  are defined respectively as follows:
\begin{align*}
\cR_K(\delta,\cF)&=\EE_{\epsilon,\cV}\bigg[\sup_{f\in \cF,\|f\|_2\le \delta}\Big|\frac{1}{K}\sum_{i=1}^{K}\epsilon_if(v_i)\Big|\bigg],\\
\cR_{K,s}(\delta,\cF)&=\EE_{\epsilon}\bigg[\sup_{f\in \cF,\|f\|_{2,K}\le \delta}\Big|\frac{1}{K}\sum_{i=1}^{K}\epsilon_if(v_i)\Big|\bigg],
\end{align*}
where $\{v_i\}_{i=1}^{K}$ are i.i.d. observations sampled  from some distribution  $\cV,$ and $\{\epsilon_i\}_{i=1}^{K}$ are  i.i.d. Rademacher random variables taking values  in $\{-1,1\}$ with equal probability. Here $\EE_{\epsilon,\cV}$ means the   expectation is taken with respect to   both random variables $\{\epsilon_i\}_{i=1}^K$ and $\{v_i\}_{i=1}^K,$ whereas $\EE_{\epsilon}$ means  the expectation is only taken with respect to $\{\epsilon_i\}_{i=1}^K$.
 \end{definition}

Based on the definition of local Rademacher complexity, we introduce the critical radius of a function class as follows.
\begin{definition}[Critical Radius]
The  critical radius $\delta_K$  
of a function class $\cF:V\rightarrow [-L,L]$ is defined as the minimal  $\delta$ that satisfies  
\begin{align}\
    \cR_K(\delta,\cF)&\le \delta^2/L.\label{critical_point}
\end{align}
\end{definition}
If the function class $\cF$ is star-shaped, then $\cR_K(\delta,\cF)/\delta$ is a decreasing function over $\delta.$ Thus, the solution to $\cR_K(\delta,\cF)\le \delta^2/L$ always exists \citep{bartlett2005local,wainwright_2019}. 

\section{Problem Setting}
\label{sec:algo}
In this section, we introduce the model formulation, data collection process and the learning objectives  in our problem setting.
\subsection{Model Formulation}\label{model}


We consider a framework that models the interaction between a principal and a sequence of
agents. Given the principal’s action, the agents are myopically rational and strategic.  Specifically, they always
maximize their immediate rewards based on their private types and generate observations for the principal. 
The principal then collects an immediate reward and future state through the interaction
with these agents. This results in a strategic MDP  as follows.
\begin{definition}[Strategic MDP]\label{strategic_mdp}
The strategic MDP model is formulated as follows.
\begin{itemize}
    \item Principal's action: Given the environmental state $s_h$, the principal takes an action $a_h$ which may depend on the whole history instead of only $s_h$. Without loss of generality, we assume the initial state is generated from a fixed distribution $\rho_0$.
    \item Agent's action: 
      For any stage $h$, an agent comes to system whose  private type $i_h\sim \PP_h(\cdot)$ is sampled independently from an unknown distribution $\PP_h(\cdot)$. 
    The agent then takes the action 
    $$b_h:=\argmax_{b}R_{ah}^*(s_h,a_h,i_h,b)$$
   to maximize its immediate private reward $R_{ah}^*(\cdot)$ that depends on the current state $s_h,$ principal's action $a_h,$ its private type $i_h$ and available actions.
    The principal can not observe the private type $i_h$ and action $b_h$. 
    
    \item {Observation (manipulated feature):
   The principal receives an observation $o_h$ sampled from an observation channel $F_h$, namely,  $$o_h\sim F_h(\cdot\given s_h,a_h,i_h):=F_{ah}(\cdot\given s_h,i_h,b_h). 
    $$
    Intuitively, here $o_h$ is regarded as part of the feature manipulated by the strategic agent based on the current state $s_h$, its own private type $i_h$, and principal's action $a_h$ (through $b_h$).    
    }
   
    \item Reward: 
     Principal receives an immediate reward $r_h=R_h^{*}(s_h,a_h,o_h)+g_h,$ where $g_h$ is an unobserved noise correlated with the agent's private type $i_h$. For simplicity, we assume $g_h=f_{1h}(i_h)+\epsilon_h$ is a zero mean, subgaussian random variable. Here $\epsilon_h$ is assumed to be a random variable independent of all other variables.
    \item Transition: The environment then transits to the next state, $s_{h+1}\sim G_h^*(s_h,a_h,o_h)+\xi_h.$ Here  $\xi_h$ is also assumed to be correlated with $i_h$. Without loss of generality, we let $\xi_h=f_{2h}(i_h)+\eta_h.$ Here $\eta_h$ is a Gaussian 
    random variable $N(0,\sigma^2\II)$ with an known covariance $\sigma^2\II$ and is independent of  all other random variables. In short, we have
    \begin{align}
	s_{h+1}&\sim\PP_{h}^*(\cdot \given s_h,a_h,o_h,i_h)\sim N(G_h^{*}(s_h,a_h,o_h)+f_{2h}(i_h),\sigma^2\II).\label{transition_new}
\end{align}
For the simplicity of our notation, we use $\PP_{h}^*(\cdot \given s_h,a_h,o_h,i_h)$ to represent this transition kernel.
\end{itemize}
\end{definition}
 To summarize, strategic MDP models the interactions between a principal and a sequence of strategic agents with private information. Given the principal's action $a_h$ and current state variable $s_h$, the agent then strategically manipulates principal's observation $o_h$ through its action $b_h$ and private type $i_h$. These factors then affect the reward received by the principal and also lead to the next state of the system. It is worth noting that this framework can be viewed as a version of the bilevel Stackelberg  game \citep{bacsar1998dynamic}, where the principal first takes an action and the agent then chooses its action that maximizes its reward  based on the principal's policy and its private type. Especially, when there does not exist strategic agents (i.e., the agents report their actions $o_h=b_h$ that is independent with their private type), this formulation reduces to standard Stackelberg stochastic game \citep{bacsar1998dynamic}. Moreover, our model also generalizes recently studies on strategic classification and regression \citep{miller2020strategic,stevenwu2021,shavit2020causal}, where the principal only interacts with the strategic agents in one round with no state transition. 
\begin{remark}\label{remark:embedding}
{Here we consider the most general case where the reward and transition functions depend on all variables $(s_h,a_h,o_h).$  Generally speaking, throughout this paper, for all $h\in[H],$ these functions $R_h^*(\cdot),G_{h,j}^*(\cdot),\forall j\in [d_1]$ map an embedding of the space of $(s_h,a_h,o_h)$ to the Euclidean space, in the sense that  $R_h^*(\cdot),G_{h,j}^*(\cdot):\phi_{x}(s_h,a_h,o_h)\in \RR^{d}\rightarrow \cY\in \RR$ with some known $\phi_{x}(\cdot)$.} In the following, for simplicity of our notation, we omit the embedding function $\phi_{x}(\cdot)$ and use $(s_h,a_h,o_h)$ to represent $\phi_{x}(s_h,a_h,o_h)$ without specification.
\end{remark}
Next, we use a real-world example of the existence of noncompliant agents in recommendation systems \citep{robins1998correction,ngo2021incentivizing} to digest this setting .
\begin{example}[Noncompliant Agents in Recommendation Systems]\label{example_noncompliant}
	For every stage $h$, the principal obtains the state variable $s_h,$ which represents the current environment condition (e.g., inventory level, the operation condition of the company). 
	A new agent comes to the system, and the principal recommends $a_h$ (recommends some new products) based on the current state variable. 
	However, the agent is noncompliant and may deny recommendation $a_h$ and choose another action $b_h$ (such as buying some other products) that maximizes the agent's reward according to its private preference $i_h$. The principal then observes the true action $o_h=b_h$ and receives the revenue  $r_h=R_h^*(s_h,b_h)+g_h$ from this purchase.
	The system finally transits to the next state $s_{h+1}\sim G_h^*(s_h,b_h)+\xi_h$.
	\end{example}

This example serves as a particular case of our strategic MDP, namely, the principal observes the actual action $o_h=b_h$ the strategic agent chooses. However, in some situations, such as a clinical trial, the noncompliant agent may not report the real action $b_h$ but the manipulated observation $o_h$ instead \citep{pearl2009causality}. This is also captured by our model.




  In the offline RL problem in strategic MDP, the goal is to learn the  optimal policy of the principal through some historical (offline) data in the face of a sequence of strategic agents from a certain population who always take their best responses. 
 Thus, in the following subsections, we describe the data generation process and introduce the learning objectives, respectively. 
 \subsection{Offline Data Collection}\label{offline_data}
In this subsection, we describe the offline setting for the data collection process.
We sample $K$ trajectories $\{(s_{h}^{(k)},a_h^{(k)},o_h^{(k)},r_h^{(k)})\}_{k=1,h=1}^{K,H}$ independently, in which every trajectory $\{(a_h,s_h,o_h)\}_{h=1}^{H}$ is sampled from a joint distribution $\rho: \{(\cS_h\times \cA_h\times\cO_h)\}_{h=1}^{H}\rightarrow \RR.$ The dynamics are the same with those in Definition \ref{strategic_mdp} and we summarize it as follows:
\begin{align}
a_h&\sim \pi_h, \textrm{ a behavior policy that may depend on past information } \{s_i,a_i,o_i\}_{i=1}^{h-1}\cup \{s_h\},\nonumber \\i_h&\sim P_{h}(\cdot),\textrm{ an agent with private type $i_h$ comes into the system,  }\nonumber\\
o_h&\sim F_{h}(\cdot\given s_h,a_h,i_h), \textrm{ with } b_h= \argmax_{b}R^*_{ah}(s_h,a_h,i_h,b), \label{observation}\\
    r_{h}&\sim R_h^*(s_h,a_h,o_h)+g_h, \textrm{ with }g_h=f_{1h}(i_h)+\epsilon_h, \label{reward_confounder}\\
        s_{h+1}&\sim G_{h}^{*}(s_h,a_h,o_h) +\xi_h, \textrm{ with } \xi_h=f_{2h}(i_h)+\eta_h, \forall{h}\in [H].  \label{tran_confounder}
\end{align}

{Here, we assume that distributions $F_{h}(\cdot\given s_h,a_h,i_h),R_{ah}^*(\cdot),P_h(\cdot)$ and functions $f_{1h}(\cdot),f_{2h}(\cdot)$  in Definition \ref{strategic_mdp} are  \emph{unknown} in the data collection process. Moreover, these functions are \emph{not necessarily} the same as those in our planning stage discussed in \S\ref{planning} below. Furthermore, we highlight that during the data collection process, the action $a_h$ may depend on all past information (i.e., following some non-Markovian policies) instead of only relying on $s_h$ for all $h\in [H]$.}

  \subsection{Evaluation (Planning): Personalized Policy Optimization}\label{planning}
 In this subsection, we discuss the planning process for our strategic MDP, where
 we learn the optimal policy of the principal, targeting at a specific population of agents.  
 
 More specifically, in the planning stage, we assume the distribution of private type $i_h\sim P_h(\cdot)$, private reward functions $R_{ah}^*(\cdot)$ and $f_{1h}(\cdot), f_{2h}(\cdot), o_h\sim F_{h}(\cdot\given s_h,a_h,i_h),\forall h\in [H]$ are known to the principal.  The motivation behind this is that the aforementioned quantities affect the optimal policy but are never observed through the data. Therefore, one only hopes to return policy for a target population and take these terms related to agents' private types as input. 
  Quantitatively, this is equivalent to learning the optimal policy for the aggregated MDP, which marginalizes the effect of agents. To be clear, we next define a new MDP $(\cS,\cA,H,\{\bar{\PP}_h^*\}_{h=1}^H,\{\bar{R}\}_{h=1}^H)$ which takes the distribution of a specific population of agents into account in the planning stage.

We define  a new (marginalized) true reward function $\bar{R}_h^*:\cS\times\cA\rightarrow \RR,h\in[H]$ and transition kernel $\bar{\PP}_h^*:\cS\rightarrow \RR,h\in[H]$ in the $h$-th step as follows:
\begin{align}
    \bar{R}^*_h(s_h,a_h)&=\int_{o_h,i_h}\Big[R_h^*(s_h,a_h,o_h)+f_{1h}(i_h)\Big]\ud F_h(o_h\given s_h,a_h,i_h)\ud P_h(i_h).\label{margin_rew}\\
         \bar{\PP}^*_h(\cdot\given s_h,a_h)&=\int_{o_h,i_h} \PP_h^*(\cdot\given s_h,a_h,o_h,i_h)\ud F_h(o_h\given s_h,a_h,i_h)\ud P_h(i_h),\label{transit_new2}
\end{align}
where $\PP_h^*(\cdot\given s_h,a_h,o_h,i_h)$ is given in \eqref{transition_new}. Finally, we have the actual underlying model of our aggregated MDP as $\{\bar{R}_h^*(s_h,a_h),\bar{\PP}^*_h(\cdot\given s_h,a_h)\}_{h=1}^{H},$
in the planning stage.

 For any given $h\in[H]$, we next define its associated $Q$-function $\bar{Q}_h^{\pi}:\cS\times\cA\rightarrow \RR$ and value function $\bar{V}_h^{\pi}:\cS\rightarrow \RR$ under any given policy $\pi$ as follows:
 \begin{align*}
    \bar{Q}_{h}^{\pi}(s,a)&=\EE_{\pi}\Bigg[ \sum_{j=h}^{H}\bar{R}_j^*(s_j,a_j)\given s_h=s,a_h=a\Bigg ],\\    \bar{V}_h^{\pi}(s)&=\EE_{\pi}\Bigg[ \sum_{j=h}^{H}\bar{R}_j^*(s_j,a_j)\given s_h=s\Bigg ].
\end{align*}
Here the expectation is taken with respect to the randomness of the state-action sequence $\{s_i,a_i\}_{i=h}^{H}$ with $\{s_i,a_i\}_{i=h}^{H}$ following the dynamics induced by $\pi$ and transition kernel  $\bar{\PP}_h^*(\cdot \given s_h,a_h)$.  
The associated Bellman equation is
\begin{align}
    \bar{Q}_{h}^{\pi}(s,a)&=\bar{R}_h^*(s_h,a_h)+\EE_{s_{h+1}\sim \bar{P}_{h}^*(\cdot \given s_h,a_h)}\bar{V}_{h+1}^{\pi}(s_{h+1}),\nonumber\\
    \bar{V}_h^{\pi}(s)&=\langle \bar{Q}_{h}^{\pi}(s,\cdot), \pi_h(\cdot \given s)\rangle, \bar{V}_{H+1}^{\pi}(\cdot)=0. \label{Bellman_eq}
\end{align}

It is worth noting that we only need to study the Markovian policy class $\Pi:=\{\pi=(\pi_1,\cdots,\pi_H), \pi_h:\cS\rightarrow \Delta(\cA),\forall h\in [H]\}$ in the planning stage thanks to the Markov property. 

 Recall that we assume the initial state is generated from a fixed distribution $\rho_0$ in Definition \ref{strategic_mdp}.
 Then, for any given policy $\pi\in \Pi$, the expected total rewards of the principal under the true model $M^*:=\{(\bar{R}_h^*,\bar{\PP}_h^*)\}_{h=1}^{H}$ is denoted by $J(M^*,\pi)$, with $J(M^*,\pi):=\EE_{s\sim \rho_0}[\bar{V}^{\pi}_1(s)].$  
The optimal policy is given by $\pi^*=\argmax_{\pi\in \Pi}J(M^*,\pi).$ 
 Recall that when the agents come from a specific population, learning the optimal policy of the aforementioned aggregated MDP is equivalent to studying the principal's planning problem with $J(M^*,\pi)$ being the principal's total rewards.
This is also equivalent to minimizing the suboptimality, which quantifies the loss we get by implementing our policy $\pi$ versus implementing the optimal policy $\pi^*$, namely,
\begin{align}\label{sub_opt}
   \textrm{SubOpt}(\pi)=J(M^*,\pi^*)-J(M^*,\pi).
\end{align}
In the next section, we provide a detailed algorithm for constructing a policy $\pi$ that optimizes the suboptimality based on the collected offline data.

\section{The \algo\,\,Algorithm}\label{sec:method}
This section introduces the algorithm to optimize the policy for our strategic MDP using pre-collected datasets and a model-based algorithm. We first provide a detailed explanation of the challenges of our model in \S\ref{analysis_model}. We then design our algorithm that tackles these challenges for strategic MDP \S\ref{alg_well} and \S\ref{pessimism_alg}. 

\subsection{A Peek into Strategic MDP: Why  Challenging?}\label{analysis_model}
Before proceeding to analyze the challenges in learning the strategic MDP, we first rigorously define the notion of confounders. 

\begin{definition}[Confounders] \label{def_conf}A random variable $u$ is a confounder with respective to $(X,Y)$ if both of $(X,Y)$ are caused by $u$.
\end{definition}
The first challenge in studying strategic MDP is the existence of unobserved confounders. We observe that $i_h$ affects both principal's observed feature $o_h$, immediate reward $r_h$ and future state $s_{h+1}$ for all $h\in[H]$ according to Definition \ref{strategic_mdp}. Therefore, for any $h\in [H],$ we see that the private information $i_h$ serves as an unobserved confounding variable to $(r_h,o_h)$ and $(s_{h+1},o_h)$ by Definition \ref{def_conf}. 
In the scenario, we have $\EE[g_h\given s_h,a_h,o_h]$ and $\EE[\xi_h\given s_h,a_h,o_h]\neq 0, $ and we fail to identify the true reward function $R_h^*(\cdot)$ or transition function $G_h^*(\cdot)$ via well-used square loss. 


 There exist a series of works which study offline RL using the model-based method. They first learn the model (reward functions and transition kernels) and then optimize their policy \citep{azar2017minimax,agarwal2020model,uehara2021pessimistic,zanette2021provable,li2022settling}. 
 Thus, a tempting way to solving  our strategic MDP is to apply standard model-based RL techniques by treating $(o_h,s_h)$ as the observed state variable, such as using MLE to estimate the model and  then plan using the estimated model. However, as we mentioned above, due to  the unobserved confounders $i_h,h\in [H],$ these standard techniques will result in biased estimators for rewards and transition kernels \citep{pearl2009causality,hernan2010causal}. In the worst scenario, the error caused by the bias is lower bounded by a constant level, leading to a considerable loss in optimality.

There are two additional challenges, namely,  distribution shift (insufficient data coverage) and the existence of an enormous number of states and actions, in studying our strategic MDP.

In terms of the distribution shift, there are two implications. First, the distribution  of state-action pairs induced by the behavior policy may not cover the distribution induced by some other policies in the planning stage.  Second, we are interested in learning the policy for a target population of agents, whose distribution is also not necessarily the same as that in the data generation process. 
These result in the existence of distribution shift between our collected data and the target population.

The third challenge arises
due to an enormous number of states and actions in reinforcement learning (RL) applications. In this case, traditional tabular RL is inefficient in modeling, and specific function approximation is necessary to approximate the value function or the policy. Unlike most existing works studying offline RL in finite actions and states (tabular MDP)  or assuming particular constraints on the models (linear MDP),  we consider using general function approximation, especially under the existence of latent confounders in the paper.
 


To tackle these challenges, in the following subsection, we introduce the ideas of instrumental variable regression and the  pessimism principle.  

\subsection{Ideas for Addressing These Challenges}\label{alg_well}
In this subsection, we address the aforementioned challenges on latent confounder, insufficient data coverage and function approximation in \S\ref{iv} and \S\ref{reason_pess}, respectively.

\subsubsection{Dealing with the  Latent Confounder: Instrumental Variables (IV)}\label{iv}
We first tackle the challenge due to latent confounders by leveraging instrumental variables (IV). We define instrumental variables and illustrate why IV  remedies the curse of latent confounder in strategic MDP.

\begin{definition}[Instrumental Variables]\label{def_iv}
	A random variable $Z$ is an instrumental variable with respective to $(X,Y),$ if it is satisfies the following two conditions:
	\begin{itemize}
		\item $Z$ is not independent with $X$.
		\item $Z$ only affects $Y$ through $X$, and is independent with all other variables that has influence on $Y$ but are not mediated by $X.$
	\end{itemize}
\end{definition}
To utilize the instrumental variables, we present an assumption on the latent variables $i_h$. Recall that $i_h$  represents the private type of the  agent at step $h$.
{ \begin{assumption}\label{indenp_priviate}
	We assume $\{i_h\}_{h=1}^{H}$ are independent random variables. Moreover, for any given $h\in [H],$ $i_h$ is also assumed to be independent of $\{(s_t,a_t)\}_{t=1}^h$.
\end{assumption}}
 This assumption requires the private type $i_h$ involved in every stage does not share confounding variables across different stages $h\in [H]$ and is independent of past states and actions $\{(s_t,a_t)\}_{t=1}^h$.  This is satisfied  by the example discussed in \S\ref{model}, where we have a sequence of agents whose private types are drawn independently from some distribution and are independent of past states and actions. This assumption can also be satisfied when the principal is only interacting with a single agent until stage $H$, where for different $h\in[H]$, $i_h$ is independent with each other by representing the private type of that agent from various aspects.   For example, in a job interview with multiple stages, the hiring manager would like to test the agent's ability from different aspects (such as leadership and technical capabilities) in these stages. Thus, the underlying skills $\{i_h\}_{h=1}^{H}$ of the agent satisfy Assumption \ref{indenp_priviate}.

According to Definition \ref{def_iv}, we observe that $(s_h,a_h)$ serves as an instrumental variable for $(x_h,r_h)$ and $(x_h,s_{h+1}).$ 
 Here $x_h=\phi_x(s_h,a_h,o_h)$ with $\phi_x$ being a fixed and known embedding function. We next summarize this property in Lemma \ref{lem-instrument}.

\begin{lemma}\label{lem-instrument}
	Under our model settings given in \S\ref{model} and Assumption \ref{indenp_priviate},   $z_h:=(s_h,a_h)$ is  an instrumental variable for $(x_h,r_h)$ and $(x_h,s_{h+1})$ where $x_h=\phi_x(s_h,a_h,o_h)$ with $\phi_x$ being a fixed and known embedding function.
\end{lemma}

\begin{figure}[h]
	\centering
	\includegraphics[width=0.70\textwidth]{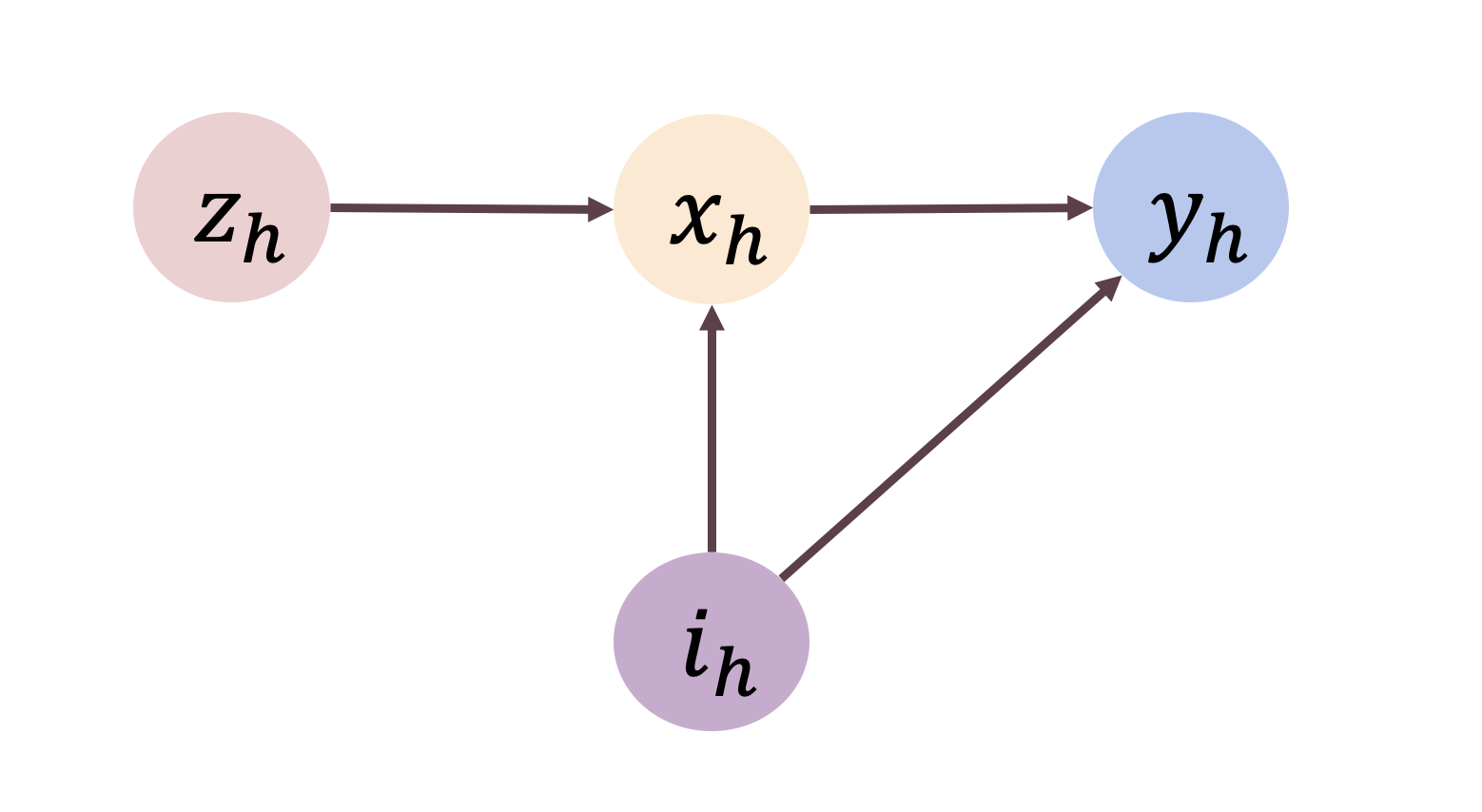}
	\caption{{Causal graph of the $h$-stage of strategic MDP. Here $z_h=(s_h,a_h)$  denotes the instrumental variable and $x_h = \phi_x(s_h,a_h,o_h)$ is our covariate. Moreover, $y_h$ represents the response variable (reward $r_h$ or next state $s_{h+1}$) and $i_h$ is the hidden confounder variable.}}
	\label{fig:DAG}
\end{figure}

Next, we utilize the IV $(s_h,a_h)$ to identify the strategic MDP model from the the confounded offline data. 
Therefore, $R_h^*(\cdot),G_{h}^*(\cdot)$ satisfy the following moment equations \eqref{identify_r} and \eqref{identify_g}:
\begin{align}
	\EE_{\rho}\big[R_h^*(s_h,a_h,o_h)-r_h\given s_h,a_h\big]&=0,\label{identify_r}\\
	\EE_{\rho}\big[G_h^*(s_h,a_h,o_h)-s_{h+1}\given s_h,a_h\big]&=\mathbf{0},\label{identify_g}
\end{align}
as $\EE_{\rho}[\xi_h\given s_h,a_h]=\mathbf{0}$ and $\EE_{\rho}[g_h\given s_h,a_h]=0$ by our model formulation and Assumption \ref{indenp_priviate}. 

Thus, in order to identify $R_h^*(\cdot)$, one then minimizes the following loss function
\begin{align}\label{loss_new}
	\min_{R_h\in L^2(\cX)} L(R_h)=\min_{R_h\in L^2(\cX)}\EE_{\rho}\Big[\EE_{\rho}[R_h(s_h,a_h,o_h)-r_h\given s_h,a_h]\Big]^2,
\end{align}
which is a projected least squares loss function. By \eqref{identify_r}, we see $R_h^*(\cdot)$ is a minimizer of this loss. 
It is worth noting that we are also able to use  $\psi_z(s_h,a_h)$ with some known embedding function $\psi_z(\cdot)$ as the instrumental variable. Similar to Remark \ref{remark:embedding}, for the simplicity of our notation, we also omit $\psi_z(\cdot)$ without specification. 

In this subsection, we utilize algorithmic instrumental variables to identify our model and thus tackle the first challenge on latent confounding variables. In the following subsection, we introduce the idea of pessimism for handling the distribution shift.

\subsubsection{Handling Distribution Shift: Pessimism Principle}\label{reason_pess}
In this subsection, we will provide some intuition of using pessimism to handle distribution shift. At a high level, for all $\pi\in \Pi$, we want to construct a data-driven $\hat J(\pi)$ which serves as a pessimistic estimate of $J(M^*,\pi)$, in the sense that 
\begin{align}\label{pessimism}
\hat J(\pi) \le J(M^*,\pi), \textrm{ for all } \pi\in \Pi.
\end{align}
Here $M^*$ denotes the true model defined in \S\ref{planning}.
We then choose the optimal policy with respect to such a pessimistic estimate, namely,
\begin{align}
	\hat\pi:=\argmax_{\pi\in \Pi}\hat J(\pi).\label{pess}
	\end{align}
In this scenario,   the suboptimality of $\hat\pi$ is bounded by
\begin{align}
	\textrm{SubOpt}(\hat\pi)&=J(M^*,\pi^*)-J(M^*,\hat\pi)\nonumber\\
	&=J(M^*,\pi^*)-\hat J(\pi^*)+\hat J(\pi^*)-\hat J(\hat\pi)+\hat J(\hat\pi)-J(M^*,\hat\pi)\nonumber\\
	&\le J(M^*,\pi^*)-\hat J(\pi^*)+\hat J(\hat\pi)-J(M^*,\hat\pi)\nonumber\\&\le J(M^*,\pi^*)-\hat J(\pi^*).\label{upper_subopt}
\end{align}
The first inequality follows from our construction of $\hat\pi$ given in \eqref{pess} and the second inequality follows from \eqref{pessimism}.  One observes that \eqref{upper_subopt} reflects the bias due to pessimism. As it only depends on the trajectory induced by $\pi^*,$ it is small as long as the dataset has good coverage over $\pi^*.$

To solve the third challenge, we leverage general function approximation to accommodate this issue. Using general function approximation admits many merits. It allows more flexible function classes than tabular and linear cases, see \S\ref{pessimism_alg} for more details. Moreover, we are also able to handle the function class misspecification which will be discussed in \S\ref{sec:misspecified}. 

Next, we combine all pieces and provide our algorithm for solving strategic MDP. In specific, we construct such pessimistic value functions using lower level sets of the IV-assisted loss functions with general function approximation.

\subsection{Putting All Pieces Together: Pessimistic Strategic MDP with General Function Approximation}\label{pessimism_alg}
In this subsection, we propose a novel algorithm, namely, \underbar{p}essimistic policy \underbar{l}earning with \underbar{a}lgorithmic  i\underbar{n}struments  (\algo), which resolves the aforementioned challenges by leveraging the ideas of pessimism principle and instrumental variable regression with general function approximation. 

\subsubsection{A Glimpse of the  \algo~Algorithm}\label{abstract_alg}

As we mentioned in \S\ref{reason_pess}, to leverage pessimism, we need to construct a data-driven pessimistic function $\hat J(\pi)$ such that $\hat J(\pi)\le J(M^*,\pi),$ for all $\pi\in \Pi$.  For any confidence region $\cM$ of $M^*=\{(\bar{R}_h^*,\bar{\PP}_h^*)\}_{h=1}^{H},$ which is constructed based on data, and any given policy $\pi$, we let  $\hat J(\pi):=\min_{M\in \cM}J(M,\pi).$ Here $J(M,\pi)$ represents the total value function evaluated by policy $\pi$ under model $M=\{(\bar{R}_h,\bar{\PP}_h)\}_{h=1}^{H}.$ Then it can be shown that $\hat J(\pi)=\min_{M\in \cM}J(M,\pi)\le J(M^*,\pi)$ as long as $\cM$ contains $M^*.$
In this scenario, we then construct the  estimated policy $\hat\pi$ as we discussed in \S\ref{reason_pess}, namely,
\begin{align}\label{est_policy}
	\hat\pi=\argmax_{\pi \in \Pi}\min_{M\in \cM}J(M,\pi).
\end{align}
As a consequence, by \eqref{upper_subopt}, we have
\begin{align*}
	\textrm{SubOpt}(\hat\pi)\le  J(M^*,\pi^*)-\hat J(\pi^*)=J(M^*,\pi^*)-\min _{M\in\cM} J(M,\pi^*).
\end{align*}
Next, we construct the confidence region $\cM$ of $M^*=\{(\bar{R}_h^*,\bar{\PP}_h^*)\}_{h=1}^{H}$.

\subsubsection{Construction of Confidence Sets}\label{const_level_set}
In this subsection, we    construct the confidence sets $\cM:=\{\bar{\cR}_h,\bar{\cG}_h\}_{h=1}^H$ for $\{\bar{R}_h^*,\bar{\PP}_h^*\}_{h=1}^H$  via  lower level sets of the sample loss functions, where the loss functiosn are derived from minimax estimation.

First,  we introduce the sample version of the loss functions. Recall that the model identification is derived in \S\ref{iv} via the loss function in \eqref{loss_new}
 which admits $R_h^*$ as a minimizer. Note that \eqref{loss_new} can not be directly estimated from the data due to the conditional expectation inside the square function. By Fenchel duality \citep{rockafellar2009variational,shapiro2021lectures}, the loss in \eqref{loss_new} is equivalent to  the following minimax loss function that can be estimated via data:
\begin{align*}
    \min_{R_h\in L^2(\cX)}L(R_h)=2 \min_{R_h\in L^2(\cX)}\max_{f\in L^2(\cS\times \cA)}\EE_{\rho}\big[(r_h-R_h(s_h,a_h,o_h))f(s_h,a_h)\big]-\EE_{\rho}\big[f^2(s_h,a_h)\big].
\end{align*}
Here $L^2(\cX)$ denotes the $\ell_2$-integrable functions on domain $\cX$ under probability measure $\rho.$ 

Next, we derive confidence sets $\bar{\cR}_h$ and $\bar{\cG}_h$ for $\bar{R}^*_h$ and $\bar{\PP}^*_h$, respectively, by leveraging the sample version of the loss function and the offline dataset in \S\ref{offline_data}.
 First, we construct a Wilk's type confidence set \citep{wilks1938large} for $R_h^*$ by letting
\begin{align}\label{confidence_r}
    \cR_h=\Big\{R_h\in \RR_h: \cL_K(R_h)-\cL_K(\hat R_h)\lesssim c_{r,h,K}^2 \Big\},
\end{align}
where we define $\cL_K$ and $\hat R_h$ respectively as 
\begin{align*}
    \cL_K(R_h)&=\sup_{f\in \cF}\Bigg\{\frac{1}{K}\sum_{k=1}^{K}\Big(r_{h}-R_h\Big(s_h^{(k)},a_h^{(k)},o_h^{(k)}\Big)\Big) f\Big(s_h^{(k)},a_h^{(k)}\Big)-\frac{1}{2K}\sum_{k=1}^{K}f^2\Big(s_h^{(k)},a_h^{(k)}\Big) \Bigg\},\,\,\textrm{and}\\
    \hat R_h&=\argmin_{R_h\in \RR_h}\cL_K(R_h).
\end{align*}
Here $\RR_h:=\{R_h:\cX\rightarrow [-L,L]\}$ with $R_h^*\in\RR_h$ and $\cF:=\{f:\cS\times \cA\rightarrow [-L,L]\}$.
Moreover, $c_{r,h,K}\in \RR$ is a threshold that will be specified later. By properly choosing this threshold, we can prove that $\cR_h$ contains $R_h^*$ with high probability. In other words, the confidence set $\cR_h$ is valid.  

  For our aggregated MDP discussed in \S\ref{planning}, we define the high-confidence set for $\bar{R}_h^*$ in \eqref{margin_rew}  as a transformation of $\cR_h$ according to the target population of agents: 
\begin{align*}
    \bar{\cR}_h=\Bigg\{\bar{R}_h:\bar{R}_h=\int_{o_h,i_h}[R_h(s_h,a_h,o_h)+f_{1h}(i_h)]\ud F_h(o_h\given s_h,a_h,i_h)\ud P_h(i_h),R_h\in \cR_h \Bigg\}.
\end{align*}
Here we recall that $F_h(\cdot\given s_h,a_h,i_h)$ is the known conditional distribution of $o_h$ given $(s_h,a_h,i_h)$ in the planning stage. Observe that when $R_h^*\in \cR_h$, we have $\bar{R}_h^*\in \bar{\cR}_h$.

Similarly, the confidence set for the $j$-th coordinate of function $G_{h}^*:\cX\rightarrow  \RR^{d_1}$ is defined as
\begin{align}\label{confidence_g}
    \cG_{h,j}=\Big\{G_{h,j}\in \mathbb{G}_{h,j}: \cL_K(G_{h,j})-\cL_K(\hat G_{h,j})\lesssim c_{h,g,K}^2 \Big\}.
\end{align}
Here $\mathbb{G}_{h,j}:=\{G_{h,j}:\cX\rightarrow [-L,L]\},$ with $G_{h,j}^*\in\mathbb{G}_{h,j}\forall j\in[d_1]$. In addition, with slight abuse of notation, we define $\cL_K$ and $\hat G_{h,j}$ as
\begin{align*}
    \cL_K(G_{h,j})&=\sup_{f\in \cF}\Bigg\{\frac{1}{K}\sum_{k=1}^{K}\Big(s^{(k)}_{h+1,j}-G_{h,j}\Big(s_h^{(k)},a_h^{(k)},o_h^{(k)}\Big)\Big) f\Big(s_h^{(k)},a_h^{(k)}\Big)-\frac{1}{2K}\sum_{k=1}^{K}f\Big(s_h^{(k)},a_h^{(k)}\Big)^2 \Bigg\},\,\,\textrm{and}\\
    \hat G_{h,j}&=\argmin_{G_{h,j}\in \mathbb{G}_{h,j}}\cL_K(G_{h,j}).
\end{align*}
Similarly, we define $\bar{\cG}_{h,j}$ as the confidence set for $\bar{G}_{h,j}^*$: 
\begin{align*}
    \bar{\cG}_{h,j}&=\Bigg\{\PP_{h,j}(\cdot\given s_h,a_h):\PP_{h,j}(\cdot\given s_h,a_h)=\int_{o_h,i_h}\PP_{h,j}(\cdot\given s_h,a_h,o_h,i_h)\ud F_h(o_h\given s_h,a_h,i_h)\ud P_h(i_h) ,\\&\PP_{h,j}(\cdot\given s_h,a_h,o_h,i_h)\sim N(G_{h,j}(s_h,a_h,o_h)+f_{2h,j}(i_h),\sigma^2),\textrm{ with }G_{h,j}(\cdot) \in \cG_{h,j} \Bigg\}.
\end{align*}
 Combining these confidence sets together, we achieve that $\cM=\{(\bar{\cR}_h,\bar{\cG}_h)\}_{h=1}^{H}$  contains the true model $M^*$ with high probability. 

It is worth noting that most existing literature only studies tabular or linear MDP using pessimism-based ideas with no hidden confounders \citep{jin2021pessimism,rashidinejad2021bridging,shi2022pessimistic,yan2022efficacy}. However, our algorithm \algo\,\,works for a general class of non-parametric reward and transition kernels and even permits the existence of hidden confounding variables. In the next section, we provide more details on the theoretical guarantees of \algo.

\section{Theoretical Results}\label{sec:theory}
In this section, we present theoretical guarantees for \algo\,\, in \S\ref{sec:method}. 
We next analyze our policy optimization results with well-specified function classes in \S\ref{sec:well-specify} and misspecified ones in \S\ref{sec:misspecified}, respectively.
\subsection{Theoretical Results for Suboptimality under Well-Specified Function Class.}\label{sec:well-specify}
In this subsection, we provide theoretical guarantees for the suboptimality of our estimated policy \eqref{est_policy}. Here $\RR_h,\mathbb{G}_{h,j}$ are correctly specified, containing $R_h^*$ and $G_{h,j}^*$. 

Before proceeding to the theoretical analysis, we first present several assumptions. The first assumption ensures that the sampling distribution $\rho$ satisfies a partial coverage condition.
\begin{assumption} [Concentrability Coefficients of Partial Coverage] \label{assume:concentrability}
{There exists a constant $C_{\pi^*}>0$ such that 
\begin{align*}
    \sup_{R_h\in \{\cR_h-R_h^*\}}\frac{\EE_{\ud ^{(\pi^*,\PP)}(s_h,a_h,o_h)}[R_h^2(s_h,a_h,o_h)]}{\EE_{\rho(s_h,a_h,o_h)}[ R_h^2(s_h,a_h,o_h)]}&\le C^2_{\pi^*} ,\forall h\in [H],\\
        \sup_{G_{h,j}\in \{\cG_{h,j}-G_{h,j}^*\}}\frac{\EE_{
        \ud ^{(\pi^*,\PP)}(s_h,a_h,o_h)
        }[G^2_{h,j}( s_h,a_h,o_h)]}{\EE_{\rho(s_h,a_h,o_h)}[G_{h,j}^2( s_h,a_h,o_h)]}& \leq C^2_{\pi^*},\forall j\in [d_1],\forall h\in [H],
\end{align*} 
where we define $\cR_h-R_h^*=\{(R_h-R_h^*)(\cdot), R_h(\cdot)\in \cR_h\},$ $\cG_{h,j}-G_{h,j}^*=\{(G_{h,j}-G_{h,j}^*)(\cdot), G_{h,j}(\cdot)\in \cG_{h,j}\}.$} In addition, we define $\ud^{(\pi^*,\PP)}(s_h,a_h,o_h)$ as the joint distribution of $(s_h,a_h,o_h)$ at stage $h$ induced by $\pi^*=\{\pi_j^*(a_j\given s_j)\}_{j=1}^{h}$ and $\PP=\{\PP_j(o_i\given s_j,a_j)\}_{j=1}^{h},$ {where $\PP_j(\cdot \given s_j,a_j):=\int_{i_j} F_j(\cdot\given i_j,a_j,s_j)\ud P_j(i_j)$ is the known marginal distribution of $o_j$ given $(s_j,a_j)$ in the $j$-th planning stage.} 
\end{assumption}
In this assumption, we let the ratio of the square loss under two distributions, namely, the distribution $d^{(\pi^*,\PP)}(\cdot)$ induced by the optimal policy with ``population-specific" agents and the sample distribution $\rho,$ be bounded by a concentrability coefficients $C_{\pi^*}^2$. This notion is adapted from \cite{uehara2021pessimistic}. A sufficient condition for this assumption is when the ratio of these two densities $\sup_{(a_h,s_h,o_h)}{f^{(\pi^*,\PP)}(s_h,a_h,o_h)}/{f_{\rho}(s_h,a_h,o_h)}$  is upper bounded by $C_{\pi^*}^2$.
Therefore, Assumption \ref{assume:concentrability} is mild in a sense that it only requires partial coverage of the sample distribution $\rho(\cdot)$ (i.e., covering $\ud^{(\pi^*,\PP)}(\cdot)$ induced by the optimal policy $\pi^*$ instead of all $\pi$). 

Furthermore, as we utilize IV regression in \algo, 
analysis of IV regression enables us to 
 bounded {projected MSE (PMSE)} $\EE_{\rho}[\EE_{\rho}[(R_h-R_h^*)(s_h,a_h,o_h)\given s_h,a_h]^2]$, which is  measured in the space of instrumental variable, as opposed to the more desired  MSE  $\EE_{\rho}[R_h-R_h^*)^2(s_h,a_h,o_h)]$. Transforming the PMSE to MSE, acting as a discontinuous mapping, results in an ill-posed inverse problem \citep{horowitz2013ill}. 
Thus, 
we next present an ill-posed condition, which measures the difficulty of such an inverse problem. 
\begin{assumption}[Ill-posed Condition]{ \label{ill-pose} We assume there exist coefficients $\tau_{r,h,K}$ and $\tau_{G,h,K}$ depending on $K,h,$ and function classes $\cR_h,\cG_{h}$ such that the following conditions hold:
 \begin{align*}
\sup_{R_h\in \{\cR_h-R_h^*\}}\frac{\EE_{\rho(s_h,a_h,o_h)}[R_h^2(s_h,a_h,o_h)]}{\EE_{\rho(x_h,a_h)}[ \EE_{\rho(o_h\given s_h,a_h)}[R_h(s_h,a_h,o_h)\given s_h,a_h]^2]}&\le \tau^2_{r,h,K},\forall h\in [H],\\
\sup_{G_{h,j}\in \{\cG_{h,j}-G_{h,j}^*\}}\frac{\EE_{\rho(s_h,a_h,o_h)}[G_{h,j}^2(s_h,a_h,o_h)]}{\EE_{\rho(s_h,a_h)}[ \EE_{\rho(o_h\given s_h,a_h)}[G_{h,j}(s_h,a_h,o_h)\given s_h,a_h]^2]}&\le \tau^2_{G,h,K},\forall j\in [d_1],  \forall h\in [H].
\end{align*}}
\end{assumption}

It is worth noting that the ill-posed condition is a standard assumption in the literature of {(nonparametric)} instrumental regression. See, e.g.,  \cite{chen2011rate,darolles2011nonparametric,horowitz2013ill,dikkala2020} for more details. 

Next, to ensure that the minimax estimation procedure given in \S\ref{const_level_set} is valid, we impose an assumption  on the employed function classes.
\begin{assumption}\label{assume:functionclass}
Let $\cF:=\{f:\cZ\rightarrow [-L,L]\},\RR_h=\{R_h:\cX\rightarrow [-L,L]\},\mathbb{G}_{h,j}=\{G_{h,j}:\cX\rightarrow [-L,L]\},\forall j\in[d_1],\forall h\in[H]$ be function classes given in \eqref{confidence_r} and \eqref{confidence_g} with $R_h^*\in \RR_h,G_{h,j}^*\in \mathbb{G}_{h,j}\forall j\in[d_1].$ Moreover, we assume that for all $ R_h\in \RR_h,$ and $G_{h,j}\in \mathbb{G}_{h,j}$, it holds that  
\begin{align*}
\mathbb{T}(R_h-R_h^*)(z)&:=\EE_{\rho}[R_h(o_h,a_h,s_h)-R_h^*(o_h,a_h,s_h)\given (a_h,s_h)=z]\in \cF,\\ \mathbb{T}(G_{h,j}-G_{h,j}^*)(z)&:=\EE_{\rho}[G_{h,j}(o_h,a_h,s_h)-G_{h,j}^*(o_h,a_h,s_h)\given (a_h,s_h)=z]\in \cF, \forall h\in [H],\forall j\in [d_1].
\end{align*}
Furthermore,  $\cF$ is assumed to be a symmetric and star-shaped function class.
\end{assumption}
 In this assumption, we assume function classes $\cF,\RR_h,\mathbb{G}_{h,j},\forall j\in [d_1],\forall h\in [H]$ given in \eqref{confidence_r} and \eqref{confidence_g} are well-specified in the sense that $R_h^*\in \RR_h,G_{h,j}^*\in \mathbb{G}_{h,j}$ and 
$\cF$ contains the projected function classes $$\cP(R,h):=\{\mathbb{T}(R_h-R_h^*)(z),R_h\in \RR_h\}$$ and $$\cP(G,h,j):=\{\mathbb{T}(G_{h,j}-G_{h,j}^*)(z),G_{h,j}\in \mathbb{G}_{h,j}\},\,\,\forall j\in[d_1],h\in[H].$$ 

Recall that  the confidence sets for $R_h^*$ and $G_{h,j}^*$ are constructed in \eqref{confidence_r} and \eqref{confidence_g}. We now specify the corresponding thresholding parameters $c_{h,r,K}$ and $c_{h,G,K}$.
We set $c_{h,r,K}=\cO(\delta_{\cR,h,\delta})$ with $\delta_{\cR,h,\delta} =\cO(\delta_{\cR,h}+\sqrt{\log(H/\delta)/K}).$ Here $ \delta_{\cR,h}$ is the maximum critical radii of $\cF$ and the following function class 
\begin{align}\label{cr_h}
    \cR_h^{*}:=\Big\{c(R_h-R_h^*)(x)\cdot \mathbb{T}(R_h-R_h^*)(z): \cX\times\cZ\rightarrow\RR; R_h\in \mathbb{R}_h,\forall c\in [0,1] \Big\}, h\in [H].
\end{align}
Here $\cR_h^*$ is the star hull of the function class which is the product of $(\RR_h^*-R_h^*)$ and its projections in $\cP(R,h).$ 
The upper bound of the critical radii of $\cR_h^*$ and $\cF$ measures the maximal complexity of function classes $\cR_h^{*}$ and $\cF$, respectively.

Similarly, we define $c_{h,G,K}:=\cO(\delta_{\cG,h,\delta})$ with $\delta_{\cG,h,\delta} =\cO(\delta_{\cG,h}+\sqrt{\log(Hd_1/\delta)/K})$ and  $\delta_{\cG,h}$ being the maximum critical radii of $\cF$ and the following function class
\begin{align}\label{cg_hj}
    \cG_{h,j}^{*}:=\Big\{ c(G_{h,j}-G_{h,j}^*)(x) \cdot\mathbb{T}(G_{h,j}-G_{h,j}^*)(z):\cX\times\cZ\rightarrow\RR; G_{h,j}\in \mathbb{G}_{h,j},\forall c\in [0,1] \Big\},
\end{align}
$\forall j\in [d_1].$ Here we recall that $d_1$ is the dimension of the state variable.
With these necessary assumptions at hand, we bound the suboptimality in the following theorem.
\begin{theorem}\label{sub_optimality}
Under Assumptions \ref{assume:concentrability}, \ref{ill-pose}, and \ref{assume:functionclass}, with probability $1-\delta-1/K$, the suboptimality of $\hat\pi$ returned by \algo\,\, is upper bounded by $$\textrm{SubOpt}(\hat\pi)\lesssim C_{\pi^*}\Bigg[H\sum_{h=1}^{H}\tau_{G,h,K}\sqrt{d_1}L_{K,d_1}\delta_{\cG,h,\delta}+\sum_{h=1}^{H}\tau_{r,h,K}L_{K,1}\delta_{\cR,h,\delta}\Bigg],$$ where $L_{K,x}=L+\sigma\sqrt{(\log H+1)\log (Kx)},$ with $x\in \{1,d_1\}$,
$\delta_{\cR,h,\delta}=\delta_{\cR,h}+\sqrt{\log(H/\delta)/K},$ and $\delta_{\cG,h,\delta}=\delta_{\cG,h}+\sqrt{\log(d_1H/\delta)/K}$. Here $L$ is the upper bound of bounded functions in $\cF,\RR_h,\mathbb{G}_{h,j},\forall h\in [H],\forall j\in [d_1]$ mentioned in Assumption \ref{assume:functionclass} in $\ell_{\infty}$-norm. Besides, $\delta_{\cG,h}$ is the maximum critical radii of $\cF,\cG^{*}_{h,j}\forall j\in[d_1]$, and  $\delta_{\cR,h}$ is an upper bound of the critical radii of $\cF,\cR_h^*.$  
\end{theorem}
Theorem \ref{sub_optimality} provides an upper bound for the suboptimality of $\hat\pi$ under mild conditions. This upper bound involves critical radii, measuring the complexity of function classes $\cF$ and $\RR_h,h\in[H]$, time horizon $H$, and concentrability and ill-posed coefficients. For every single stage, we achieve a fast statistical rate with order $\cO(\sqrt{1/K})$ under many function classes; see \S\ref{linear_function} and \S\ref{example:rkhs} for more details. This matches the well-known regression rate and thus is minimax optimal up to ill-posed coefficients for general instrumental regression problems. 

Compared with existing literature that studies strategic regression \citep{stevenwu2021}, our contribution is two-fold. First, we extend the strategic regression to the decision-making problem where the principal interacts with strategic agents in multiple stages. In contrast,  strategic regression only involves single-stage decision making. Moreover, we consider general function classes and utilize the method of moments and pessimism to tackle the technique difficulty, whereas \cite{stevenwu2021} only consider the linear function class and use point estimators.

Compared with existing works studying offline MDP, we propose a new framework, namely, strategic MDP, that captures strategic interactions between a sequence of agents and a principal. In specific, the agents have private types, acting as latent confounders, affect both principal's observation, immediate reward and future state. We propose a model-based algorithm and develop novel proof frameworks by leveraging IV regression, pessimism via lower level sets, and general function approximation to eliminate the technical challenges. We achieve similar statistical rates of suboptimality with standard model-based RL (without latent confounders) with function approximation \citep{,duan2020minimax,uehara2021pessimistic}, which is minimax optimal up to $H$ factors.

We next provide two instantiations of Theorem \ref{sub_optimality} with linear and kernel functions, respectively.

\subsubsection{Example: Linear Function Class}\label{linear_function}
In this subsection, we provide an instantiation of Theorem \ref{sub_optimality} when $R_h^*,G_{h,j}^*\forall j\in [d_1]$ lie in linear spaces with finite dimensions. Throughout this subsection, we let $x_h:=\phi_x(s_h,a_h,o_h)\in \cX$ and $z_h:=\psi_z(s_h,a_h)\in \cZ$ with $\phi_x(\cdot)$ and $\psi_z(\cdot)$ being some embedding mappings. 
For all $h\in [H]$, we define proper linear function classes $\cF$, $\RR_h$, and $\mathbb{G}_{h,j}$ as
\begin{align}
    \cF&=\big\{\langle\beta, z \rangle:\cZ\rightarrow \RR; \beta\in \RR^{m},\|\beta\|_2\le U \big\},\label{test_class}\\ \RR_h&=\mathbb{G}_{h,j}=\big\{\langle\theta, x \rangle:\cX\rightarrow\RR; \theta\in \RR^{n_h}, \|\theta\|_2\le B \big\}.\label{regression_class}
 \end{align} 
 Here we assume the covariate $x_h$ has dimension $n_h, h\in [H]$ and instrumental variable has dimension $m$. We next summarize necessary assumptions on our true models in the following Assumption \ref{assume:linear}.
 \begin{assumption}\label{assume:linear}
 Suppose $R_h^*(x_h)=\langle x_h,\theta^* \rangle$, $G_{h,j}^*(x_h)=\langle x_h,\theta_j^*\rangle,\forall j\in[d_1]$ with $\|\theta^*\|_2\le B,\|\theta_j^*\|_2\le B.$
 We assume $\EE[x_h\given z_h]=W_h z_h$ with $W_h\in \RR^{n_h\times m}$ and $\|W_h\|_{\textrm{F}}\le B_F, \forall h\in [H].$
 \end{assumption}
 In Assumption \ref{assume:linear}, we assume every dimension of $x_h$ conditional on $z_h$ is a linear function of $z_h.$ Under this assumption, we have $\EE[R_h(x_h)-R_h^*(x_h)\given z_h=z]=\langle\theta-\theta^*,W_h z\rangle=\langle \tilde{\beta}_h,z \rangle\in \cF$, for any $ R_h\in\mathbb{R}_h$ with $\tilde{\beta}_h=W_h^\top(\theta-\theta^*),$ if the constant $U$ given in \eqref{test_class} is large enough. Similar situation also works for every coordinate of $G_{h}.$ Thus, Assumption \ref{assume:functionclass} given in \S\ref{sec:well-specify} holds in this case.
 Following \eqref{cr_h} and \eqref{cg_hj}, we have
 \begin{align}
    \cR_{h}^*&=\Big\{ \langle \theta-\theta^*,x \rangle\langle \beta,z \rangle: \cX\times\cZ\rightarrow \RR; \|\theta-\theta^*\|_2\le B,\|\beta\|_2\le U\Big\},\label{product_class}
    \\\cG_{h,j}^*&=\Big\{\langle \theta-\theta_j,x\rangle\langle \beta,z \rangle: \cX\times\cZ\rightarrow\RR; \|\theta-\theta_j^*\|_2\le B,\|\beta\|_2\le U\Big\}.\label{product_class_g}
\end{align}

Next, we discuss the corresponding ill-posed condition defined in Assumption \ref{ill-pose} when both $R_h^*(\cdot)$ and $G_{h,j}^*(\cdot)$ fall in linear function classes.

\begin{assumption}[Ill-posed condition]\label{ill-posed-linear} We assume $0/0=0$ and
\begin{align}\label{ill-pose_linear}
  \sup_{x\in \RR^{n_h}}\frac{x^\top\EE_{x_h\sim \rho(\cdot)}\Big[x_h x_h^\top\Big]x}{x^\top \EE_{z_h\sim \rho(\cdot)}\Big[W_hz_h z_h^\top W_h^\top\Big]x }\le \tau_{h}^2.
\end{align}
\end{assumption}

It is worth noting that Assumption \ref{ill-posed-linear} does not imply that the  matrix $\EE_{z_h\sim \rho(z)}[W_hz_hz_h^\top W_h^\top]$ is invertible.  With the convention $0/0=0$,  this assumption holds when the eigenspace of $\mathbb{Z}:=\EE_{z_h\sim \rho(z)}[W_h z_hz_h^\top W_h^\top]$ with respect to nonzero eigenvalues contains that of $\mathbb{X}:=\EE_{x_h\sim \rho(x)}[x_hx_h^\top]$. Imagine an extreme case, where $x_h=z_h$, we have the ill-posed coefficient is equal to 1 and this does not imply the invertibility of $ \mathbb{Z} =\EE[x_hx_h^\top]$.
 
We observe from Theorem \ref{sub_optimality} that the suboptimality upper bound only involves the critical radius of $\cF,\cR_h^*$ and $\cG_{h,j}^*,\forall j\in [d_1].$
As  $\cF$ is a linear space with finite dimension, its critical radius  is of the order $\cO(\sqrt{{m}/{K}})$ \citep{wainwright_2019}.
Meanwhile, as $\cR_h^*$ can be viewed as the product of two linear spaces, each with finite dimensions, its critical radius is of order $\delta_{\cR,h}=\cO(\sqrt{{\max\{m,n_h\}\log K}/{K}})$ \citep{wainwright_2019}. Similar results also holds for $ \cG_{h,j}^*\forall j\in[d_1].$ Plugging these results into the upper bound of Theorem \ref{sub_optimality}, we obtain the following Corollary \ref{example_linear}.

\begin{corollary}\label{example_linear}
 Let $\cF,\RR_h,\mathbb{G}_{h,j},\cR_h^*,\cG_{h,j}^*$ be defined as in \eqref{test_class},\eqref{regression_class}, \eqref{product_class} and \eqref{product_class_g}, respectively. Assume that we construct confidence sets $\cR_h,\cG_h$ and policy $\hat\pi$ according to \algo. Under Assumption \ref{ill-posed-linear}, with probability $1-\delta-1/K,$ we obtain
\begin{align*}
    \textrm{SubOpt}(\hat\pi)\lesssim \sum_{h=1}^{H}\tau_{h}C_{\pi^*}\Big(H\cdot\sqrt{d_1}L_{K,d_1}+L_{K,1}\Big)\Bigg(\sqrt{\frac{\max\{m,n_h\}\log K}{K}}+\sqrt{\frac{\log (Hd_1/\delta)}{K}}\Bigg).
\end{align*}
Here $L_{K,x}=L+\sigma\sqrt{(\log H+1)\log (Kx)}$ with  $x\in \{1,d_1\}$ and $L$ is the upper bound of all functions in $\RR_h,\mathbb{G}_{h,j},\forall h\in [H],j\in [d_1]$ in $\ell_{\infty}$-norm. 
\end{corollary}

\subsubsection{Example: Kernel Function Class}\label{example:rkhs}
In this subsection, we study an instantiation of Theorem \ref{sub_optimality} when $R_h^*,G_{h,j}^*\forall j\in [d_1]$ lie in a RKHS. Before continuing, we first give a brief introduction to RKHS. 

 An RKHS is associated with a positive semidefinite kernel $\cK:\cX\times\cX\rightarrow \RR.$ Under mild regularity conditions, Mercer’s theorem guarantees that $\cK$ admits an eigenexpansion
of the form
\begin{align*}
    \cK(x,x')=\sum_{i=1}^{\infty}\mu_i\phi_i(x)\phi_i(x'),
\end{align*}
for a sequence of nonnegative eigenvalues $\{\mu_i\}_{i\ge 1}$ and eigenfunctions $\{\phi_i\}_{i\ge 1}$ which are orthogonal in $L^2(\PP).$ Here $\PP(\cdot)$ is some probability measure.
Given such an expansion, the RKHS norm can be written as
\begin{align*}
    \|f\|_{\cH}^2=\sum_{i=1}^{\infty}\frac{\theta_i^2}{\mu_i},\textrm{ with } \theta_i=\int_{\cX}f(x)\phi_i(x)\ud \PP(x).
\end{align*}
Thus, the induced RKHS by kernel $\cK$ is written as
\begin{align*}
    \cH:=\bigg\{f=\sum_{i=1}^{\infty}\theta_i\phi_i\given \sum_{i=1}^{\infty}\frac{\theta_i^2}{\mu_i}<\infty \bigg\},
\end{align*}
with the inner product
\begin{align*}
    \langle f,g\rangle_{\cH}=\sum_{i=1}^{\infty}\frac{\langle f,\phi_i\rangle \langle g,\phi_i\rangle }{\mu_i}.
\end{align*}
Here $\langle \cdot,\cdot\rangle$ denotes the inner product in $L^2(\cX,\PP).$

We assume $R_h^*,G_{h,j}^*\forall j\in [d_1],\forall h\in [H]$ lie in the Reproducing Kernel Hilbert spaces (RKHS) $\cH_{\cR},\cH_{\cG}$ with kernels $\cK_{\cR}(\cdot,\cdot)$ and $\cK_{\cG}(\cdot,\cdot)$ that are bounded in $\ell_{\infty}$-norm, respectively.   The associated probability measure is the sampling probability measure $\rho$.  Moreover, in this subsection, we define covariate $x_h:=\phi_x(s_h,a_h,o_h)\in \cX$ and instrumental variable $z_h:=\psi(s_h,a_h)\in \cZ$ with some known embedding functions $\phi_x(\cdot),\psi_z(\cdot).$ 

To begin with, we define the function classes that contains $r_h^*,G_{h,j}^*, \forall h\in[H]$ as:
\begin{align}
    \RR_h&=\Big\{R_h\in \cH_{\cR}: \cX\rightarrow \RR; \|R_h\|_{\cH_{\cR}}^2\le C_{h}^2 \Big\}.\label{rkhs_r}\\
        \mathbb{G}_{h,j}&=\Big\{G_{h,j}\in \cH_{\cG}:\cX\rightarrow \RR; \|G_{h,j}\|_{\cH_{\cG}}^2\le C_{h}^2 \Big\},\forall j\in [d_1],\label{rkhs_g}
    \end{align}
   where $C_{h}^{2}$ is a constant only depending on $h$. 
    In addition,  the test function class $\cF $ in RKHS $\cH_{\cF}$ with $\ell_{\infty}$-bounded kernel $\cK_{\cF}$ is given as follows:
\begin{align}
        \cF&=\Big\{f\in \cH_{\cF}:\cZ\rightarrow \RR; \|f\|_{\cH_{\cF}}^2\le C_1^2 \Big\},\label{rkhs_f}
\end{align}
in which $C_1$ is an absolute constant. 

 Similar to Assumption \ref{assume:functionclass}, in the example with RKHS,  the test function class $\cF$ with kernel $\cK_{\cF}$ is also assumed to contain the projected functions from $\RR_h,\mathbb{G}_{h,j},\forall h\in [H],j\in [d_1].$  
In addition, we deduce  $$\cR_h^*=\{(R_h-R_h^*)(x)\mathbb{T}(R_h-R_h^*)(z): R_h\in \cH_{\cR},\|R_h-R_h^*\|_{\cH_{\cR}}^2\le C_3^2\},$$ and $$\cG_{h,j}^*=\{(G_{h,j}-G_{h,j}^*)(x)\mathbb{T}(G_{h,j}-G_{h,j}^*)(z): G_{h,j}\in \cH_{\cR},\|G_{h,j}-G_{h,j}^*\|_{\cH_{\cG}}^2\le C_3^2\},$$ $\forall h\in [H],\forall j\in [d_1],$ from \eqref{cr_h}, \eqref{cg_hj}, respectively.

According to the Proposition 12.31 of \cite{wainwright_2019}, the space $\cR_h^*,\forall h\in [H]$ also admits a reproducing kernel, defined as 
\begin{align*}
    \cK_{\cR_h^{*}}((x,z),(x^{*},z^{*}))=\cK(x,x^*)\cK(z,z^*),
\end{align*}
if the inner product of $\cR_h^*$ is defined as  $\langle R_h f_h,\tilde{R}_h\tilde{f}_h \rangle_{\cR_h^*}=\langle R_h,\tilde{R}_h \rangle_{\cH_{\cR}}\langle f_h,\tilde{f}_h \rangle_{\cH_{\cF}}.$ Similar situation also holds for $\cG_{h,j}^*,\forall j\in [d_1],h\in [H].$

Finally, we discuss the ill-posed condition in RKHS as follows.
\begin{assumption}\label{ill-posed_rkhs}
 We assume 
\begin{align*}
    \max_{R_h\in \RR_h:\EE_{\rho}[\mathbb{T}(R_h-R_h^*)^2(z_h)]\le x^2}\EE_{\rho}\Big[\big(R_h-R_h^*\big)^2(x_h)\Big]&\le  \tau_{r}^2(x),\forall h\in [H],\\
    \max_{G_{h,j}\in \mathbb{G}_{h,j}:\EE_{\rho}[\mathbb{T}(G_{h,j}-G_{h,j}^*)^2(z_h)]\le x^2}\EE_{\rho}\Big[\big(G_{h,j}-G_{h,j}^*\big)^2(x_h)\Big] &\le \tau_{G}^2(x),\forall h\in [H],j\in [d_1],
\end{align*}
where $\tau_r(x),\tau_G(x)$ are fixed functions that only depends on $x$ and the associated probability measure is the sample distribution $\rho.$
\end{assumption}
\begin{remark}
In this remark, we provide more details on functions $\tau_r(\cdot),$ as $\tau_{G}(\cdot)$ can be analyzed in the same way. 
Let $I=\{1,\cdots,m\},$ $Z_{m,h}=\EE[\EE[e_{I}(x_h)\given z_h]\EE[e_{I}(x_h)\given z_h]^\top],$ where $e_{I}(x_h)$ is the first $m$ eigenfunctions of RKHS with kernel $\cK_{\cR}$. If
\begin{itemize}
    \item $\lambda_{\min}(Z_{m,h})\ge \tau_m,$
    \item $\forall i\le m<j,$  $| \EE[\EE[e_{i}(x_h)\given z_h]\EE[e_{j}(x_h)\given z_h]^\top]|\lesssim \tau_m,$
\end{itemize}
then it holds that $\tau_r^2(\delta)\lesssim \min_{m}(\frac{4\delta^2}{\tau_m}+C\lambda_{m+1}),$ where $\lambda_{m+1}$ is the $(m+1)$-th eigenvalue of $\cK_{\cR},$ by Lemma 11 in \cite{dikkala2020}. 
Specifically, we have the following explicit results on $\tau_r^2(\delta)$.
\begin{itemize}
    \item If $\lambda_{m}\sim m^{-b},\tau_m\ge m^{-a},a,b>0,$ we have $\tau_r^2(\delta)\sim \delta^{\frac{2b}{a+b}}.$
    \item If $\lambda_{m}\sim e^{-m},\tau_m\ge m^{-a},a>0,$ we have $\tau_r^2(\delta)\sim \delta^2\log(1/\delta)^{a}.$
    \item If $\lambda_{m}\sim e^{-bm},\tau_m\ge e^{-am},a>0,$ we have $\tau_r^2(\delta)\sim \delta^{\frac{2b}{a+b}}.$
    \item If $\lambda_{m}\sim m^{-b},\tau_m\ge e^{-m},$ we have $\tau_r^2(\delta)\sim 1/\log(1/\delta)^{2b}.$
\end{itemize}
\end{remark}
The aforementioned remark illustrates that the ill-conditioned coefficient in every stage $h$ is determined by the relative decaying speed of eigenvalues of $\cK_{\cR}$ and $Z_{m,h}.$
With these necessary components at hand, we present the corollary of Theorem \ref{sub_optimality} for the kernel setting as follows.


\begin{corollary}\label{cor_rkhs}
Suppose $R_h^*$ and $G_{h,j}^*,\forall j\in [d_1]$ lie in the function classes \eqref{rkhs_r} and \eqref{rkhs_g}. Under Assumptions \ref{assume:concentrability} and \ref{assume:functionclass}, by constructing  our policy $\hat\pi$ using the same way with \eqref{est_policy}, with function classes defined above, we obtain
\begin{itemize}
    \item If the eigenvalues of all kernels $\cK_{\cR_{h}^*},\cK_{\cG_{h}^*},\cK_{\cF}$ defined in this section decay exponentially, with rate $\exp(-i)$, with probability $1-\delta-1/K,$ we obtain 
\begin{align*}
\textrm{SubOpt}(\hat\pi)&\lesssim \sum_{h=1}^{H}C_{\pi^*}\Bigg[H\sqrt{d_1}\tau_{G}\Bigg(L_{K,d_1}\Bigg(\sqrt{\frac{\log K}{K}}+\sqrt{\frac{\log(d_1/\delta)}{K}}\Bigg)\Bigg)\\&\qquad +\tau_{r}\Bigg(L_{K,1}\Bigg(\sqrt{\frac{\log K}{K}}+\sqrt{\frac{\log(d_1/\delta)}{K}}\bigg)\Bigg)\Bigg].
\end{align*}
\item If the eigenvalues of all kernels decay polynomially, with rate $i^{-\alpha},\alpha>1,$ with probability $1-\delta-1/K,$ we have
\begin{align*}
\textrm{SubOpt}(\hat\pi)&\lesssim \sum_{h=1}^{H}C_{\pi^*}\Bigg[H\sqrt{d_1}\tau_G\Bigg(L_{K,d_1}
\Bigg(\sqrt{\frac{\log K}{K^{\alpha/(\alpha+1)}}}+\sqrt{\frac{\log(d_1/\delta)}{K}}\Bigg)\Bigg)
\\&\qquad+\tau_{r}\Bigg(L_{K,1}\bigg(\sqrt{\frac{\log K}{K^{\alpha/(\alpha+1)}}}+\sqrt{\frac{\log(d_1/\delta)}{K}}\Bigg)\Bigg)\Bigg].
\end{align*}
\end{itemize}
Here $L_{K,x}=L+\sigma\sqrt{(\log H+1)\log (Kx)}$ with $x\in\{1,d_1\}$ and $L$ being the upper bound of all functions in $\RR_h,\mathbb{G}_{h,j},\forall h\in [H],\forall j\in [d_1]$ in $\ell_{\infty}$-norm. 

\end{corollary}

In aforementioned examples, results are established based on the realizability of function classes, which could be restrictive. In the following, we relax such case by allowing function class misspecification.
\subsection{Suboptimality  under Misspecified Functional Classes}\label{sec:misspecified}
In this subsection, we establish theoretical guarantees for \algo\,\, under function class misspecification.  We utilize function classes $\tilde{\RR}_h,\tilde{\mathbb{G}}_{h,j},$ which may not contain the true reward and transition functions, to estimate the model and solve the planning problem. 

In specific, we let
\begin{align}\label{est_policy_2}
    \hat\pi^{M}=\argmax_{\pi \in \Pi}\min_{M\in  \cM^M}J(M,\pi)
\end{align}
where $\cM^M:=\{(\bar{\cG}^M_h,\bar{\cR}^M_h)\}_{h=1}^{H}$ are confidence sets constructed in the same way as \S\ref{const_level_set} via function classes $\tilde{\RR}_h,\tilde{\mathbb{G}}_{h,j},j\in[d_1],h\in[H]$.

In terms of theoretical analysis, we characterize two kinds of the misspecification errors in the following assumption.
\begin{assumption}\label{approximation_error} We present two approximation errors on both primal and dual function classes:
\begin{itemize}
    \item (Primal Function Class) For all $h\in [H],j\in [d_1],$ we have $\min_{R_h\in \tilde{\RR}_h} \|R_h(\cdot)-R_h^*(\cdot)\|_{\infty}\le \eta_{r,h,K},$ $\min_{G_{h,j}\in \tilde{\mathbb{G}}_{h,j}}\|G_{h,j}(\cdot)-G_{h,j}^*(\cdot)\|_{\infty}\le \eta_{G,h,K}.$ Moreover, we assume the corresponding elements that achieve the minimum approximation error to $R_h^*(\cdot),G_{h,j}^*(\cdot)$ in $\ell_{\infty}$-norm exist and are denoted by $R_h^0(\cdot),G_{h,j}^0(\cdot),\forall h\in [H],j\in [d_1]$, respectively.
    \item (Dual Function Class) We assume $\cF:=\{f:\cZ\rightarrow [-L,L]\}$ is a symmetric and star-shaped function class and 
 \begin{align*}
    \forall R_h\in \tilde{\RR}_h, :\min_{f\in \cF}\big\|f(s_h,a_h)-\mathbb{T}\big(R_h-R_h^0\big)(s_h,a_h)\big\|_2&\le \xi_{r,h,K},\\\forall  j\in [d_1], \forall G_{h,j}\in \tilde{\mathbb{G}}_{h,j}, \min_{f\in \cF}\big\|f(s_h,a_h)-\mathbb{T}\big(G_{h,j}-G_{h,j}^0\big)(s_h,a_h)\big\|_2&\le \xi_{G,h,K},
\end{align*}
in which the probability measure is $\rho$.
\end{itemize}
\end{assumption}
For all $R_h\in \RR_h,G_{h,j}\in \mathbb{G}_{h,j}$, we let $f_{R_h}=\argmin_{f\in \cF}\|f-\mathbb{T}(R_h-R_h^0)\|_2,f_{G_{h,j}}=\argmin_{f\in \cF}\|f-\mathbb{T}(G_{h,j}-G_{h,j}^0)\|_2.$ Meanwhile, for all $h\in [H]$, we denote 
\begin{align}
\tilde{\cR}_h^{*,M}:&=\Big\{ c(R_h-R_h^0)(x)\cdot f_{R_h}(z):\cX\times\cZ\rightarrow \RR; R_h\in \tilde{\mathbb{R}}_h,\forall c\in [0,1] \Big\},\label{rhm}\\
\tilde{\cG}_{h,j}^{*,M}:&=\Big\{ c(G_{h,j}-G_{h,j}^0)(x)\cdot f_{G_{h,j}}(z):\cX\times\cZ\rightarrow \RR; G_{h,j}\in \tilde{\mathbb{G}}_{h,j},\forall c\in [0,1] \Big\}, \forall j\in [d_1].\label{ghm}
\end{align}
The correponding  thresholding parameters ${c}_{r,h,K}$ and ${c}_{G,h,K}$ involved in the high-confidence sets given in \S\ref{const_level_set} are specified as $c_{h,r,K}=\cO(\delta_{\cR,h,\delta}+\eta_{r,h,K})=\cO(\delta_{\cR,h}+\sqrt{\log(H/\delta)/K}+\eta_{r,h,K})$ and ${c}_{G,h,K}:=\cO(\delta_{\cG,h,\delta}+\eta_{G,h,K})=\cO(\delta_{\cG,h}+\sqrt{\log(Hd_1/\delta)/K}+\eta_{G,h,K}),$ respectively. Here $ \delta_{\cR,h}$ is maximum the critical radii of $\cF$, $\tilde{\cR}_h^{*,M}$ and $\delta_{\cG,h}$ is the maximum critical radii of $\cF$ and $\tilde{\cG}_{h,j}^{*,M},\forall j\in[d_1]$. 

Combining all assumptions mentioned above, we present the following theorem on the suboptimality of $\hat \pi ^{M}$.
\begin{theorem}
\label{sub_optimality_approx}
Under Assumptions \ref{assume:concentrability}, \ref{ill-pose} (with replacing $R_h^*,G_{h,j}^*$ by $R_h^0$ and $G_{h,j}^0,$ $\forall h\in [H],j\in [d_1]$), and Assumption \ref{approximation_error}, with probability $1-\delta-1/K$, the suboptimality is upper bounded by 
\begin{align*}
\textrm{SubOpt}(\hat\pi^M)&\lesssim C_{\pi^*}\Bigg[H\sum_{h=1}^{H}\tau_{K,G,h}\sqrt{d_1}L_{K,d_1}\Bigg(\delta_{\cG,h,\delta}+\xi_{G,h,K}+\frac{\eta_{G,h,K}^2}{\delta_{\cR,h,\delta}}+\eta_{G,h,K}\Bigg)\\
&\qquad\qquad+\sum_{h=1}^{H}\tau_{K,r,h}L_{K,1}\Bigg(\delta_{\cR,h,\delta}+\xi_{r,h,K}+\frac{\eta_{r,h,K}^2}{\delta_{\cG,h,\delta}}+\eta_{r,h,K}\Bigg)\Bigg],
\end{align*}
where $L_{K,x}=L+\sigma\sqrt{(\log H+1)\log (Kx)},$ with $x\in\{1,d_1\}$,
$\delta_{\cR,h,\delta}=\delta_{\cR,h}+\sqrt{\log(H/\delta)/K},$ and $\delta_{\cG,h,\delta}=\delta_{\cG,h}+\sqrt{\log(Hd_1/\delta)/K}$. 
Moreover, if we choose proper function classes such that $\delta_{\cG,h,\delta}=\Theta(\eta_{G,h,K})$ and $\delta_{\cR,h,\delta}=\Theta(\eta_{r,h,K}),$ the aforementioned upper bound reduces to  
\begin{align*}
\textrm{SubOpt}(\hat\pi^M)&\lesssim C_{\pi^*}\Bigg[H\sum_{h=1}^{H}\tau_{K,G,h}\sqrt{d_1}L_{K,d_1}\Big(\delta_{\cG,h,\delta}+\xi_{G,h,K}\Big)\sum_{h=1}^{H}\tau_{K,r,h}L_{K,1}\Big(\delta_{\cR,h,\delta}+\xi_{r,h,K}\Big)\Bigg].
\end{align*}
\end{theorem}
Theorem \ref{sub_optimality_approx} presents an upper bound for the suboptimality under misspecified function classes. Besides the critical radii of involved function classes, the upper bound also involves the approximation errors induced by misspecification. If one chooses approximation function classes properly, one is able to balance the bias and variance and achieve minimax optimal statistical rates of estimating functions in the true function class. To illustrate this idea, we provide a concrete example using the class of neural networks to approximate the Sobolev ball in Appendix \S\ref{example:nn}.
\section{Applications}\label{application}
This section discusses two applications of strategic MDP, including strategic regression, strategic bandits. 
\subsection{Strategic Regression}\label{strategic_regresssion}
In this subsection, we show that the strategic regression proposed by \cite{stevenwu2021}, where the principal interacts with an agent in one round with no state transition, is a special case of the strategic MDP.  Their setting admits a specific example on college exams:
\begin{itemize}
    \item The principal (school) announces some information on evaluation criterion $a$ before the final exam, such as ``At least 50\% of problems in the final exam are about math".
    \item The agent (student) maximizes its reward by taking some action $b\in \RR^d$ such as reviewing math knowledge given the information $a\in \RR^d$ and its private information $i:=(z,W)$. Here $z\in \RR^d$ denotes the baseline ability of that student, and $W\in \RR^{d\times d}$ represents the effort transition matrix. In specific, the agent takes  $$b=\argmax_{b}R_a(b,a,i),$$ where $R_a(b,a,i)=(z+Wb)^\top a-\frac{1}{2}\|b\|_2^2$. In simple terms, we have $b=W^\top a.$
    \item Principal observes the feature manipulated by the student $o=(z+Wb)=(z+WW^\top a).$
    \item Principal obtains reward $r=o^\top \theta^*+g$ which represents the overall score the student obtain (including math).
    Here $g$ is a mean zero random variable that is correlated with the baseline ability $z$ of that student but is independent with $a$. 
    \end{itemize}

We next design a  policy of the principal to maximize overall score $\EE[r]$ for a specific population of agents (e.g., overall score of students who major in literature). In this scenario, the distribution $P(\cdot)$ of private information $i$ of this population is assumed to be known in the planning stage. When faced with such a population, the marginalized reward function of the principal is given by:
\begin{align*}
    \bar{R}(a)=\EE_{o\sim P(\cdot\given a)}\Big[o^\top\theta \given a\Big]=\EE_{(z,W)\sim P(\cdot)}\Big[(z+WW^\top a)^\top\theta \given a\Big],
\end{align*}
where $P(o\given a)$ is the conditional distribution of $o$ given $a$ when $i\sim P(\cdot).$ Moreoever, the function class that contains the true reward function $R^*(\cdot)=\langle \cdot,\theta^*\rangle$ is defined as: $$\RR_1=\Big\{ \langle \theta, o \rangle: o\rightarrow\RR;\|\theta\|_2\le B,\theta\in \RR^{d}\Big\},$$
which is a special case of our example in  \S\ref{linear_function}.
In this scenario, the value function evaluated by policy $\pi$ under the true model $\bar{R}^*(a)=\EE_{o\sim P(\cdot\given a)}[o^\top\theta^* \given a]$ reduces to
\begin{align*}
   J(\bar{R}^*,\pi):=\EE_{a\sim \pi }\big[\bar{R}^*(a)\big].
\end{align*}
Let $\pi^*$ be the optimal policy that maximizes $J(\bar{R}^*,\pi)$. Therefore, for any given policy $\pi\in \Delta(\cA),$ the suboptimality  of $\pi$ is given by 
\begin{align}\label{sub_opt}
   \textrm{SubOpt}(\pi)=J(\bar{R}^*,\pi^*)-J(\bar{R}^*,\pi). 
\end{align}
Our goal is to design a certain policy that minimizes this suboptimality.

In this application, we assume the offline data $\{a_t,o_t,r_t\}_{t=1}^{K},$ are collected i.i.d. with $(a_t,o_t)\sim \rho(a,o)$. By our model formulation and Definition \ref{def_iv}, we see that $a$ serves as the instrumental variable for $(o,r)$ \citep{stevenwu2021}. 

Next, we apply \algo~with  $H=1$ and linear functions (introduced in  \S\ref{linear_function}), covariate $x=o$, and instrumental variable $z=a$ to derive a policy $\hat\pi$. It is worth noting that the ill-posed coefficient under the setting reduces to
\begin{align*}
    \sup_{x\in \RR^{d}}\frac{x^\top\EE_{o\sim \rho(o)}[oo^\top]x}{x^\top\EE_{a\sim \rho(a)}[\EE[WW^\top]aa^\top\EE[WW^\top]^\top]x}\le \tau_1^2.
\end{align*}
 With these necessary tools at hand, we summarize our conclusion in the following Proposition \ref{prop_reg}.

\begin{proposition}\label{prop_reg}
Given the settings in \S\ref{strategic_regresssion} and Assumptions \ref{assume:concentrability} and \ref{ill-posed-linear}, with probability $1-\delta-1/K$, we have 
\begin{align*}
        \textrm{SubOpt}(\hat\pi)\lesssim \tau_1 C_{\pi^*}L_{K}\Bigg(\sqrt{\frac{d\log K}{K}}+\sqrt{\frac{\log (c_0/\delta)}{K}}\Bigg).
\end{align*}
Here $L_{K}=L+\sigma\sqrt{\log (K)}$ with $L$ being the upper bound of all functions in $\RR_1$ in $\ell_{\infty}$-norm. 
\end{proposition}
We deduce from Proposition \ref{prop_reg} that the suboptimality is proportional to $\cO(\sqrt{d/K})$, which is minimax optimal in the linear class.

\subsection{Strategic Bandit}\label{strategic_bandit}

In this section, we apply \algo\,\,to strategic bandit problem where the principal interacts with the agents by taking actions in two rounds. One possible real-world application is the interaction between some information providers (stock information provider, medical treatment specialist,  insurance company, and etc) and strategic agents \citep{sayin2018dynamic}. 
For example, in the stock market, if some information provider distributes some information about the future stock market, the strategic agents will manipulate their features by selling or buying more stocks given such information. In this scenario, the information provider will profit by taking a second action based on the agents' manipulated features. Besides, other real-world examples such as the interaction between a  medical treatment provider or an insurance company and strategic agents also fit in this setting. 

We call this framework as strategic bandit mathematically and formulate it as follows.

\begin{definition}[Strategic Bandit]  The interaction protocol is specified as follows:
\begin{itemize}
    \item The principal first announces an action $a_{1}$. 
    \item Given the principal's first action $a_{1}$, a myopic agent with private type $i\sim P(\cdot)$ takes an action to maximize its immediate reward: $b=\argmax_{b}R_a^*(a_{1},i,b)$.
    \item The principal receives an observation $o\sim F(\cdot\given a_{1},i):=F_{a}(\cdot\given b,i)$ that is the feature generated by the agent based on its private type $i$ and the principal's  action $a_{1}$ (through $b$). 
    \item The principal takes the second action $a_{2}$ based on observation $o$ 
    and receives a reward $r=R^*(o,a_{2})+g,$ where $g$ is a mean zero random variable but is affected by the private information $i$ of the agent. 
\end{itemize}
\end{definition}


Our goal is learning the optimal policy of the principal i.e., maximizing the principal's expected reward $\EE[r]$ for a specific population of agents, whose distribution of private type $P(\cdot)$, private reward function $R_a(\cdot)$ and $F(\cdot\given a_{1},i)$ are known to the principal in the planning stage. 

{ We consider the function class $\RR_1$  with $R^*\in\RR_1.$  For any given reward model $R(\cdot)\in \RR_1$ and policy $\pi\in \Pi=\{\pi_1\times\pi_2:\pi_1\in \Delta(\cA_1),\pi_2:\cO\rightarrow\Delta(\cA_2)\}$, under a specific population of agents, the value function is defined as follows:
\begin{align*}
    J(R,\pi)=\int [R(o,a_2)]\ud \pi(a_2\given o)\ud F(o\given a_1,i)\ud P(i)\ud \pi(a_1).
\end{align*}
Let $\pi^*$ be the optimal policy that maximizes $J(R^*,\pi).$ For any given policy $\pi$, we define the suboptimality as 
\begin{align}\label{sub_opt}
   \textrm{SubOpt}(\pi)=J(\bar{R}^*,\pi^*)-J(\bar{R}^*,\pi). 
\end{align}

In this application, we assume our offline data $\{a_{1t},o_t,a_{2t},r_t\}_{t=1}^{K}$ is collected i.i.d. with $(a_{1t},o_t,a_{2t})\sim \rho(a_1,o,a_2)$. 
From the statement of model formulation given above, one observes that $a_{1}$ serves as an instrumental variable for $(o_t,a_1,r)$ by Lemma \ref{lem-instrument} based on Definition \ref{def_iv}.

The way of constructing confidence sets and policy $\hat\pi$ are almost the same with that in \S\ref{const_level_set}, by treating $(o,a_2)$ as the covariate and $a_1$ as the instrumental variable. Therefore,  we omit the relevant details. 
We next state the assumptions on concentrability coefficient in order to establish theoretical guarantees for $\hat\pi$.
\begin{assumption}[Concentrability Coefficients]\label{bandit_concent}
We assume 
\begin{align*}
\sup_{a_{1},a_{2},o}\frac{ f_{\pi^*}(a_{2}\given o)f_p(o\given a_{1})f_{\pi^*}(a_{1})}{f_{\rho}(o,a_{1},a_{2})}\le C_{\pi^*}^2.
\end{align*}
\end{assumption}
In this assumption, we let the ratio of two densities, namely, the joint density of $(a_1,o,a_2)$ induced by the optimal policy for a specific population of agents and the sampling density $f_{\rho},$ be bounded by a concentrability coefficients $C_{\pi^*}^2$. Here $f_p(o\given a_{1})=\int_{i} \ud F(o\given a_1,i)\ud P(i)$ and is assumed  known in the planning stage.  

Besides, other assumptions on the ill-posed condition and completeness of function class $\cF$ are also almost the same with Assumptions \ref{ill-pose} and \ref{assume:functionclass} by treating $(o,a_{2})$ as the covariate and $a_{1}$ as the instrumental variable. In specific, we let $\tau_1$ be the associated ill-posed coefficient such that $$\sup_{R\in \{\RR_1-R^*\}}\frac{\EE_{\rho}[R^2(o_{t},a_{2t})]}{\EE_{\rho}[ \EE_{\rho}[R(o_{t},a_{2t})\given a_{1t}]^2]}\le \tau_1^2. $$
Combining these assumptions, we provide an upper bound on the suboptimality of  $\hat\pi$ constructed by \algo\,\,in the following proposition.
\begin{proposition}\label{Prop_bandit}
Under Assumption \ref{bandit_concent}, Assumptions \ref{ill-pose} and \ref{assume:functionclass} (treating $(o,a_{2})$ as the covariate and $a_{1}$ as the instrumental variable), with probability $1-\delta-1/K$, we obtain
\begin{align*}
     \textrm{SubOpt}(\hat\pi)\lesssim C_{\pi^*}\tau_{1}\delta_{\cR,\delta},
\end{align*}
where $\delta_{\cR,\delta}=\delta_{\cR}+\sqrt{\log(c_0/\delta)/K}$ with $ \delta_{\cR}$ being the upper bound of the critical radii of $\cF$ and $\cR^*:=\{(a_{2t},o_t,a_{1t})\rightarrow c(R(a_{2t},o_t)-R^*(a_{2t},o_t))\cdot\EE[R(a_{2t},o_t)-R^*(a_{2t},o_t)\given a_{1t}],\forall R\in \mathbb{R}_1,\forall c\in [0,1] \}.$ 
\end{proposition}

In addition to strategic regression and bandits, our algorithm \algo\,\,also works for the principal when noncompliant agents exist in the recommendation system, as discussed in Example \ref{example_noncompliant}. We present the corresponding paragraph in the Appendix \S\ref{no-compliant} .

\section{Conclusion}
In this paper, we study multi-agent reinforcement learning with information asymmetry and propose a general framework named strategic MDP to model the interaction between the principal and strategic agents. We design a provably efficient algorithm \algo\, using instrumental variable regression and pessimism principal with general function approximation to handle the challenges of studying the strategic MDP, including the existence of unobserved confounders and distribution shift. We prove that \algo\, outputs near optimal policy under mild conditions. Our framework also admits several real-world examples such as strategic regression, strategic bandits, and non-compliance agents in recommendation systems as exceptional cases.

There are a few future directions that are worth exploring. First, it would be interesting to study the online setting of strategic MDP. In addition, we consider the interaction between one principal with multiple agents. It will be appealing to extend strategic MDP to the setting with multiple principals that are either cooperative or competitive. 





\newpage
\appendix

\section{Additional Related Works}
 
Our work is also related to the literature on offline RL in stochastic games. 
Most of the existing works focus on  two-player zero-sum stochastic games. See, e.g.,
\cite{lagoudakis2012value, perolat2015approximate, perolat2016use, fan2019theoretical, zhao2021provably,perolat2016softened,  alacaoglu2022natural, cui2022offline, zhong2022pessimistic, xiong2022nearly, yan2022model, cui2022provably} and the references therein. 
In addition, 
\cite{perolat2017learning, cui2022provably} study offline RL for general-sum stochastic games and \cite{zhong2021can} study Stackelberg equilibria  in stochastic games with myopic followers via offline RL. 
All of these works consider the settings where the learner has complete information, and thus do not face the issue of unobserved confounders. 
Furthermore, there is a recent line of research that develops decentralized  online reinforcement learning algorithms for  stochastic games, where the goal is to learn the optimal policy from the perspective of a single player  based on its local information \citep{wei2017online, xie2020learning, tian2021online,
song2021can, jin2021v, sayin2021decentralized, liu2022learning,zhan2022decentralized,kao2022decentralized, mao2022improving}. 
In such a decentralized setting, it shown in \cite{liu2022learning} that, when the actions of the opponent are hidden,  learning the optimal policy in hindsight against an arbitrarily adversarial opponent is statistically intractable.    
Thus, existing works mainly focus on either learning the value of the game \citep{wei2017online, xie2020learning, tian2021online}, or  the optimal policy based on some side information \citep{liu2022learning,zhan2022decentralized}, or finding Nash or correlated equilibria in the tabular setting \citep{song2021can, jin2021v, sayin2021decentralized,kao2022decentralized, mao2022improving}.
Compare to these works, we also design a learning  algorithm for a single player, namely the principal, who is unaware of the private information of the agents. 
However, we focus on the offline setting where the private information brings the challenge of confounding. Thus, our work is not directly comparable to these works. 

Meanwhile, our work adds to the rapidly growing literature on strategic classification. 
See, e.g., 
\cite{dalvi2004adversarial, bruckner2011stackelberg,alacaoglu2022natural,bruckner2012static, hardt2016strategic,  dong2018strategic, hu2019disparate, 
milli2019social,miller2020strategic,
ghalme2021strategic, ahmadi2021strategic, levanon2021strategic,zrnic2021leads, nair2022strategic}
and the references therein. 
In addition to classification, 
various works study regression and ranking models under  strategic manipulations 
\citep{liu2016bandit, shavit2020causal,bechavod2020causal,gast2020linear,  stevenwu2021, harris2021stateful, liu2022strategic}. 
Moreover, there is a line of research on  performative prediction, which is a general framework of machine learning models with strategic data
\citep{perdomo2020performative,mendler2020stochastic,  miller2021outside,narang2022multiplayer,brown2022performative, li2022state, jagadeesan2022regret}. 
These works can all be formulated as a stackelberg game between the machine learning model and an agent, where the machine learning models performs certain prediction about the agent, while the agent strategically manipulates its feature in order to maximize its own reward. 
Among these works, our work is particularly   related to \cite{miller2020strategic, stevenwu2021,shavit2020causal}, which study strategic classification and regression from a causal inference perspective. 
In particular, under the setting of strategic  regression, \cite{stevenwu2021} show that the announced regression model serves as a valid instrument. 
Our work generalize such a key observation to the dynamic setting with sequential interactions. 
 


\section{Proof of Results in \S\ref{sec:method}}
In this section, we provide the proof of 
Lemma \ref{lem-instrument}.
\subsection{Proof of Lemma \ref{lem-instrument}}
\begin{proof}[Proof of Lemma \ref{lem-instrument}]	
	We prove that for any given $h\in [H], (s_h,a_h)$  serves as an instrumental variable for $(\phi_x(s_h,a_h,o_h),r_h)$ and $(\phi_x(s_h,a_h,o_h),s_{h+1})$. We see that  $Z_h:=(s_h,a_h)$ affects $X_h:=\phi_x(s_h,a_h,o_h)$ by the definition of $o_h,$ which justifies the first point of Definition \ref{def_iv}. To demonstrate the second point, we observe that $Z_h:=(s_h,a_h)$ is independent with $g_h,\xi_h$ by our Definition \ref{strategic_mdp} and Assumption \ref{indenp_priviate} and thus affects $(r_h,s_{h+1})$  only through $X_h$. Thus, the second point of Definition \ref{def_iv} is also satisfied. 
Combining the arguments given above, we conclude our proof of Lemma \ref{lem-instrument}. 
\end{proof}
\section{Proof of Theorem \ref{sub_optimality}}
In this section, we will prove Theorem \ref{sub_optimality}, which establishes the suboptimality of $\hat\pi$ returned by \algo. 

\begin{proof}[Proof of Theorem \ref{sub_optimality}]
First, we prove that for any $h\in [H]$,  $R_h^{*}(\cdot),$ $G_h^*(\cdot)$ belong to the confidence sets $\cR_h,\cG_h$ with high probability, respectively. We summarize the results in the following lemma.

\begin{lemma}\label{contain_true}
Given the confidence sets $\cR_h,\cG_{h,j}$ for $R_h^*(\cdot)$ and $G_{h,j}^*(\cdot),\forall h\in [H],j\in [d_1]$ constructed in \eqref{confidence_r} and \eqref{confidence_g}, for any $h\in [H],$ with probability $1-\delta-1/(d_1K)^{\beta}$, we have
\begin{align*}
R_h^{*}\in \cR_h&=\Big\{R_h\in \RR_h: \cL_n(R_h)-\cL_n(\hat R_h)\lesssim L_{K,\beta,1}^2\delta_{\cR,h,\delta}^2 \Big\},\\
G_h^*\in \cG_h&=\Big\{G_h\in (\cG_{h,1},\cG_{h,2},\cdots,\cG_{h,d_1})\Big\},\\\textrm{ where }  \cG_{h,j}&=\{G_{h,j}\in \mathbb{G}_{h,j}: \cL_n(G_{h,j})-\cL_n(\hat G_{h,j})\lesssim L_{K,\beta,d_1}^2\delta_{\cG,h,\delta}^2 \Big\}.
\end{align*}
 Here we let $\delta_{\cR,h,\delta}=\delta_{\cR,h}+c_1\sqrt{\log(c_0/\delta)/K},\delta_{\cG,h,\delta}=\delta_{\cG,h}+c_2\sqrt{\log(c_0d_1/\delta)/K}$ and $L_{K,\beta,d_1}:=L+\sigma\sqrt{(\beta+1)\log Kd_1}$ with $c_0,c_1,c_2$ being absolute constants. Moreover, $\delta_{\cG,h}$ is an upper bound of the maximum critical radius of $\cF,\cG^{*}_{h,j}\forall j\in[d_1]$ (defined in \eqref{cg_hj}) and $\delta_{\cR,h}$ is an upper bound of the maximum radius of $\cF,\cR_h^*$ (defined in \eqref{cr_h}). 
\end{lemma}
\begin{proof}
See \S\ref{subsec:contain_true} for a detailed proof.
\end{proof}

After proving that the true model lies in the confidence sets with high probability, we next show that the suboptimality given in \eqref{sub_opt} is bounded by the estimation errors of the pessimistic models. We summarize these results in the following Lemma \ref{Pessmistic}.

\begin{lemma}\label{Pessmistic}
We let $\tilde{M}:=\{(\bar{\tilde{R}}_h,\bar{\tilde{\PP}}_h)\}_{h=1}^{H}=\argmin _{M\in\cM}V_{M}^{\pi^*},$ where $\cM$ is the confidence set defined in \eqref{est_policy}. Then we have 
\begin{align}
J(M^*,\pi)-J(M^*,\hat\pi)&\le H\sum_{h=1}^{H}\EE_{a_h\sim \pi^*(\cdot \given s_h),\,s_h\sim d^{\pi^*}_{h-1}}\bigg[\Big\|\Big(\bar{\tilde{ \PP}}_h-\bar{\PP}_h^*\Big)(\cdot\given s_h,a_h)\Big\|_1\bigg]\notag \\&\qquad+\sum_{h=1}^{H}\EE_{a_h\sim \pi^*(\cdot \given s_h),\,s_h\sim d^{\pi^*}_{h-1}}\bigg[\Big|\bar{\tilde{ R}}_h(s_h,a_h)-\bar{R}_h^{*}(s_h,a_h)\Big| \bigg], \label{theory_subopt}
\end{align}
where $d^{\pi^*}_{h-1}$ represents the distribution of state $s_h$, which is generated by following policy $\pi^*$ and the true transition functions $\{\bar{\PP}_j^*\}_{j=1}^{h-1}$ defined in \eqref{transit_new2}.
\end{lemma}
\begin{proof}
See \S\ref{subsec:pessmistic} for a detailed proof.
\end{proof}
Next, we prove that for all functions in the confidence sets $\cR_h,\cG_h,\forall h\in [H]$, their projected MSEs are small. We summarize this property in the following lemma.

\begin{lemma}\label{project_MSE}
For all $R_h,G_h$ given in $\cR_h,\cG_h$, with  probability $1-H\delta-H/(Kd_1)^{\beta}$, we have
\begin{align}
\EE_{\rho}\bigg[\Big(\EE_{\rho}\big[R_h(s_h,a_h,o_h)-R_h^{*}(s_h,a_h,o_h) \given s_h,a_h\big]\Big)^2\bigg]&\lesssim L_{K,\beta,1}^2\delta_{\cR,h,\delta}^2, \label{radius_r}\\
\EE_{\rho}\bigg[\Big\|\EE_{\rho}\big[G_h(s_h,a_h,o_h)-G_h^{*}(s_h,a_h,o_h) \given s_h,a_h\big]\Big\|_2^2\bigg]&\lesssim d_1L_{K,\beta,d_1}^2\delta_{\cG,h,\delta}^2. \label{radius_g}
\end{align}
\end{lemma}
\begin{proof}
See \S\ref{subsec:PMSE} for a detailed proof.
\end{proof}
Finally, we utilize conclusions from Lemma \ref{contain_true}, Lemma \ref{Pessmistic} and Lemma \ref{project_MSE} to prove Theorem \ref{sub_optimality}. 
Specifically, by Lemma \ref{Pessmistic}, we have
\begin{align}
 \mathbf{(I)}:&= \EE_{a_h\sim \pi^*(\cdot \given s_h),\,s_h\sim d^{\pi^*}_{h-1}}\bigg[\Big\|(\bar{\tilde{ \PP}}_h-\PP_h^*)(\cdot\given s_h,a_h)\Big\|_1\bigg]\notag \\
   &= \EE_{a_h\sim \pi^*(\cdot \given x_h),\,s_h\sim d^{\pi^*}_{h-1}}\Bigg[\bigg\|\int_{i_h,o_h} \tilde{\PP}(\cdot\given s_h,a_h,o_h,i_h)- \PP^*(\cdot\given s_h,a_h,o_h,i_h)\ud F_h(o_h\given s_h,a_h,i_h)\ud P_h(i_h)\bigg\|_1 \Bigg]\notag\\& \le \EE_{a_h\sim \pi^*(\cdot \given s_h),\,s_h\sim d^{\pi^*}_{h-1}}\bigg[\int_{i_h,o_h}\Big\|\ \tilde{\PP}(\cdot\given s_h,a_h,o_h,i_h)- \PP^*(\cdot\given s_h,a_h,o_h,i_h)\Big\|_1\ud F_h(o_h\given s_h,a_h,i_h)\ud P_h(i_h) \bigg]\notag
   \\& \lesssim \sqrt{\EE_{a_h\sim \pi^*(\cdot \given x_h),\,s_h\sim d^{\pi^*}_{h-1},i_h\sim P_h(\cdot),o_h\sim F_h(\cdot\given s_h,a_h,i_h)}\bigg[\textrm{TV}\bigg(\tilde{ \PP}_h(\cdot\given s_h,a_h,o_h,i_h),\PP_h^*(\cdot \given s_h,a_h,o_h,i_h)\Big)^2\bigg]}\notag
\end{align}
Here the first inequality follows from Jensen's inequality.
In addition, the second inequality follows from Cauchy-Schwartz inequality and the relation between total variation distance and $\ell_1$-distance, namely, TV$(P,Q)=\|P-Q\|_1/2$ for any two probability measures $P,Q$ \citep{levin2017markov}. Next, we leverage our assumption on Gaussian transition in \S\ref{sec:algo} to derive an upper bound for $\mathbf{(I)}.$ To be more specific, we have
\begin{align}
   \mathbf{(I)}&  \lesssim\sqrt{\EE_{a_h\sim \pi^*(\cdot \given s_h),\,s_h\sim d^{\pi^*}_{h-1},o_h\sim P_h(\cdot\given s_h,a_h)}\bigg[\Big\|\tilde{G}_h( s_h,a_h,o_h)-G_h^*( s_h,a_h,o_h)\Big\|_2^2\bigg]}\notag
    \\& \lesssim C_{\pi^*} \sqrt{\EE_{\rho(s_h,a_h,o_h)}\bigg[\Big\|\tilde{G}_h( s_h,a_h,o_h)-G_h^*( s_h,a_h,o_h)\Big\|_2^2\bigg]}\notag
    \\ & \lesssim C_{\pi^*} \tau_{K,G,{h}}\sqrt{\EE_{\rho(s_h,a_h)}\bigg[\Big\|\EE_{\rho(o_h\given s_h,a_h)}\big[\tilde{G}_h(s_h,a_h,o_h)-G_h^{*}(s_h,a_h,o_h) \given s_h,a_h\big]\Big\|_2^2\bigg]}\notag
    \\ &\lesssim C_{\pi^*} \tau_{K,G,{h}}\sqrt{d_1}L_{K,\beta,d_1}\delta_{h,\cG,\delta}.\label{proof_g}
\end{align}
Here ${ P_h(\cdot\given s_h,a_h)=\int_{i_h}F_h(\cdot \given s_h,a_h,i_h)\ud P_h(i_h)}$ denotes the conditional distribution of $o_h$ given state and action $(s_h,a_h)$ in the planning stage. Here, the first inequality follows from our model assumption on Gaussian transition given $(s_h,a_h,o_h,i_h)$ by Theorem 1.2 of \cite{devroye2018total}. 
The second inequality follows from our Assumption \ref{assume:concentrability}, where we shift the probability measure to our sample distribution $\rho$. The third inequality follows from the ill-posed condition in Assumption \ref{ill-pose}. Finally, the last inequality follows from \eqref{radius_g}.
Similarly, to bound the second term of suboptimality (differences between reward functions) given in \eqref{theory_subopt}, we obtain
\begin{align}
    &\EE_{a_h\sim \pi^*(\cdot \given s_h),\,s_h\sim d^{\pi^*}_{h-1}}\Big[\big|\bar{\tilde{ R}}_h(s_h,a_h)-\bar{R}_h^{*}(s_h,a_h)\big| \Big]\notag\\ &\qquad \le
    \EE_{a_h\sim \pi^*(\cdot \given s_h),\,s_h\sim d^{\pi^*}_{h-1}}\bigg[\int_{i_h,o_h}\big|(\tilde{R}_h-{R}_h^*)(s_h,a_h,o_h)\big|\ud F_h(o_h\given s_h,a_h,i_h)\ud P_h(i_h)\bigg]\notag
    \\&\qquad\lesssim \sqrt{\EE_{a_h\sim \pi^*(\cdot \given s_h),\,s_h\sim d^{\pi^*}_{h-1},o_h\sim P_h(o_h\given s_h,a_h)}\Big[\big|\tilde{ R}_h(s_h,a_h,o_h)-R_h^{*}(s_h,a_h,o_h)\big|^2 \Big]}\notag\\&
    \qquad\lesssim C_{\pi^*}\sqrt{\EE_{\rho(s_h,a_h,o_h)}\Big[|(\tilde{R}_h-{R}_h^*)(s_h,a_h,o_h)\big|^2\Big]}\notag
    \\&\qquad \lesssim C_{\pi^*}\tau_{K,r,h}\sqrt{\EE_{\rho(s_h,a_h)}\bigg[\Big(\EE\big[\tilde{R}_h(s_h,a_h,o_h)-R_h^{*}(s_h,a_h,o_h) \given s_h,a_h\big]\Big)^2\bigg]}\notag\\&\qquad \lesssim C_{\pi^*}\tau_{K,r,h}L_{K,\beta,1}\delta_{h,\cR,\delta}.\label{proof_r}
\end{align}
The first and second inequalities follow from the Jensen's and Cauchy Schwartz inequalities, respectively. The third inequality follows from our Assumption \ref{assume:concentrability}. Moreover, the last but one inequality  holds by the ill-posed condition in Assumption \ref{ill-pose}. Finally, the last one holds by Lemma \ref{project_MSE}.

Finally, leveraging our conclusions obtained from Lemma \ref{Pessmistic}, we have
\begin{align*}
    \textrm{SubOpt}(\hat\pi)\le C_{\pi^*}\Bigg[H\sum_{h=1}^{H}\tau_{K,G,h}\sqrt{d_1}L_{K,\beta,d_1}\delta_{h,\cG,\delta}+\sum_{h=1}^{H}\tau_{K,r,h}L_{K,\beta,1}\delta_{h,\cR,\delta}\Bigg],
\end{align*}
with probability $1-H\delta-H/(Kd_1)^{\beta}.$ By utilizing the union bound and taking $\delta'=\delta/H$ and $\beta=\log H$, we conclude the proof of Theorem \ref{sub_optimality}.
\end{proof}
In the following subsections, we provide detailed proofs  of Lemma \ref{contain_true}, Lemma \ref{Pessmistic} and Lemma \ref{project_MSE}, respectively.
\subsection{Proof of Lemma \ref{contain_true}}\label{subsec:contain_true}
In this subsection, we will prove Lemma \ref{contain_true} by demonstrating that the true model $\{R_h^*,\PP_h^*\}_{h=1}^{H}$ lies in our constructed confidence sets with high probability.
Here, without loss of generality, we assume $G_h(\cdot),R_h(\cdot),\forall h\in [H]$ are functions that maps $\cX$ to $\RR^{d_1}$ with $d_1=1$. Otherwise we can construct confidence sets for every dimension of $G_h(\cdot)$ and $R_h(\cdot)$ following the similar proof. 

\begin{proof}
Here we only prove that $R_h^*(\cdot)$ falls in $\cR_h.$ One can derive the corresponding proof for every coordinate of $G_h^*(\cdot)$ similarly.   To simplify the notation, we use $x_h^{(k)}=(s_h^{(k)},a_h^{(k)},o_h^{(k)})$ and let $z_h^{(k)}=(s_h^{(k)},a_h^{(k)})$ 
in the following proof.  It is worth noting that here we omit the embedding functions $\phi_x(\cdot),\psi_z(\cdot)$ to simply the notation.

First, by the definition of $\cL_K$, we have
\begin{align}
    \cL_K(R_h^*)-\cL_K(\hat R_h)&=\sup_{f\in \cF}\Bigg\{\frac{1}{K}\sum_{k=1}^{K}\Big(R_h^*(x_h^{(k)})-r_h^{(k)}\Big)f\Big(z_h^{(k)}\Big)-\frac{1}{2K}\sum_{k=1}^{K}f^2\Big(z_h^{(k)}\Big) \Bigg\}\nonumber\\&\qquad-\sup_{f\in \cF}\Bigg\{\frac{1}{K}\sum_{k=1}^{K}\Big(\hat R_h(x_h^{(k)})-r_h^{(k)}\Big)f\Big(z_h^{(k)}\Big)-\frac{1}{2K}\sum_{i=1}^{K}f^2\Big(z_h^{(k)}\Big) \Bigg\}\nonumber\\&:=\mathbf{(I)}-\mathbf{(II)}. \label{i_ii}
\end{align}
In order to obtain an upper bound of $\cL_K(R_h^*)-\cL_K(\hat R_h),$ we establish an upper bound of $\mathbf{(I)}$ and a lower bound of $\mathbf{(II)}.$

We first establish an upper bound for $\mathbf{(I)}.$ To simplify the notation, we define
\begin{align*}
    \Phi_K(R_h,f)&:=\frac{1}{K}\sum_{k=1}^{K}\Big(R_h\Big(x_h^{(k)}\Big)-r_h^{(k)}\Big)f\Big(z_h^{(k)}\Big),\textrm{ and}\\
    \Phi(R_h,f)&:=\EE\big[(R_h(x_h)-r_h)f(z_h)\big].
\end{align*}
We then have
\begin{align*}
    \mathbf{(I)}=\sup_{f\in \cF}\Bigg\{\Phi_K(R_h^*,f)-\frac{1}{2K}\sum_{k=1}^{K}f^2\Big(z_h^{(k)}\Big) \Bigg\}.
\end{align*}
In the following Lemma, we derive a concentration inequality for $\Phi_K(R_h^*,f)$. 

\begin{lemma}\label{conc_phi} If $\cF$ is a star-shaped, $b$-uniformly bounded function class and for all $k\in [K]$ and the loss function $\ell(R_h,f)=(R_h(x_h^{(k)})-r_h^{(k)})f(z_h^{(k)})$ 
is $L_{K,\beta,d_1}:=L+\sigma\sqrt{(\beta+1)\log Kd_1}$-Lipschitz in $f,\forall f\in \cF$, with high-probability, then with probability $1-\delta-1/K^{\beta}$, $\forall f\in \cF$, we have 
\begin{align*}
    \Big|\Phi_K(R_h^*,f)-\Phi(R_h^*,f)\Big|&\lesssim  \Big(L_{K,\beta,1}\delta_{\cR,h,\delta}\|f\|_2+L_{K,\beta,1}\delta_{\cR,h,\delta}^2\Big).
\end{align*}
Here $\delta_{\cR,h,\delta}=\delta_{\cR,h}+c_0\sqrt{\log(c_1/\delta)/n}$ with $\delta_{\cR,h}$ being an upper bound of the critical radius of the function classes $\cF$ and $\cR^*_h.$ A similar inequality also holds for $G_{h,j}^*,\forall j\in [d_1]$ with probability $1-\delta-1/(d_1K)^{\beta}$ by replacing $L_{K,\beta,1},\delta_{\cR,h,\delta}$ with $L_{K,\beta,d_1},\delta_{\cG,h,\delta}$.
\end{lemma}
\begin{proof}
See \S\ref{auxiliary} for a detailed proof.
\end{proof}
Next, according the following Lemma \ref{conc_phi}, with probability $1-\delta-1/K^{\beta}$, we have
\begin{align}
    \mathbf{(I)}\le \sup_{f\in \cF}\Bigg\{\Phi(R_h^*,f)-\frac{1}{2K}\sum_{k=1}^{K}f^2\Big(z_h^{(k)}\Big) +C_1L_{K,\beta}\Big(\delta_{\cR,h,\delta}\|f(z)\|_2+\delta_{\cR,h,\delta}^2\Big) \Bigg\}. \label{ineq_1}
\end{align}

As $\cF$ is a star-shaped and $b$-uniformly bounded function class, by Theorem 14.1 of \cite{wainwright_2019}, with probability $1-\delta$, $\forall f\in \cF$, we obtain
\begin{align}\label{conc_f}
\Big|\|f\|_{2,K}^2-\EE\big[f(z_h)^2\big]\Big|\le 0.5\EE\big[f(z_h)^2\big]+0.5\delta_{\cR,h,\delta}^2.
\end{align}
Thus, we have
\begin{align}\label{ineq_2}
    \frac{1}{4}\EE\big[f^2(z_h)\big]-\frac{1}{4}\delta_{\cR,h,\delta}^2\le \frac{1}{2K}\sum_{k=1}^{K}f^2\Big(z_h^{(k)}\Big)\le \frac{3}{4}\EE\big[f^2(z_h)\big]+\frac{1}{4}\delta_{\cR,h,\delta}^2.
\end{align}
Combining \eqref{ineq_1} with \eqref{ineq_2}, with probability $1-\delta-1/K^{\beta}$, we have
\begin{align}\label{ineq_3}
    \mathbf{(I)}\le \sup_{f\in \cF}\bigg\{\Phi(R_h^*,f)-\frac{1}{4}\EE\big[f(z_h)^2\big]+\frac{1}{4}\delta^2_{\cR,h,\delta} +L_{K,\beta,1}C_1\Big(\delta_{\cR,h,\delta}\sqrt{\EE\big[f(z_h)^2\big]}+\delta_{\cR,h,\delta}^2\Big) \bigg\}.
\end{align}
Next, we use a general inequality to obtain an upper bound for the right hand side of \eqref{ineq_3}. For all $a,b>0$, by simple calculation, we have
\begin{align}\label{general_ineq}
    \sup_{f\in \cF}\big(a\|f\|-b\|f\|^2\big)\le \frac{a^2}{4b}.
\end{align}
Here $\|\cdot\|$ can represent any norm.
Moreover, we also have $\Phi(R_h^*,f)=0$ by the  identifiability condition in \eqref{identify_r}. Thus, we finally have 
\begin{align}\label{ub_1}
    \mathbf{(I)}\lesssim L_{K,\beta,1}^2\delta_{\cR,h,\delta}^2,
\end{align}
with probability $1-\delta-1/K^{\beta}$ by utilizing \eqref{general_ineq}.

Next, we aim at getting a lower bound for $\mathbf{(II)}$. Recall that the $\mathbf{(II)}$ given in \eqref{i_ii}. We have
\begin{align}
    \mathbf{(II)}&=\sup_{f\in \cF}\Bigg\{\frac{1}{K}\sum_{k=1}^{K}\Big(\hat R_h(x_h^{(k)})-r_h^{(k)}\Big)f\Big(z_h^{(k)}\Big)-\frac{1}{K}\sum_{k=1}^{K}\Big(R_h^*(x_h^{(k)})-r_h^{(k)}\Big)f\Big(z_h^{(k)}\Big)\nonumber\\&\qquad+\frac{1}{K}\sum_{k=1}^{K}\Big(R_h^*(x_h^{(k)})-r_h^{(k)}\Big)f\Big(z_h^{(k)}\Big)-\frac{1}{2K}\sum_{k=1}^{K}f\Big(z_h^{(k)}\Big)^2 \Bigg\}\nonumber\\
    &\ge \sup_{f\in \cF}\Big\{\Phi_K(\hat R_h,f)-\Phi_K(R_h^*,f)-\|f\|_{2,K}^2\Big\}+\inf_{f\in \cF}\Big\{\Phi_K(R_h^*,f)+\frac{1}{2}\|f\|_{2,K}^2 \Big\}\nonumber\\&
    =\sup_{f\in \cF}\Big\{ \Phi_K(\hat R_h,f)-\Phi_K(R_h^*,f)-\|f\|_{2,K}^2\Big\}-\sup_{f\in\cF}\Big\{\Phi_K(R_h^*,f)-\frac{1}{2}\|f\|_{2,K}^2 \Big\}\label{symmetric}\\
    &\ge \sup_{f\in \cF}\Big\{ \Phi_K(\hat R_h,f)-\Phi_K(R_h^*,f)-\|f\|_{2,K}^2\Big\}-CL_{K,\beta,1}^2\delta_{\cR,h,\delta}^2\nonumber.
\end{align}
Here, the second equality holds since we assume $\cF$ is symmetric in Assumption \ref{assume:functionclass}. The last inequality follows from the obtained upper bound for $\mathbf{(I)}$ in \eqref{ub_1} and $C$ is an absolute constant. 

Next we will state the following Lemma  \ref{conc_rf} which is crucial in providing a lower bound for $ \sup_{f\in \cF}\{ \Phi_K(\hat R_h,f)-\Phi_K(R_h^*,f)-\|f\|_{2,K}^2\}$.
\begin{lemma} \label{conc_rf}
We let 
\begin{align}\label{frh}
    f_{R_h}(z_h):=\mathbb{T}(R_h-R_h^*)=\EE[R_h(x_h)-R_h^*(x_h)\given z_h].
\end{align}
As $\cR^*_h$ is star-shaped, then with probability $1-\delta,$ we have
\begin{align*}
 &\Big|\Phi_K(R_h,f_{R_h})-\Phi_K(R_h^*,f_{R_h})-\big[\Phi(R_h,f_{R_h})-\Phi(R_h^*,f_{R_h})\big]\Big|\\&\qquad\lesssim \delta_{\cR,h,\delta}\sqrt{\EE\Big\{\big[(R_h-R_h^*)(x_h)f_{R_h}(z_h)\big]^2\Big\}}+\delta^2_{\cR,h,\delta}, \forall R_h\in \RR_h.
\end{align*}
Here $\delta_{\cR,h,\delta}$ is an upper bound of the critical radius of $\cR_h^*.$ 
\end{lemma}
\begin{proof}
See \S\ref{auxiliary} for a detailed proof.
\end{proof}
Next, we use the conclusion of Lemma \ref{conc_rf} to derive a lower bound for $ \sup_{f\in \cF}\{ \Phi_K(\hat R_h,f)-\Phi_K(R_h^*,f)-\|f\|_{2,K}^2\}$. Here we recall that we denote $f_{\hat R_h}=\mathbb{T}(\hat R_h-R_h^*)=\EE[\hat R_h(x_h)-R_h^*(x_h)\given z_h]$ in \eqref{frh}. By our Assumption \ref{assume:functionclass}, we assume for all $R_h\in \RR_h,$ $\mathbb{T}(R_h-R_h^*)\in \cF$. We next divide our analysis into two cases.

When $\|f_{\hat R_h}\|_2\le \delta_{\cR,h,\delta},$
we have
\begin{align*}
  \sup_{f\in \cF}\Big\{ \Phi_K(\hat R_h,f)-\Phi_K(R_h^*,f)-\|f\|_{2,K}^2\Big\}\ge \Phi_K\big(\hat R_h,f_{\hat R_h}\big)-\Phi_K\big(R^*_h,f_{\hat R_h}\big)-\|f_{\hat R_h}\|_{2,K}^2, 
\end{align*}
since $f_{\hat R_h}$ belongs to $\cF.$
By Lemma \ref{conc_rf}, with probability $1-\delta,$ we further have
\begin{align*}
   \Phi_K(\hat R_h,f_{\hat R_h})-\Phi_K(R^*_h,f_{\hat R_h})-\|f_{\hat R_h}\|_{2,K}^2&\ge  \Phi(\hat R_h,f_{\hat R_h})-\Phi(r^*_h,f_{\hat R_h})-\delta_{\cR,h,\delta}\|f_{\hat R_h}\|_2-\delta_{\cR,h,\delta}^2-\|f_{\hat R_h}\|_{2,K}^2
   \\&\ge \EE\big[(\hat R_h(x_h)-R_h^*(x_h))\EE[\hat R_h(x_h)-R_h^*(x_h)\given z_h]\big]-C_3\delta_{\cR,h,\delta}^2
   \\&\ge 0-C_3\delta_{\cR,h,\delta}^2.
\end{align*}
Here $C_3$ is an absolute constant.
The second inequality holds since we assume $\|f_{\hat R_h}\|_2\le \delta_{\cR,h,\delta},$ and the fact that $\|f_{\hat R_h}\|_{2,K}^2\le 1.5\|f_{\hat R_h}\|_2^2+\delta_{\cR,h,\delta}^2$ according to \eqref{conc_f}.

When $\|f_{\hat R_h}\|_2\ge \delta_{\cR,h,\delta},$ we let $\kappa=\xi_{\cR,h,\delta}/(2\|f_{\hat R_h}\|_2)\in [0,0.5].$ We have $\kappa f_{\hat R_h}\in \cF$ as $\cF$ is assumed to be star-shaped in Assumption \ref{assume:functionclass}. In this case, we have
\begin{align}
    \sup_{f\in \cF}\Big\{ \Phi_K(\hat R_h,f)-\Phi_K(R_h^*,f)-\|f\|_{2,K}^2\Big\}&\ge \kappa\big[\Phi_K(\hat R_h,f_{\hat R_h})-\Phi_K(R_h^*,f_{\hat R_h})\big]-\kappa^2\|f_{\hat R_h}\|_{2,K}^2.\nonumber\\
    &\ge\kappa\big[\Phi(\hat R_h,f_{\hat R_h})-\Phi(R_h^*,f_{\hat R_h})\big]\nonumber\\&\qquad-\kappa\big(\delta_{\cR,h,\delta}\|f_{\hat R_h}\|_2+\delta_{\cR,h,\delta}^2\big)-\kappa^2\|f_{\hat R_h}\|_{2,K}^2\notag\\&\ge 0-C_4\delta_{\cR,h,\delta}^2. \label{ineq1}
\end{align}
Here $C_4$ is an absolute constant.
The second inequality holds by Lemma \ref{conc_rf}. The last inequality follows from several facts. First, we have  $\kappa \delta_{\cR,h,\delta}\|f_{\hat R_h}\|_2 \lesssim \delta_{\cR,h,\delta}^2,${ by the definition of } $\kappa=\delta_{\cR,h,\delta}/(2\|f_{\hat R_h}\|_2).$ Next, we obtain $\delta_{\cR,h,\delta}^2,$ since $\kappa\in [0,0.5].$ Moreover, we obtain $\kappa^2\|f_{\hat R_h}\|_{2,K}^2\le \kappa^2\big(1.5\|f_{\hat R_{h}}\|_2^2+\delta_{\cR,h,\delta}^2\big)\lesssim \delta_{\cR,h,\delta}^2$. Combining these together, we obtain \eqref{ineq1}.

Finally, combining the upper bound for $\mathbf{(I)}$ and lower bound for $\mathbf{(II)},$ we have
\begin{align*}
    \cL_K(R_h^*)-\cL_K(\hat R_h)\lesssim L_{K,\beta,1}^2\delta_{\cR,h,\delta}^2.
\end{align*}
Thus, we conclude our proof of Lemma \ref{contain_true}.
\end{proof}
\subsection{Proof of Lemma \ref{Pessmistic}}\label{subsec:pessmistic}
In this subsection, we will prove Lemma \ref{Pessmistic}, which provides an upper bound of the suboptimality. 
\begin{proof}
We let $\cM$ be the product of the confidence sets $\{(\bar{\cR}_h,\bar{\cG}_h)\}_{h=1}^{H}.$ We have $M^*=\{(\bar{R}_h^*,\bar{\PP}_h^*)\}_{h=1}^{H}\in \cM$ with high probability, i.e. the true model lies in the confidence set with high probability by Lemma \ref{contain_true}.

For any given policy $\pi$, we denote $\hat J(\pi)=\min_{M\in \cM}J(M,\pi).$
We have
\begin{align*}
&J(M^*,\pi^*)-J(M^*,\hat\pi)\\&=J(M^*,\pi^*)-\hat J(\pi^*)+\hat J(\pi^*)-J(M^*,\hat\pi)\\
&=J(M^*,\pi^*)-\hat J(\pi^*)+\hat J(\pi^*)-\hat J(\hat\pi)+\hat J(\hat\pi)-J(M^*,\hat\pi)\\
&\le J(M^*,\pi^*)-\hat J(\pi^*)+\hat J(\hat\pi)-J(M^*,\hat\pi)\\&\le J(M^*,\pi^*)-\hat J(\pi^*).
\end{align*}

The first inequality follows from the fact that $\min_{M\in \cM}J(M,\pi^*)\le \min_{M\in \cM}J(M,\hat\pi)$ according to our selection of $\hat\pi$ given in \eqref{est_policy}. The second inequality holds by pessimism, namely, $M^*\in \cM$ and $\hat J(\hat\pi)=\min_{M\in \cM}J(M,\hat \pi)\le J(M^*,\hat\pi).$ Thus, we have
\begin{align*}
\textrm{SubOpt}(\hat\pi)\le J(M^*,\pi^*)- \min_{M\in\cM}J(M,\pi^*).
\end{align*}

We next let $\tilde{M}:=\{(\bar{\tilde{R}}_h,\bar{\tilde{P}}_h)\}_{h=1}^{H}=\argmin_{M\in \cM}J(M,\pi^*).$ In addition, we let $\tilde{M}_h=\{(\bar{\tilde{R}}_j,\bar{\tilde{\PP}}_j)\}_{j=h}^{H}$ and $M_h^*=\{(\bar{R}_j^*,\bar{\PP}_j^*)\}_{j=h}^{H}$ are models starting from stage $h$. Moreover, we define $\bar{V}_{M_h}^{\pi}(s)$ as the total value function evaluated by policy $\pi$ on model $M_h.$

Next, we expand the expression of $J(M^*,\pi^*)-J(\tilde{M},\pi^*)$. The intuition is similar with the simulation lemma given in \cite{sun2019model} and we deduce this in our own setting.
By definition of the value function, we have
\begin{align*}
    & J(M^*,\pi^*)-J(\tilde{M},\pi^*) \notag\\
    & \qquad =-\bigg\{\EE_{a_1\sim \pi_1^*(\cdot \given s_1), s_1\sim  \rho_0}\bigg[\EE_{s_2\sim \bar{\tilde{P}}_1(\cdot\given s_1,a_1)}\Big(\bar{\tilde{ R}}_1(s_1,a_1)+V_{\tilde{M}_2}^{\pi^*}\Big)
    \\&\qquad\qquad  -\EE_{s_2\sim\bar{{P}}^*_1(\cdot\given s_1,a_1)}\Big(\bar{\tilde{ R}}_1(s_1,a_1)+\bar{V}_{\tilde{M}_2}^{\pi^*}(s_2)\Big) +\EE_{s_2\sim\bar{{P}}^*_1(\cdot\given s_1,a_1)}\Big(\bar{\tilde{ R}}_1(s_1,a_1)+\bar{V}_{\tilde{M}_2}^{\pi^*}(s_2)\Big) \\&\qquad\qquad  -\EE_{s_2\sim\bar{{P}}^*_1(\cdot\given s_1,a_1)}\Big( \bar{R}_1^{*}(s_1,a_1))+\bar{V}_{M_2^*}^{\pi^*}(s_2)\Big)\bigg]\bigg\}.
    \end{align*}
After a direct calculation, we have
    \begin{align*}
    J(M^*,\pi^*)-J(\tilde{M},\pi^*)&=-\EE_{a_1\sim \pi_1^*(\cdot\given s_1),s_1\sim\rho_0}\bigg\{ \int \bar{V}_{\tilde{M}_2}^{\pi^*}(\cdot)\Big[\ud\bar{\tilde{P}}_1(\cdot\given s_1,a_1)-\ud \bar{P}_1^*(\cdot \given s_1,a_1) \Big]\\&\qquad+\Big[(\bar{\tilde{R}}_1-\bar{R}_1^*)(s_1,a_1)\Big]\bigg\}+\EE_{s_2\sim \bar{P}_1^*(\cdot \given s_1,a_1)}\Big[\bar{V}_{M_2^*}^{\pi^*}(s_2)-\bar{V}_{\tilde{M}_2}^{\pi^*}(s_2)\Big].
   \end{align*}
   After expanding $\bar{V}_{M_h^*}^{\pi^*}(s_h)-\bar{V}_{\tilde{M}_h}^{\pi^*}(s_h),h\ge 2$ in the same way as above,  we obtain
   \begin{align*}
    J(M^*,\pi^*)-J(\tilde{M},\pi^*)&=-\sum_{h=1}^{H}\EE_{a_h\sim \pi_h^*(\cdot\given s_h),s_h\sim d_{h-1}^{\pi^*}}\bigg[\int  \bar{V}^{\pi^*}_{\tilde{M}_{h+1}}(\cdot)\Big(\ud\bar{\tilde{P}}_h-\ud \bar{P}_h^*\Big)(\cdot\given s_h,a_h) \bigg]
    \\&\qquad-\sum_{h=1}^{H}\EE_{a_h\sim \pi_h^*(\cdot\given s_h),s_h\sim d_{h-1}^{\pi^*}}\Big[(\bar{\tilde{ R}}_h-\bar{R}_h^*)(s_h,a_h)\Big],
\end{align*}
where $d^{\pi^*}_{h-1}$ represents the distribution of state $s_h$, which is generated by following policy $\pi^*$ and the true transition functions $\{\bar{\PP}_j^*\}_{j=1}^{h-1}$ defined in \eqref{transit_new2}.

As we only consider bounded functions in our defined function classes, thus $H$ is an upper bound for our value functions under all policy $\pi$ and model all $M\in \cM$.  Finally, according the decomposition given above, we have
\begin{align*}
   J(M^*,\pi^*)-J(\tilde{M},\pi^*)&\le \sum_{h=1}^{H}\EE_{a_h\sim \pi_h^*(\cdot \given s_h),\,s_h\sim d^{\pi^*}_{h-1}}\bigg[2H\cdot \Big\|(\bar{\tilde{P}}_h-\bar{P}_h^*)(\cdot \given s_h,a_h)\Big\|_1\\&\qquad+\Big|\Big(\bar{\tilde{R}}_h-\bar{R}_h^*\Big)(s_h,a_h)\Big|\bigg].
\end{align*}
Thus, we finish our proof of Lemma \ref{Pessmistic}.
\end{proof}
\subsection{Proof of Lemma \ref{project_MSE}}\label{subsec:PMSE}
In this subsection, for any given $h\in [H]$, we will prove that for all $R_h\in \cR_h,G_h\in\cG_h$, 
\begin{align*}
    &\EE_{\rho}\bigg[\Big(\EE_{\rho}\big[R_h(s_h,a_h,o_h)-R_h^{*}(s_h,a_h,o_h) \given s_h,a_h\big]\Big)^2\bigg]~~\textrm{and}\\
&\EE_{\rho}\bigg[\Big\|\EE_{\rho}\big[G_h(s_h,a_h,o_h)-G_h^{*}(s_h,a_h,o_h) \given s_h,a_h\big]\Big\|_2^2\bigg] 
\end{align*}
are small with high probability. Here the distribution $\rho$ denotes the sampling distribution. Similar to the proof of Lemma \ref{contain_true}, here we also only prove the case for $R_h,$ as the proof for every coordinate of $G_h$ is almost the same. 
\begin{proof}
For all $R_h\in \cR_h$, we have
\begin{align}\label{ineq2}
    & \sup_{f\in \cF}\Big\{\Phi_K(R_h,f)-\frac{1}{2}\|f\|_{2,K}^2\Big\}\notag \\
    & \qquad \ge \sup_{f\in \cF}\Big\{ \Phi_K(R_h,f)-\Phi_K(R_h^*,f)-\|f\|_{2,K}^2\Big\}-\sup_{f\in \cF}\Big\{\Phi_K(R_h^*,f)-\frac{1}{2}\|f\|_{2,K}^2\Big\},
\end{align}
by using similar argument of \eqref{symmetric}.
By the definition of $\cL_K(\cdot)$  in \S\ref{planning}, $\forall R_h\in \cR_{h},$ with probability $1-\delta-1/K^{\beta}$, we have
\begin{align}
    \sup_{f\in \cF}\Big\{\Phi_K(R_h,f)-\Phi_K(R_h^*,f)-\|f\|_{2,K}^2\Big\}&\le \cL_K(R_h)+\cL_K(R_h^*)\nonumber
    \\&\le \cL_K(\hat R_h)+C_5L_{K,\beta,1}^2\delta_{\cR,h,\delta}^2+\cL_K(R_h^*)\nonumber\\&\le 2\cL_K(R_h^*)+C_5L_{K,\beta,1}^2\delta_{\cR,h,\delta}^2\nonumber\\&\lesssim L_{K,\beta,1}^2\delta_{\cR,h,\delta}^2.\label{upperbound_rh}
\end{align}
Here the first inequality follows from \eqref{ineq2}. The second inequality holds since $R_h\in \cR_h$ and $\cL_K(R_h)\lesssim \cL_K(R_h^*)+C_5L_{K,\beta,1}^2\delta_{\cR,h,\delta}^2$ by the definition of $\cR_h.$ The third inequality holds by the definition of $\hat R_h.$ The last inequality follows from \eqref{ub_1}.

Next, we prove that $\EE_{\rho}[(\EE_{\rho}[R_h(s_h,a_h,o_h)-R_h^{*}(s_h,a_h,o_h) \given s_h,a_h])^2]$ is small.  We assume there exists $R_h\in \cR_h$ such that $\|f_{R_{h}}\|_2=\sqrt{\EE_{\rho}[\mathbb{T}(R_h-R_h^*)^2]}\ge L_{K,\beta,1}\delta_{\cR,h,\delta}$, otherwise we obtain our conclusion directly. 
For such a $R_h$, we let $\kappa_{R_h}=L_{K,\beta,1}\delta_{\cR,h,\delta}/(2\|f_{R_h}\|_2),$  and we have $\kappa_{R_h}\in[0,0.5]$.
Thus, $\forall R_h\in\cR_h$, when $\|f_{R_{h}}\|_2=\sqrt{\EE_{\rho}[\mathbb{T}(R_h-R_h^*)^2]}\ge L_{K,\beta,1}\delta_{\cR,h,\delta}$, with probability $1-\delta$, we obtain
\begin{align}
   & \sup_{f\in\cF}\Big\{\Phi_K(R_h,f)-\Phi_K(R_h^*,f)-\|f\|_{2,K}^2\Big\}\notag\\&\qquad\ge \kappa_{R_h}\big\{\Phi_K(R_h,f_{R_h})-\Phi_K(R_h^*,f_{R_h})\big\}-\kappa_{R_h}^2\|f_{R_h}\|_{2,K}^2
    \nonumber\\&\qquad\ge \kappa_{R_h}\big[\Phi(R_h,f_{R_h})-\Phi(R_h^*,f_{R_h})\big]-\kappa_{R_h}^2\|f_{R_h}\|_{2,K}^2\nonumber\\&\qquad\qquad-\kappa_{R_h}\big(\delta_{\cR,h,\delta}\|f_{R_h}\|_2+\delta_{\cR,h,\delta}^2\big)\nonumber\\& \qquad\ge \kappa_{R_h}\big[\Phi(R_h,f_{R_h})-\Phi(R_h^*,f_{R_h})\big]-C_6L_{k,\beta,1}^2\delta_{\cR,h,\delta}^2\nonumber\\&\qquad=\frac{L_{k,\beta,1}\delta_{\cR,h,\delta}}{2}\sqrt{\EE_{\rho}\big[\mathbb{T}(R_h-R_h^*)^2\big]}-C_6L_{k,\beta,1}^2\delta_{\cR,h,\delta}^2.\label{lowerbound_rh}
\end{align}
The first inequality holds since $\cF$ is star-shaped and thus $\kappa_{R_h}f_{R_h}\in \cF,$ when $$\|f_{R_{h}}\|_2=\sqrt{\EE_{\rho}[\mathbb{T}(R_h-R_h^*)^2]}\ge L_{K,\beta,1}\delta_{\cR,h,\delta}.$$
The second inequality holds uniformly for all $R_h$ with high probability by Lemma \ref{conc_rf}. The third inequality follows by several facts. First, we obtain $ \kappa_{R_h} \delta_{\cR,h,\delta}\|f_{R_h}\|_2 \lesssim L_{K,\beta,1}\delta_{\cR,h,\delta}^2,$ by the definition that  $\kappa_{R_h}=L_{K,\beta,1}\delta_{\cR,h,\delta}/(2\|f_{R_h}\|_2).$ Second, it holds that $\kappa_{R_h} \delta_{\cR,h,\delta}^2\lesssim  \delta_{\cR,h,\delta}^2$, since $\kappa_{R_h}\in [0,0.5]$. Third, we obtain $\kappa_{R_h}^2\|f_{R_h}\|_{2,K}^2\le \kappa_{R_h}^2\big(1.5\|f_{ R_{h}}\|_2^2+\delta_{\cR,h,\delta}^2\big)\lesssim L_{K,\beta,1}^2\delta_{\cR,h,\delta}^2.$ Combining these together, we obtain \eqref{lowerbound_rh}.

Finally, combining our results obtained in \eqref{upperbound_rh} and \eqref{lowerbound_rh}, with probability $1-\delta-1/K^{\beta},$ we have
\begin{align*}
    \sqrt{\EE_{\rho}\Big[\EE_{\rho}\big[R_h(s_h,a_h,o_h)-R_h^*(s_h,a_h,o_h)\given s_h,a_h\big]^2\Big]}\lesssim L_{K,\beta,1}\delta_{\cR,h,\delta}.
\end{align*}
Similarly, we are also able to obtain \eqref{radius_g} with probability $1-\delta-1/(d_1K)^{\beta}$ by following the similar proof. To be more specific, one only needs to replace $R_h(\cdot),R_h^{(k)}$ by $G_{h,j}(\cdot),s_{h+1,j}$ for each $j\in [d_1]$ and obtain the upper bounds for $\sqrt{\EE_{\rho}[\mathbb{T}(G_{h,j}-G_{h,j}^*)^2]}$ for all $j.$ Putting all pieces together, we conclude our proof of Lemma \ref{project_MSE}.
\end{proof}

\section{Proof of Theorem \ref{sub_optimality_approx}}
\begin{proof}[Proof of Theorem \ref{sub_optimality_approx}]
In this section, we will provide theoretical proof for Theorem \ref{sub_optimality_approx}. Recall that we let $\RR_h$ and $\mathbb{G}_{h,j}$ be the true function classes for $R_h^*(\cdot)$ and $G_{h,j}^*(\cdot), \textrm{ for all } j\in [d_1],h\in [H].$  However, we use misspecified function classes $\tilde{\RR}_h,\tilde{\mathbb{G}}_{h,j}$ to estimate our model. In addition, $ \textrm{ for all } h\in [H], j\in [d_1],$ we let $R_h^0:=\argmin_{R_h\in \tilde{R}_h}\|R_h-R_h^{*}\|_{\infty}, G_{h,j}^0:=\argmin_{G_{h,j}\in \tilde{G}_{h,j}}\|G_{h,j}-G_{h,j}^{*}\|_{\infty}.$ 

First, we prove $R_h^0,G_{h,j}^0$ belongs to the confidence sets defined in \eqref{confidence_r} and \eqref{confidence_g} (by replacing $\RR_h$ and $\mathbb{G}_{h,j}$ with $\tilde{\RR}_h,\tilde{G}_{h,j},$ respectively). We summarize the conclusion in the following Lemma \ref{contain_true}.
\begin{lemma}\label{contain_true_mis}
 For any $h\in [H],$  with probability $1-\delta-1/(d_1K)^{\beta}$, we have
\begin{align*}
R_h^{0}\in \cR_h^M&=\Big\{R_h\in \tilde{\RR}_h: \cL_n(R_h)-\cL_n(\hat R_h)\lesssim L_{K,\beta,1}^2\delta_{\cR,h,\delta}^2+\eta_{K,r,h}^2 \Big\},\\
G_h^0\in \cG_h^M&= (\cG_{h,1}^M,\cG_{h,2}^M,\cdots,\cG_{h,d_1})^M, \textrm{ where we define} \\ \cG_{h,j}&=\Big\{G_{h,j}\in \tilde{\mathbb{G}}_{h,j}: \cL_n(G_{h,j})-\cL_n(\hat G_{h,j})\lesssim L_{K,\beta,1}^2\delta_{\cG,h,\delta}^2+\eta_{K,G,h}^2\Big\}.
\end{align*}
 Here we have $\delta_{\cR,h,\delta}=\delta_{\cR,h}+c_1\sqrt{\log(c_0/\delta)/K},\delta_{\cG,h,\delta}=\delta_{\cG,h}+c_2\sqrt{\log(c_0d_1/\delta)/K}$ with $c_0,c_1,c_2$ being absolute constants. Moreover, $\delta_{\cG,h}$ is an upper bound of the maximum critical radius of $\cF,\tilde{\cG}^{*M}_{h,j}, \textrm{ for all } j\in[d_1]$ (given in \eqref{ghm}) and $\delta_{\cR,h}$ is an upper bound of the maximum radius of $\cF,\tilde{\cR}_h^{*M}$ (given in \eqref{rhm}). Meanwhile, we define $L_{K,\beta,d_1}:=L+\sigma\sqrt{(\beta+1)\log Kd_1}$. 
\end{lemma}
\begin{proof}
See \S\ref{proof:contain_mis} for a detailed proof.
\end{proof}

Next, we prove that, for any given $h\in [H]$, for all $R_h\in \tilde{\cR}_h,G_h\in\tilde{\cG}_h,$ 
\begin{align*}
    &\EE\Big[\big(\EE\big[R_h(s_h,a_h,o_h)-R_h^{0}(s_h,a_h,o_h) \given s_h,a_h\big]\big)^2\Big],\textrm{ and}\\
&\EE\Big[\big\|\EE\big[G_h(s_h,a_h,o_h)-G_h^{0}(s_h,a_h,o_h) \given s_h,a_h\big]\big\|_2^2\Big] 
\end{align*}
are small with high probability.
We summarize this property in the following Lemma \ref{project_MSE_mis}.
\begin{lemma}\label{project_MSE_mis}
For all $R_h,G_h$ given in $\tilde{\cR}_h,\tilde{\cG}_h$, with  probability $1-H\delta-H/(Kd_1)^{\beta}$, we have
\begin{align}
\sqrt{\EE_{\rho}\Big[\big(\EE_{\rho}\big[R_h(s_h,a_h,o_h)-R_h^{0}(s_h,a_h,o_h) \given s_h,a_h\big]\big)^2\Big]}&\lesssim L_{K,\beta,1}\delta_{\cR,h,\delta}+\xi_{K,r,h}+\frac{\eta_{K,r,h}^2}{\delta_{\cR,h,\delta}}, \label{radius_rm}\\
\sqrt{\EE_{\rho}\Big[\big\|\EE_{\rho}\big[G_h(s_h,a_h,o_h)-G_h^{0}(s_h,a_h,o_h) \given s_h,a_h\big]\big\|_2^2\Big]}&\lesssim d_1L_{K,\beta,1}\delta_{\cG,h,\delta}+\xi_{K,G,h}+\frac{\eta_{K,G,h}^2}{\delta_{\cG,h,\delta}}, \label{radius_gm}
\end{align}
where $\xi_{K,r,h},\xi_{K,G,h},\eta_{K,r,h},\eta_{K,G,h}$ are approximation errors defined in \S\ref{sec:misspecified}.
\end{lemma}
\begin{proof}
See \S\ref{proof:mse_mis} for a detailed proof.
\end{proof}
Finally we will provide an upper bound for the suboptimality of $\hat\pi^M$ given in \eqref{est_policy_2}.

Recall that for all $ h\in [H], j\in [d_1],$ we let $$R_h^0:=\argmin_{R_h\in \tilde{R}_h}\|R_h-R_h^{*}\|_{\infty}, \qquad\qquad  G_{h,j}^0:=\argmin_{G_{h,j}\in \tilde{G}_{h,j}}\|G_{h,j}-G_{h,j}^{*}\|_{\infty}$$ in Assumption \ref{approximation_error}. In the following, we denote $$\bar{R}_h^0=\int_{o_h,i_h}[R_h^0(s_h,a_h,o_h)+f_{1h}(i_h)]\ud F_h(o_h\given s_h,a_h,i_h)\ud P_h(i_h).$$ In addition, we also define $\bar{P}_h^0$ in the same way as in \eqref{transit_new2} using $G_{h}^0=\{G_{h,j}^0\}_{j=1}^{d_1}$. Moreover, we let $\tilde{M}^*:=\{(\bar{R}_h^0,\bar{P}_h^0)\}_{h=1}^{H}.$


 The suboptimality is able to be written as:
\begin{align*}
   J(M^*,\pi^*)-J(M^*,\hat\pi^M) &=\Big[J(M^*,\pi^*)-J(\tilde{M}^*,\pi^*)\Big]+\Big[J(\tilde{M}^*,\pi^*)-J(\tilde{M}^*,\hat\pi^M)\Big]\\&\qquad+\Big[J(\tilde{M}^*,\hat\pi^M)-J(\tilde{M}^*,\pi^*)\Big]\\&=\mathbf{(i)}+\mathbf{(ii)}+\mathbf{(iii)}.
\end{align*}
Next, we provide an upper bound for $\mathbf{(ii)}.$ The way of controlling $\mathbf{(ii)}$ is almost the same as controlling the suboptimality without misspecification. To be more specific, we also define $\hat J(\pi) =\min_{M\in {\cM}^M}J(M,\pi), $ where $\cM^M:=\{(\bar{\cG}^M_h,\bar{\cR}^M_h)\}_{h=1}^H$ is defined in \S\ref{sec:misspecified}. We have
\begin{align*}
J(\tilde{M}^*,\pi^*)-J(\tilde{M}^*,\hat\pi^M)&=J(\tilde{M}^*,\pi^*)-\hat J(\pi^*)+\hat J(\pi)-J(\tilde{M}^*,\hat\pi^M)\\& \le J(\tilde{M}^*,\pi^*)-\hat J(\pi^*)+\hat J(\hat\pi)-J(\tilde{M}^*,\hat\pi^M)\\&\le  J(\tilde{M}^*,\pi^*)-\hat J(\pi^*).
\end{align*}
The first inequality follows from our definition of $\hat\pi^M.$ The second inequality follows from and pessimism and Lemma \ref{contain_true_mis} since $\tilde{M}^*=\{(\bar{R}_h^0,\bar{P}_h^0)\}_{h=1}^H$ lies in $\cM^M$ with high probability.

We next let  $\tilde{M}=\argmin_{M\in \cM^M}J(M,\pi^*),$ with $\tilde{M}:=\{(\bar{\tilde{\PP}}_h,\bar{\tilde{R}}_h)\}_{h=1}^H.$
Following similar arguments given in Lemma \ref{Pessmistic} or the simulation lemma in \cite{sun2019model}, we obtain
\begin{align*}
    \mathbf{(ii)}&=\Big[J(\tilde{M}^*,\pi^*)-J(\tilde{M}^*,\hat\pi^M)\Big]\\&\le H\sum_{h=1}^{H}\EE_{a_h\sim \pi^*(\cdot \given s_h),\,s_h\sim d^{\pi^*}_{h-1}}\bigg[\Big\|\Big(\bar{\tilde{ \PP}}_h-\bar{\PP}_h^0\Big)(\cdot\given s_h,a_h)\Big\|_1\bigg]\\&\qquad+\sum_{h=1}^{H}\EE_{a_h\sim \pi^*(\cdot \given s_h),\,s_h\sim d^{\pi^*}_{h-1}}\bigg[\Big|\bar{\tilde{ R}}_h(s_h,a_h)-\bar{R}_h^{0}(s_h,a_h)\Big| \bigg].
    \end{align*}
 Combining Lemma   \ref{project_MSE_mis} and following the similar to the derivations of \eqref{proof_g} and \eqref{proof_r}, we obtain 
    \begin{align*}
     \mathbf{(ii)}&\lesssim H\sum_{h=1}^{H}\tau_{K,G,h}\sqrt{d_1}\Bigg(L_{K,\beta,1}\delta_{h,\cG,\delta}+\xi_{K,G,h}+\frac{\eta^2_{K,G,h}}{\delta_{h,\cG,\delta}}\Bigg)\\&\qquad+ \sum_{h=1}^{H}\tau_{K,r,h}\bigg(L_{K,\beta,1}\delta_{\cR,h,\delta}+\xi_{K,r,h}+\frac{\eta_{K,r,h}^2}{\delta_{\cR,h,\delta}}\bigg).
\end{align*}

Next, we will prove an  upper bound for $\mathbf{(i)}$, and $\mathbf{(iii)}$ can be bounded in the same way. Following similar arguments in Lemma \ref{Pessmistic}, we obtain
\begin{align}
    \mathbf{(i)}&\lesssim  H\sum_{h=1}^{H}\EE_{a_h\sim \pi^*(\cdot \given s_h),\,s_h\sim d^{\pi^*}_{h-1}}\bigg[\Big\|\Big(\bar{{ \PP}}^0_h-\bar{\PP}_h^*\Big)(\cdot\given s_h,a_h)\Big\|_1\bigg]\notag\\&\qquad+\sum_{h=1}^{H}\EE_{a_h\sim \pi^*(\cdot \given s_h),\,s_h\sim d^{\pi^*}_{h-1}}\bigg[\Big|\bar{{ R}}^0_h(s_h,a_h)-\bar{R}_h^{*}(s_h,a_h)\Big| \bigg]=\mathbf{(iv)}+\mathbf{(v)}.\label{iv_and_v}
\end{align}
We next provide an upper bound for $\mathbf{(iv)}$ by following the similar idea of proving \eqref{proof_g}. In specific, we have
\begin{align}
    \mathbf{(iv)}&=H\sum_{h=1}^{H}\EE_{a_h\sim \pi^*(\cdot \given s_h),s_h\sim d^{\pi^*}_{h-1}}\Bigg[\bigg\|\int_{s_h,o_h} {\PP}^0(\cdot\given s_h,a_h,o_h,i_h)- \PP^*(\cdot\given s_h,a_h,o_h,i_h)\ud F_h(o_h\given s_h,a_h,i_h)\ud P_h(i_h)\bigg\|_1 \Bigg]\notag\\\notag& \le H\sum_{h=1}^{H} \EE_{a_h\sim \pi^*(\cdot \given s_h),s_h\sim d^{\pi^*}_{h-1}}\bigg[\int_{s_h,o_h}\big\|\ {\PP}^0(\cdot\given s_h,a_h,o_h,i_h)- \PP^*(\cdot\given s_h,a_h,o_h,i_h)\big\|_1\ud F_h(o_h\given s_h,a_h,i_h)\ud P_h(i_h) \bigg]\notag
   \\
    &\le H\sum_{h=1}^{H}\sqrt{\EE_{a_h\sim \pi^*(\cdot \given s_h),s_h\sim d^{\pi^*}_{h-1},i_h\sim P_h(\cdot),o_h\sim F_h(\cdot\given s_h,a_h,i_h)}\bigg[\textrm{TV}\Big({ \PP}^0_h(\cdot\given s_h,a_h,o_h,i_h),\PP_h^*(\cdot \given s_h,a_h,o_h,i_h)\Big)^2\bigg]}\notag
    \\ &  \lesssim H\sum_{h=1}^{H}\sqrt{\EE_{a_h\sim \pi^*(\cdot \given s_h),s_h\sim d^{\pi^*}_{h-1},o_h\sim \rho_{\pi}(\cdot\given s_h,a_h)}\bigg[\Big\|{G}_h^0( s_h,a_h,o_h)-G_h^*( s_h,a_h,o_h)\Big\|_2^2\bigg]}.\label{diff_g0}
\end{align}
The first inequality follows from Jensen's inequality and the second follows from Cauchy-Schwartz inequality. In addition, the third inequality follows from our assumption on Gaussian transition. 
Moreover, by our assumption on the approximation error in $\ell_{\infty}$-norm given in Assumption \ref{approximation_error}, we obtain 
\begin{align*}
    \eqref{diff_g0}\lesssim H\sum_{h=1}^{H} \sqrt{d_1}\max_{j\in [d_1]}\|G_{h,j}^0-G_{h,j}^*\|_{\infty}\le H\sum_{h=1}^{H} \sqrt{d_1}\eta_{K,G,h}
\end{align*}

For $\mathbf{(v)},$ following similar proof procedure, we have
\begin{align}
    \mathbf{(v)}\le \sum_{h=1}^{H}\eta_{K,r,h}.\label{v}
\end{align}
In terms of upper bounding the term $\mathbf{(iii)}$, we only need to replace $\pi^*$ by $\hat\pi$ and the other procedures will remain as the same. The reason is that we have obtained the upper bound of the difference between $G_{h,j}^0$ and $G_{h,j}^*$ (also $R_h^0$ and $R_h^*$) in $\ell_{\infty}$-norm and it is robust  to all probability measures including $\pi^*$ and $\hat \pi$. Thus, the proof of bounding $\mathbf{(iii)}$ is the same with the proof of $\mathbf{(i)}.$
 
Combining all results for $\mathbf{(i)},\mathbf{(ii)}$ and $\mathbf{(iii)}$, we have
\begin{align*}
V_{M^*}^{\pi^*}-V_{M^*}^{\hat \pi}&\le H\sum_{h=1}^{H}C_{\pi^*}\tau_{K,G,h}\sqrt{d_1}\Bigg(L_{K,\beta,1}\delta_{h,\cG,\delta}+\xi_{K,G,h}+\frac{\eta_{K,G,h}^2}{\delta_{h,\cG,\delta}}\Bigg)\\&\qquad+ \sum_{h=1}^{H}C_{\pi^*}\tau_{K,r,h}\Bigg(L_{K,\beta,1}\delta_{\cR,h,\delta}+\xi_{K,r,h}+\frac{\eta_{K,r,h}^2}{\delta_{\cR,h,\delta}}\Bigg)\\&\qquad + H\sum_{h=1}^{H} \sqrt{d_1}\eta_{K,G,h}+\sum_{h=1}^{H}\eta_{K,r,h}.
\end{align*}
We then conclude the proof of Theorem \ref{sub_optimality_approx}.
\end{proof}

\subsection{Proof of Lemma \ref{contain_true_mis}}\label{proof:contain_mis}
\begin{proof}
Here we will prove that $R_h^0$ falls in $\cR^M_h,$ and one is able to derive the corresponding proof for every coordinate of $G_h^0$ in a similar way. To simplify our notation, we use $x_h^{(k)}=(s_h^{(k)},a_h^{(k)},o_h^{(k)})$ and let $z_h^{(k)}=(s_h^{(k)},a_h^{(k)})$ and 
in the following proof. We note that here we omit the embedding functions $\phi_x(\cdot),\psi_z(\cdot)$ as we mentioned in the main body of this paper.
 
First, we have
\begin{align}
    \cL_K\big(R_h^0\big)-\cL_K\big(\hat R_h\big)&=\sup_{f\in \cF}\Bigg\{\frac{1}{K}\sum_{k=1}^{K}\Big(R_h^0\Big(x_h^{(k)}\Big)-R_h^{(k)}\Big)f\Big(z_h^{(k)}\Big)-\frac{1}{2K}\sum_{k=1}^{K}f^2\Big(z_h^{(k)}\Big) \Bigg\}\nonumber\\&\qquad-\sup_{f\in \cF}\Bigg\{\frac{1}{K}\sum_{k=1}^{K}\Big(\hat R_h\Big(x_h^{(k)}\Big)-R_h^{(k)}\Big)f\Big(z_h^{(k)}\Big)-\frac{1}{2K}\sum_{i=1}^{K}f^2\Big(z_h^{(k)}\Big) \Bigg\}\nonumber\\&:=\mathbf{(I)}-\mathbf{(II)}. \nonumber
\end{align}
Next, we aim at obtaining an upper bound for $\mathbf{(I)}.$
Similarly, we define
\begin{align*}
    \Phi_K(r,f)&=\frac{1}{K}\sum_{k=1}^{K}\Big(R_h\Big(x_h^{(k)}\Big)-R_h^{(k)}\Big)f\Big(z_h^{(k)}\Big),\textrm{ and}\\
    \Phi(r,f)&=\EE\big[(r(x_h)-r_h)f(z_h)\big].
\end{align*}
We have
\begin{align*}
    \mathbf{(I)}=\sup_{f\in \cF}\Bigg\{\Phi_K\big(R_h^0,f\big)-\frac{1}{2K}\sum_{k=1}^{K}f^2\Big(z_h^{(k)}\Big) \Bigg\}.
\end{align*}
Next, according the conclusion from Lemma \ref{conc_phi}, with probability $1-\delta-1/K^{\beta}$, we have
\begin{align}
    \mathbf{(I)}\le \sup_{f\in \cF}\Bigg\{\Phi\big(R_h^0,f\big)-\frac{1}{2K}\sum_{k=1}^{K}f^2\Big(z_h^{(k)}\Big) +C_1L_{K,\beta,1}\big(\delta_{\cR,h,\delta}\|f(z)\|_2+\delta_{\cR,h,\delta}^2\big) \Bigg\}. \label{ineq_1_m}
\end{align}
According to the expression of $\Phi(R_h^0,f)=\EE[(R_h^0(x_h)-R_h)f(z_h)],$ $\forall f\in \cF,$ we have
\begin{align*}
    \Big|\Phi\big(R_h^0,f\big)-\Phi\big(R_h^*,f\big)\Big|\le \|(R_h^0-R_h^*)(x_h)\|_{\infty}\|f(z_h)\|_{2}\le \eta_{K,r,h}\|f(z_h)\|_2.
\end{align*}
Combing this inequality with \eqref{ineq_1_m}, we obtain
\begin{align*}
    \eqref{ineq_1_m}\le \sup_{f\in \cF}\Bigg\{\Phi\big(R_h^*,f\big)-\frac{1}{2K}\sum_{k=1}^{K}f^2\Big(z_h^{(k)}\Big)+ (C_1L_{K,\beta}\delta_{\cR,h,\delta}+\eta_{K,r,h})\|f(z_h)\|_2+C_1L_{K,\beta,1}\delta_{\cR,h,\delta}^2\Bigg\}.
\end{align*}
Similar to the derivations of \eqref{ineq_2}, \eqref{ineq_3} and \eqref{general_ineq}, we have
\begin{align*}
    \mathbf{(I)}\lesssim L_{K,\beta,1}^2\delta_{\cR,h,\delta}^2+\eta_{K,r,h}^2.
\end{align*}
We have $\mathbf{(II)}\ge 0$ as $f=0$ is contained in the function class $\cF.$
Thus, we have $R_0$ lies in the constructed level set. Similarly, we are also able to prove that $G_{h,j}^{0}$ lies in the level set given in \eqref{confidence_g_m} for any $j\in[d_1],h\in [H]$.
\end{proof}
\subsection{Proof of Lemma \ref{project_MSE_mis}}\label{proof:mse_mis}
\begin{proof}
Here we only prove the case for $R_h(\cdot),$ since the case for $G_{h,j}(\cdot)$ follows in the same way. 

For all $R_h\in \tilde{\cR}_h$, we have
\begin{align*}
    \sup_{f\in \cF}\Big\{\Phi_K(R_h,f)-\frac{1}{2}\|f\|_{2,K}^2\Big\}\ge \sup_{f\in \cF}\Big\{ \Phi_K(R_h,f)-\Phi_K\big(R_h^0,f\big)-\|f\|_{2,K}^2\Big\}-\sup_{f\in \cF}\Big\{\Phi_K\big(R_h^0,f\big)-\frac{1}{2}\|f\|_{2,K}^2\Big\},
\end{align*}
by using similar argument as in \eqref{symmetric}.
Similarly, by Lemma \ref{contain_true_mis} and our definition of $\cL_K(\cdot)$ given in \S\ref{sec:misspecified}, we have
\begin{align}\label{upper_2_m}
    \sup_{f\in \cF}\Big\{ \Phi_K(R_h,f)-\Phi_K\big(R_h^0,f\big)-\|f\|_{2,K}^2\Big\}\lesssim L_{K,\beta,1}^2\delta_{\cR,h,\delta}^2+\eta_{K,r,h}^2.
\end{align}
Recall that we define $f_{R_h}=\argmin_{f\in \cF}\|f(z_h)-\mathbb{T}(R_h-R_h^0)(z_h)\|_2=\sqrt{\EE_{z_h}[(f(z_h)-\mathbb{T}(R_h-R_h^0)(z_h))^2]}.$
Without loss of generality, we assume there exist some $R_h\in \tilde{\RR}_h$ such that $\|f_{R_h}\|_2\ge L_{K,\beta,1}\delta_{\cR,h,\delta}+\xi_{K,r,h}+\eta_{K,r,h}^2/\delta_{\cR,h,\delta},$ otherwise, we obtain 
$\|\mathbb{T}(R_h-R_h^0)\|_2\lesssim L_{K,\beta,1}\delta_{\cR,h,\delta}+\xi_{K,r,h}+\eta_{K,r,h}^2/\delta_{\cR,h,\delta}$   for all $R_h\in \tilde{\RR}_h$, directly by Assumption \ref{approximation_error} and triangle inequality.
For those $R_h(\cdot)$, we let $\kappa_{R_h}=L_{K,\beta,1}\delta_{\cR,h,\delta}/2\|f_{R_h}\|_2$ and we have $\kappa_{R_h}\in[0,0.5].$ We then obtain
\begin{align*}
     &\sup_{f\in \cF}\Big\{ \Phi_K(R_h,f)-\Phi_K\big(R_h^0,f\big)-\|f\|_{2,K}^2\Big\}\\&\qquad\ge \kappa_{R_h}\big(\Phi_K(R_h,f_{R_h})-\Phi_K\big(R_h^0,f_{R_h}\big)\big)-\kappa_{R_h}^2\|f_{R_h}\|_{2,K}\\&\qquad\ge \kappa_{R_h}\big[\Phi(R_h,f_{R_h})-\Phi\big(R_h^0,f_{R_h}\big)\big]-\kappa_{R_h}^2\|f_{R_h}\|_{2,K}^2\\&\qquad\qquad-\kappa_{R_h}\big(\delta_{\cR,h,\delta}\|f_{R_h}\|_2+\delta_{\cR,h,\delta}^2\big)\\&\qquad\ge L_{K,\beta,1}\delta_{\cR,h,\delta}\big\|\mathbb{T}\big(R_h-R_h^0\big)\big\|_2-L_{K,\beta,1}\delta_{\cR,h,\delta}\xi_{K,r,h}-CL_{K,\beta,1}^2\delta_{\cR,h,\delta}^2.
\end{align*}
The first inequality holds since $\cF$ is star-shaped and $\kappa_{R_h}f_{R_h}\in \cF$. The second inequality holds uniformly for all $R_h$ by a similar argument of Lemma \ref{conc_rf}, where we only need to replace the definition of $f_{R_h}:=\mathbb{T}(R_h-R_h^*)$ by $f_{R_h}:=\argmin_{f\in \cF}\|f-\mathbb{T}(R_h-R_h^*)\|_2.$
The last inequality follows by several facts. First, we have $\kappa_{R_h} \delta_{\cR,h,\delta}\|f_{R_h}\|_2 \lesssim L_{K,\beta,1}\delta_{\cR,h,\delta}^2 $ by the definition of $\kappa_{R_h}=L_{K,\beta,1}\delta_{\cR,h,\delta}/(2\|f_{R_h}\|_2)$. Second, we obtain $\kappa_{R_h} \delta_{\cR,h,\delta}^2\lesssim  \delta_{\cR,h,\delta}^2$, since $\kappa_{R_h}\in [0,0.5]$. Last, we obtain $\kappa_{R_h}^2\|f_{R_h}\|_{2,K}^2\le \kappa_{R_h}^2\big(1.5\|f_{\hat R_{h}}\|_2^2+\delta_{\cR,h,\delta}^2\big)\lesssim L_{K,\beta,1}^2\delta_{\cR,h,\delta}^2.$ Combining these together, we conclude the last inequality.

In addition, Combining the conclusions given above with \eqref{upper_2_m}, we have
\begin{align*}
    L_{K,\beta,1}\delta_{\cR,h,\delta}\big\|\mathbb{T}\big(R_h-R_h^0\big)\big\|_2\lesssim L_{K,\beta,1}^2\delta_{\cR,h,\delta}^2+L_{K,\beta,1}\delta_{\cR,h,\delta}\xi_{K,r,h}+\eta_{K,r,h}^2.
\end{align*}
We finally obtain
\begin{align*}
    \big\|\mathbb{T}\big(R_h-R_h^0\big)\big\|_2\lesssim L_{K,\beta,1}\delta_{\cR,h,\delta}+\xi_{K,r,h}+\frac{\eta_{K,r,h}^2}{\delta_{\cR,h,\delta}}. 
\end{align*}
Similarly, we are also able to obtain an upper bound for 
$\EE[\|\EE[G_h(s_h,a_h,o_h)-G_h^{0}(s_h,a_h,o_h) \given s_h,a_h]\|_2^2]$.
Finally, we claim our proof of Lemma \ref{project_MSE_mis}.
\end{proof}

\section{Proof of Examples in \S\ref{sec:theory}}
In this section, we will provide theoretical proofs for our examples given in \S\ref{sec:theory}. We will give upper bounds for the suboptimality of our constructed policy under different function classes. To be more specific, we will prove our cases under linear function classes, RKHS and neural netwoks in \S\ref{prove_linear}, \S\ref{prove_rkhs} and \S\ref{subsec:proofnn}, respectively.
\subsection{Proof of Corollary \ref{example_linear}}\label{prove_linear}
The proof of Corollary \ref{example_linear} mainly requires computing the critical radius of \eqref{test_class}, \eqref{product_class} and \eqref{product_class_g}. We next aim at obtaining upper bounds for these critical radius mentioned above.

First, we will introduce some basic terminologies. An empirical $\epsilon$-cover of a function class $\cH$ is any function class $\cH_{\epsilon}$ such that for all $f\in \cH, \inf_{f_{\epsilon}\in \cH_{\epsilon}}\|f_{\epsilon}-f\|_{2,K}\le \epsilon$. For any given function class $\cH,$ we denote $N(\epsilon,\cH,S)$ as the smallest size of $\epsilon$-cover of $\cH.$ An empirical $\delta$-slice of $\cH$ is defined as $\cH_{S,\delta}=\{f\in \cH:\|f\|_{2,K}\le \delta \}. $ By Corollary 14.3 of \cite{wainwright_2019}, the empirical critical radius of the function class $\cH$ is upper bounded by any solution of
\begin{align}\label{sol_h}
    \int_{\delta^2/8}^{\delta}\sqrt{\frac{\log N(\epsilon,\cH_{S,\delta},S)}{K}}\ud \epsilon\le \frac{\delta^2}{20}.
\end{align}
We next make a relaxation by replacing $\cH_{S,\delta}$ by $\cH,$ in order to obtain an upper bound of the solution to \eqref{sol_h}. When $\cF$ defined in \eqref{test_class} only contains linear functions, by setting $\cH$ as $\cF$ and solving this \eqref{sol_h} for $\cF$, we ontain that the empirical critical raduis of $\cF$ is upper bounded by $\cO(\sqrt{m\log K/K})$. 

Next, we aim at getting an upper bound of the empirical critical radius of $\cR_h^*$. Since both $\RR_h$ and $\cF$ only contain uniformly bounded functions, then if $\RR_{h,\epsilon}$ is an empirical $\epsilon$-cover of $\RR_h$ and $\cF_{\epsilon}$ is an empirical $\epsilon$-cover of $\cF,$ we have $\{(R_{h,\epsilon}-R_h^*)f_{\epsilon}:R_{h,\epsilon}\in \RR_{h,\epsilon},f_{\epsilon}\in \cF_{\epsilon}\}$ acts as an empirical $C$-$\epsilon$ cover of the function class $\cR_{h}^*$ with $C$ being an absolute constant. Similar situation also holds for $\cG_{h,j}^*$. Thus, we have the empirical critical radius of $\cR_h^*$ is upper bounded by any solution to the following inequality
\begin{align}\label{linear_delta}
        \int_{\delta^2/8}^{\delta}\sqrt{\frac{\log N(\epsilon,\cF,S)}{K}+\frac{\log N(\epsilon,\RR_h,S)}{K}}\ud \epsilon\le \frac{\delta^2}{20}.
\end{align}
Since $\cF,\RR_h$ are linear function classes that only have finite dimensions, we have the solution to inequality \eqref{linear_delta} is
$\hat\delta_K=\cO(\sqrt{\max\{m,n_h\}\log K/K}).$ Moreover, as the function class $\cF,\RR_h$ only contain bounded functions, we have
\begin{align*}
    \delta_{h,\cR_h}=\cO\Bigg(\hat\delta_K+\sqrt{\frac{\log(1/\delta)}{K}}\Bigg),
\end{align*}
with probability $1-\delta$ with $\delta_{h,\cR_h}$ and $\hat\delta_K$ being the maximal critical radius and maximal empirical  critical radius of $\cF,\cR_h^*, \forall h\in [H],$ respectively, by Corollary 5 of \cite{dikkala2020}.
 In addition, we are also able to determine the value for $\delta_{h,\cG_h}$ by following similar proof procedure. We finally claim our proof of Corollary \ref{example_linear} by plugging in values of $\delta_{h,\cR_h}$ and $\delta_{h,\cG_h}$ in Theorem \ref{sub_optimality}.
\subsection{Proof of Corollary \ref{cor_rkhs}}\label{prove_rkhs}
In this subsection, we provide a proof for Corollary \ref{cor_rkhs}. We are interested in the local Rademacher complexity and those corresponding critical radius of function classes $\cF,\cR_h^*,\cG_{h,j}\forall j\in [d_1].$ Let $\{\lambda_i^{\cF}\}_{i=1}^{\infty}$ be the eigenvalues of $\cK_{\cF},$ we have
\begin{align*}
    \cR_K(\delta,\cF)\le\sqrt{  \frac{2}{K}\sum_{i=1}^{\infty}{\min\big\{\delta^2,4C_1\lambda_i^{\cF}\big\}}},
\end{align*}
where $C_1$ is an absolute constant by following Corollary 14.5 of \cite{wainwright_2019} stated in  \S\ref{auxillary_lem}.
Then the upper bound of the solution of $\cR_{K}(\delta,\cF)\le \delta^2$ is given by
\begin{align*}
    \delta_{\cF}=2\min_{j\in\NN}\Bigg\{\frac{j}{K}+\sqrt{\frac{2C_1}{K}\sum_{i=j+1}^{\infty}\lambda_i^{\cF}} \Bigg\}.
\end{align*}
Similar situation also holds for function classes $\cR_h^*,\cG_{h,j}^*,\forall j\in [d_1].$ 

Thus, when the eigenvalues of $\cK_{\cF},\cK_{\cR_h^*},\forall h\in [H],j\in [d_1]$ decay exponentially, we have $\|T(R_h-R_h^*)\|_2\lesssim  L_{K,d_1}\sqrt{{\log K}/{K}}+\sqrt{{\log(d_1/\delta)}/{K}}.$ The same situation also applies to $\cK_{\cG_{h,j}^*}$ and $\|T(G_{h,j}-G_{h,j}^*)\|_2.$ The same proof procedure also applies to the setting when eigenvalues of aforementioned kernels decay in polynomial speed.

Thus, by plugging these values into Theorem \ref{sub_optimality}, we conclude our proof of Corollary \ref{cor_rkhs}.
\subsubsection{Example: Neural Network}\label{example:nn}
In this subsection, we provide an example for Theorem \ref{sub_optimality_approx}, where we use the class neural network to approximate the underlying reward and transition functions.  We let $x_h:=\phi_x(s_h,a_h,o_h)$ with some bounded $\phi_x(\cdot)$ and $z_h:=\psi_z(s_h,a_h)$ with some embedding functions $\psi_x(\cdot),\psi_z(\cdot)$. 
In addition, for all $h\in [H]$, we assume $R_h^*(x_h)\in \RR_h,$ and $G_{h,j}^*(x_h)\in \mathbb{G}_{h,j}$, where $ \RR_h$ and $\mathbb{G}_{h,j}$ are true function classes that contain $R_h^*,G_{h,j}^*$. 
We next pose several needed assumptions on these function classes which build blocks for our theory.

\begin{assumption}\label{assump:appro_nn}
We assume true functions (reward and transition) fall in $\RR_h$ and $\mathbb{G}_{h,j},$ respectively, which are Sobolev balls with order $\alpha$ and input dimension $d$. In addition, we use misspecifed function classes $\tilde{\RR}_h,$ $\tilde{\mathbb{G}}_{h,j}$, namely, classes of ReLU neural networks with input dimension $d$ with $\cO(\log K)$ layers, and $\cO(K^{\frac{d}{2\alpha+d}})$ bounded weights. See \cite{yarotsky2017error} for a detailed introduction to Sobolev ball and ReLU neural networks. 
\end{assumption}
 When we consider using aforementioned neural networks to approximate the true functions in Sobolev function class, by Theorem 1 of \cite{yarotsky2017error} given in \S\ref{auxillary_lem}, we obtain for all $h\in [H],\forall j\in [d_1]$, 
\begin{align*}
    \eta_{K,r,h}^{N}:&=\min_{R_h\in \tilde{\RR}_h}\big\|(R_h-R_h^*)(\cdot)\big\|_{\infty}=\cO(K^{-\frac{\alpha}{2\alpha+d}}),\\
    \eta_{K,G,h}^{N}:&=\max_{j\in [d_1]}\min_{G_{h,j}\in \tilde{\RR}_h}\big\|(G_{h,j}-G_{h,j}^*)(\cdot)\big\|_{\infty}=\cO(K^{-\frac{\alpha}{2\alpha+d}}).
\end{align*}
We assume the minimizers exist and are defined as $R_h^0,G_{h,j}^0,\forall h\in [H],j\in [d_1].$

 Next, we put assumptions on $\cF$ and projected function classes $\{\mathbb{T}(R_h-R_h^0)(\cdot):=\EE[(R_h-R_h^0)(x_h)\given z_h=\cdot],R_h\in \tilde{\RR}_h\}$ and $\{\mathbb{T}(G_{h,j}-G_{h,j}^0)(\cdot):=\EE[(G_{h,j}-G_{h,j}^0)(x_h)\given z_h=\cdot],G_{h,j}\in \tilde{\mathbb{G}}_{h,j}\}$ as follows.

\begin{assumption}\label{ass:proj_nn}
 We assume $\mathbb{T}(R_h-R_h^0)(\cdot), \textrm{ for all }R_h\in\tilde{\RR}_h$ and $\mathbb{T}(G_{h,j}-G_{h,j}^0)(\cdot),\textrm{ for all } G_{h,j}\in\mathbb{G}_{h,j}$ fall in the Sobolev ball with order $\alpha$ and input dimension $d.$ In addition, we assume $\cF$ is the star hull (with center $0$) of a class of ReLU neural network with at most $\cO(\log K)$ layers, $\cO(K^{\frac{d}{2\alpha+d}})$ bounded weights and and input dimension $d$. 
\end{assumption}
Assumption \ref{ass:proj_nn} holds when the density of $\rho(x_h\given z_h)$ is smooth enough. For example, when $x_h\given z_h\sim N(\phi(z_h),
\tilde{\sigma}^2\II),$ we have $\rho(x_h\given z_h)=\tilde{\rho}(x_h-\phi(z_h))$. In this case, the conditional expectation  $T(R_h-R_h^0)(z_h)=\EE[(R_h-R_h^0)(x_h)\given z_h]$ is smooth as it is a convolution of a non-smooth function with a smooth one. Moreover, under  Assumption \ref{ass:proj_nn}, the upper bound of the approximation errors $\xi_{r,h,K},\xi_{G,h,K}$ involved in Assumption \ref{approximation_error} are given below:
 \begin{align*}
    \forall R_h\in \tilde{\RR}_h, :\min_{f\in \cF}\big\|f(z_h)-T\big(R_h-R_h^0\big)(z_h)\big\|_2&\le \xi_{r,h,K}=\cO(K^{-\frac{\alpha}{2\alpha+d}}),\\\forall  j\in [d_1], \forall G_{h,j}\in \tilde{\mathbb{G}}_{h,j}, \min_{f\in \cF}\big\|f(z_h)-T\big(G_{h,j}-G_{h,j}^0\big)(z_h)\big\|_2&\le\xi_{G,h,K}=\cO(K^{-\frac{\alpha}{2\alpha+d}}),
\end{align*}
 by Theorem 1 of \cite{yarotsky2017error}. Here the associated this probability measure is the sample distribution $\rho.$

Putting all pieces into Theorem \ref{sub_optimality_approx}, we have the following Corollary \ref{cor:nn}, which quantifies the suboptimality under misspecified function classes of neural networks.

\begin{corollary}\label{cor:nn}
Under Assumptions \ref{assume:concentrability}, \ref{ill-pose} (with replacing $R_h^*,G_{h,j}^*$ by $R_h^0$ and $G_{h,j}^0,$ $\forall h\in [H],j\in [d_1]$), \ref{assump:appro_nn} and \ref{ass:proj_nn}.  By constructing  our policy $\hat\pi$ by \algo, with probability $1-\delta-1/K,$ we obtain
\begin{align*}
\textrm{SubOpt}(\hat\pi)&\lesssim C_{\pi^*}\Bigg[\Bigg(H\sum_{h=1}^{H}\tau_{G,h,K}\sqrt{d_1}L_{K,d_1}+\sum_{h=1}^{H}\tau_{r,h,K} L_{K,1}\bigg)\cdot\bigg(K^{-\frac{\alpha}{2\alpha+d}}+\sqrt{\frac{\log(d_1/\delta)}{K}}\Bigg).
\Bigg]
\end{align*}
\end{corollary}
\subsection{Proof of Corollary \ref{cor:nn}}\label{subsec:proofnn}
By Theorem \ref{sub_optimality_approx}, we observe that the upper bound only involves approximation errors $\xi_{r,h,K},\eta_{r,h,K},\xi_{G,h,K},\eta_{G,h,K}$ and critical radii $\delta_{\cG,h},\delta_{\cR,h}.$ Next, we will specify these terms under the setting of \S\ref{example:nn}.

The approximation error of ReLU neural networks with $\cO(\log K)$ layers and $\cO(K^{\frac{d}{2\alpha+d}})$ parameters to Sobolev ball with order $\alpha$  in $\ell_{\infty}$-norm is $\cO(K^{\frac{-\alpha}{2\alpha+d}})$. This conclusion  follows directly from Theorem 1 of \cite{yarotsky2017error}. 

We next provide an upper bound for the critical radii of $\cF,\tilde{\cG}_{h,j}^{*,M},\tilde{\cR}_h^{*,M}$. Note that $\tilde{\cR}_h^{*,M}$ is contained by
\begin{align*}
    \tilde{\cR}_h^{0,M}:=\Big\{c(R_h-R_h^0)\cdot f_h, R_h\in \tilde{\RR}_h,f_h\in \cF,c\in [0,1] \Big\}.
\end{align*}
Thus, the critical radii of $ \tilde{\cR}_h^{0,M}$ will give an upper bound of the  critical radii of $\tilde{\cR}_h^{*,M}$.
We next use CR$(\cH)$ to represent the critical radius of a function class $\cH$. Following the same way of deriving \eqref{linear_delta}, we have the CR$(\tilde{\cR}_h^{0,M})\le \max\{\textrm{CR}(\cF),\textrm{CR}(\tilde{\RR}_h)\}$  in the sense that it is upper bounded by the maximum critical radii of $\cF$ and $\tilde{\RR}_h.$
Similar situation also applies to $\tilde{\cG}_{h,j}^{*,M},$  i.e., CR$(\tilde{\cG}_{h,j}^{*,M})\le \max\{\textrm{CR}(\cF),\textrm{CR}(\tilde{\mathbb{G}}_{h,j})\}$. In addition, the critical radii of function classes $\cF,\tilde{\RR}_h,\tilde{\mathbb{G}}_{h,j}$ which are defined in \S\ref{example:nn}, are of order $\cO(K^{\frac{-\alpha}{2\alpha+d}})$ by the proof of Corollary 6.6 in \cite{uehara2021finite}. 

Finally, plugging in these values into the upper bound of Theorem \ref{sub_optimality_approx}, we conclude our proof of Corollary \ref{cor:nn}.
 

\section{Proof of Applications in \S\ref{application}}
\subsection{Proof of Proposition \ref{prop_reg}}
\begin{proof}
The proof of Proposition \ref{prop_reg} is a special case of the proof of Corollary \ref{example_linear}. Here we only consider strategic regression problem with time horizon $H=1$ and no state transition. Thus, we only need to replace plug in $m_h,n=d$ into the critical radii of linear function classes set $H=1$ in Corollary \ref{example_linear}. Putting all pieces together, we then finally conclude our proof of Proposition \ref{prop_reg}. 
\end{proof}
\subsection{Proof of Proposition \ref{Prop_bandit}}
\begin{proof} 
In this subsection, we will provide our proof for Proposition \ref{Prop_bandit}. First, with high probability, we have $r^*\in \cR$ by following a similar proof of Lemma \ref{contain_true}. 
By our construction of $\hat\pi$ using pessimism, we further have
\begin{align*}
    \textrm{SubOpt}(\hat\pi)&= J(R^*,\pi^*)-J(R^*,\hat\pi)
    \\&\le J(R^*,\pi^*)-\min_{R\in \cR}J(R,\pi^*)+\min_{R\in \cR}J(R,\hat\pi)-J(R^*,\hat\pi)\\&\le J(R^*,\pi^*)-\min_{R\in\cR}J(R,\pi^*)=:\mathbf{(i)}.
\end{align*}
Here the first inequality holds by the definition of $\hat\pi.$ The second inequality holds by pessimism. We denote $\tilde{R}=\argmin_{R\in \cR}J(R,\pi^*)$. We next provide an upper bound for $\mathbf{(i)}.$
\begin{align*}
    \mathbf{(i)}=\int R^*(o_t,a_{2t})-\tilde{R}(o_t,a_{2t})\ud \pi^*(a_{2t}\given o_t)\ud F(o_t\given a_{1t})\ud \pi^*(a_{1t}).
\end{align*}
Here $F(o_t\given a_{1t})=\int_{i_t}P(o_t\given a_{1t},i_t)\ud P(i_t)$ is the conditional distribution of $o_t$ given $a_{1t}.$
According to Cauchy-Schwartz inequality, we next have
\begin{align*}
    \mathbf{(i)}&\le \bigg(\int \big[R^*(o_t,a_{2t})-\tilde{R}(o_t,a_{2t})\big]^2\ud \pi^*(a_{2t}\given o_t)\ud F(o_t\given a_{1t})\ud \pi^*(a_{1t}) \bigg)^{1/2}\\&\le C_{\pi^*}\bigg(\int \big[R^*(o_t,a_{2t})-\tilde{R}(o_t,a_{2t})\big]^2\ud \rho(a_{1t},o_t,a_{2t}) \bigg)^{1/2}\\&\le C_{\pi^*}\tau_1\bigg(\int\bigg(\int [R^*(o_t,a_{2t})-\tilde{R}(o_t,a_{2t})]\ud \rho(o_t,a_{2t}\given a_{1t})\bigg)^2\ud \rho(a_{1t})\bigg)^{1/2}\\&= C_{\pi^*}\tau_{1}\big\|T\big(\tilde{R}-R^*\big)\big\|_2\lesssim C_{\pi^*}\tau_{1}L_{K,1}\delta_{\cR,\delta}.
\end{align*}
 The second inequality follows from the concentrability in Assumption \ref{bandit_concent} and the third inequality follows from the ill-posed condition in \S\ref{strategic_bandit}. Moreover, the last inequality follows by using similar proof procedure of Lemma \ref{project_MSE}, so we omit the details here.
It is worth to note that here $L_{K}=L+\sigma\sqrt{\log (K)}$ with $L$ being the upper bound of all functions in $\RR_1$. Moreover, we have $\delta_{\cR,\delta}=\delta_{\cR}+c_1\sqrt{\log(c_0/\delta)/K}$ with $ \delta_{\cR}$ being the upper bound of the critical radii of $\cF$ and $\cR^*:=\{c(R(a_{2},o)-R^*(a_{2},o))\cdot\EE[R(a_{2},o)-R^*(a_{2},o)\given z=a_1],\forall r\in \mathbb{R}_1,\forall c\in [0,1] \}.$
\end{proof}

\subsection{Noncompliant Agents in Recommendation System}\label{no-compliant}
This subsection discusses our application to the recommendation system with noncompliant agents, which is the example we describe in \S\ref{sec:algo}. To avoid redundancy, in this section, we only introduce the mathematical formulation of this model.
\begin{itemize}
    \item At stage $h\in[H]$, the principal first announces a recommendation $a_h$.
    \item The (noncompliant) myopic agent with private type $i_h\sim P_h(\cdot)$ takes an action $b$ that maximizes its immediate reward (utility) $R_{ah}^*(\cdot)$ given the suggested recommendation $a_h$, state variable $s_h$ and its private type $i_h$, namely, $$b_h=\argmax_{b}R_{ah}^*(a_{h},s_h,i_h,b).$$
    \item The principal finally observes the $o_h=b_h,$ which is the agent's actual choice.
    \item The principal receives a reward $r_h=R_h^*(s_h,b_h)+g_h,$ where $g_h=f_{1h}(i_h)+\epsilon_h$ with $\epsilon_h$ being subGaussian and independent of all other random variables.
    \item The system transits to next state variable $s_{h+1}\sim G_h(s_h,b_h)+\xi_h.$ We assume $\xi_h=f_{2h}(i_h)+\eta_h,$ where $\eta_h\sim N(0,\sigma^2\II)$ and is independent of all other random variables.
\end{itemize}

We observe that such a setting is  a special case of strategic MDP defined in \S\ref{model} with $o_h=b_h,$ for all $h\in[H]$.  Likewise, we focus on a target population of agents and for any given $h\in [H],$ we define a new MDP
which marginalizes the effect of strategic agents. In this scenario, we assume the marginal distribution of $i_h\sim P_h(\cdot),$ $R^*_{ah}(\cdot)$ and $f_{1h}(i_h),f_{2h}(i_h)$ are known in the planning stage. We then define  the new (marginalized) true reward function and transition distribution as follows:
\begin{align*}
    \bar{r}^*_h(s_h,a_h)&=\int_{i_h}\Big[R_h^*(s_h,b_h)+f_{1h}(i_h)\Big]\ud P_h(i_h),\\ 
     \bar{\PP}^*_h(\cdot\given s_h,a_h)&=\int_{i_h} \PP_h^*(\cdot\given s_h,b_h,i_h)\ud P(i_h),
\end{align*}
where $b_h=\argmax_{b}R_{ah}^*(a_{h},s_h,i_h,b)$ and $\PP_h^*(\cdot\given s_h,b_h,i_h)$ is $N(G_h(s_h,b_h)+f_{2h}(i_h),\sigma^2\II)$. Thus, the true model is given as $\{\bar{R}_h^*(s_h,a_h), \bar{\PP}^*_h(\cdot\given s_h,a_h)\}_{h=1}^{H},$
in the planning stage. Based on the defined reward and transition functions, we can determine our value function, Q-function, and Bellman equation in the same way as in \S\ref{planning}.

Next, we discuss the data collecting process. We sample $K$ trajectories 
$\{s_{h}^{(k)},a_{h}^{(k)},b_{h}^{(k)},r_h^{(k)}\}_{h=1,k=1}^{H,K}$ independently of $\{(s_{h}^{(k)},a_{h}^{(k)},b_{h}^{(k)})\}_{h=1}^{H}$ following a joint distribution $\rho: \{(\cS_h\times \cA_h\times \cB_h)\}_{h=1}^{H}\rightarrow \RR.$ 
We then optimize the policy based on our collected dataset using \algo. In specific, we obtain
 \begin{align}\label{est_policy_3}
    \hat\pi=\argmax_{\pi \in \Pi}\min_{M=\{(\bar{\PP}_h,\bar{R}_h)\}_{h=1}^{H}\in\{(\bar{\cG}_h,\bar{\cR}_h)\}_{h=1}^{H}}J(M,\pi).
\end{align}
Here $\bar{\cG}_h$ and $\bar{\cR}_h,\forall h\in[H]$ are confidence sets that contain $\bar{\PP}_h^*,\bar{R}_h^*,\forall h\in [H]$ constructed via our offline dataset $\{(s_{h}^{(k)},a_h^{(k)},b_h^{(k)},r_h^{(k)})\}_{k=1,h=1}^{K,H}$ with high probability. The procedure is almost the same with that in \S\ref{const_level_set},  the only difference is to replace $(a_h^{(k)},o_h^{(k)})$ by $b_h^{(k)}.$  Thus, we omit the details here.
Meanwhile, those needed assumptions are also similar with those in \S\ref{sec:well-specify}.
Therefore, we finally summarize the suboptimality of $\hat\pi$ in this application in the following Proposition \ref{sub_optimality_app3}.
\begin{proposition}
\label{sub_optimality_app3}
Under Assumptions \ref{assume:concentrability}, \ref{ill-pose} and \ref{assume:functionclass} (only replacing $(a_h,o_h)$ by $b_h$), with probability $1-\delta-1/K$, the suboptimality is upper bounded by $$\textrm{SubOpt}(\hat\pi)\lesssim C_{\pi^*}\Bigg[H\sum_{h=1}^{H}\tau_{G,{h},K}\sqrt{d_1}L_{K,d_1}\delta_{\cG,h,\delta}+\sum_{h=1}^{H}\tau_{r,h,K}L_{K,1}\delta_{\cR,h,\delta}\Bigg],$$ where $L_{K,d_1}=L+\sigma\sqrt{(\log H+1)\log (Kd_1)},$
$\delta_{\cR,h,\delta}=\delta_{\cR,h}+c_1\sqrt{\log(c_0H/\delta)/K},\delta_{\cG,h,\delta}=\delta_{\cG,h}+c_2\sqrt{\log(c_0Hd_1/\delta)/K}$ with $c_0,c_1,c_2$ being absolute constants. Here we let $\delta_{\cG,h}$ and $\delta_{\cR,h}$ be the upper bounds of the maximum critical radii of $\cF,\cG^{*}_{h,j}$ and $\cF,\cR_h^*,$ (with $\cF,\cG^{*}_{h,j},\cR_h^*$ being defined in \S\ref{sec:well-specify} by replacing $(a_h,o_h)$ with $b_h$).
\end{proposition}
\subsection{Proof of Proposition \ref{sub_optimality_app3}}
In this subsection, we provide our proof of Proposition \ref{sub_optimality_app3}. Like our proof for Theorem \ref{sub_optimality}, we are also able to prove Lemma \ref{contain_true}, \ref{Pessmistic}, \ref{project_MSE}, only by replacing $(s_h,o_h,a_h)$ by $(s_h,b_h).$ We let $\tilde{M}:=\{(\bar{\tilde{R}}_h,\bar{\tilde{\PP}}_h)\}_{h=1}^H.$ Leveraging these conclusions, we have
\begin{align}
   &  \EE_{a_h\sim \pi^*(\cdot \given s_h),s_h\sim d^{\pi^*}_{h-1}}\bigg[\Big\|\Big(\bar{\tilde{ \PP}}_h-\PP_h^*\Big)(\cdot\given s_h,a_h)\Big\|_1\bigg]\notag \\
   &\qquad= \EE_{a_h\sim \pi^*(\cdot \given s_h),s_h\sim d^{\pi^*}_{h-1}}\Bigg[\bigg\|\int_{i_h} \tilde{\PP}(\cdot\given s_h,b_h,i_h)- \PP^*(\cdot\given s_h,b_h,i_h)\ud P_h(i_h)\bigg\|_1 \Bigg]\notag\\&\qquad \le \EE_{a_h\sim \pi^*(\cdot \given s_h),s_h\sim d^{\pi^*}_{h-1}}\Bigg[\int_{i_h}\Big\|\ \tilde{\PP}(\cdot\given s_h,b_h,i_h)- \PP^*(\cdot\given s_h,b_h,i_h)\Big\|_1\ud P_h(i_h) \Bigg] \notag
   \\& \qquad \le\sqrt{\EE_{a_h\sim \pi^*(\cdot \given s_h), i_h\sim P_h(\cdot), s_h\sim d^{\pi^*}_{h-1}}\bigg[\textrm{TV}\Big(\tilde{\PP}(\cdot\given s_h,b_h,i_h),\PP^*(\cdot\given s_h,b_h,i_h)\Big)^2\bigg]}\label{proof_gp}
   \end{align}
 The first inequality follows form Jensen's inequality. The second inequality follows from Cauchy-Schwartz inequality. Furthermore, we obtain 
   \begin{align}
   \eqref{proof_gp} & \lesssim\sqrt{\EE_{a_h\sim \pi^*(\cdot \given s_h), b_h\sim P_h(\cdot\given s_h,a_h), s_h\sim d^{\pi^*}_{h-1}}\bigg[\Big\|\tilde{G}_h(s_h,b_h)-G_h^*( s_h,b_h)\Big\|_2^2\bigg]}\notag
    \\& \lesssim C_{\pi^*} \sqrt{\EE_{\rho(s_h,b_h)}\bigg[\Big\|\tilde{G}_h( s_h,b_h)-G_h^*(s_h,b_h)\Big\|_2^2\bigg]}\notag
    \\ & \lesssim C_{\pi^*} \tau_{K,G,{h}}\sqrt{\EE_{\rho(s_h,a_h)}\bigg[\Big\|\EE_{\rho(b_h\given s_h,a_h)}\big[\tilde{G}_h(s_h,b_h)-G_h^{*}(s_h,b_h) \given s_h,a_h\big]\Big\|_2^2\bigg]}\notag
    \\ &\lesssim C_{\pi^*} \tau_{K,G,{h}}\sqrt{d_1}L_{K,\beta,d_1}\delta_{h,\cG,\delta}.\label{proof_gp}
\end{align}
Here $P_h(\cdot\given s_h,a_h)$ is the distribution of $b_h=\argmax R_{ah}(s_h,a_h,i_h)$ given $s_h,a_h,R_{ah}(\cdot)$ and the distribution of $i_h\sim P_h(\cdot).$
 The first inequality follows from our model assumption on Gaussian transition given $(s_h,b_h,i_h)$. 
Moreover, the second inequality follows from our Assumption \ref{assume:concentrability}. Meanwhile, the third inequality follows from ill-posed condition in Assumption \ref{ill-pose}. Finally, similar to the derivation of \eqref{radius_g}, we obtain last inequality.
Likewise, we also obtain
\begin{align}
    &\EE_{a_h\sim \pi^*(\cdot \given s_h),s_h\sim d^{\pi^*}_{h-1}}\bigg[\Big|\bar{\tilde{ R}}_h(s_h,a_h)-\bar{R}_h^{*}(s_h,a_h)\Big| \bigg]\notag\\ &\qquad \le
    \EE_{a_h\sim \pi^*(\cdot \given s_h),s_h\sim d^{\pi^*}_{h-1}}\bigg[\int_{i_h}\big|(\tilde{R}_h-{R}_h^*)(s_h,b_h)\big|\ud P_h(i_h)\bigg]\notag
    \\&\qquad\lesssim \sqrt{\EE_{a_h\sim \pi^*(\cdot \given s_h),s_h\sim d^{\pi^*}_{h-1},b_h\sim \PP_h(\cdot\given s_h,a_h)}\Big[\big|\tilde{ R}_h(s_h,b_h)-R_h^{*}(s_h,b_h)\big|^2 \Big]}\notag\\&
    \qquad\lesssim C_{\pi^*}\sqrt{\EE_{\rho(s_h,b_h)}\Big[\Big|(\tilde{R}_h-{R}_h^*)(s_h,b_h)\Big|^2\Big]}\notag
    \\&\qquad \lesssim C_{\pi^*}\tau_{K,r,h}\sqrt{\EE_{\rho(s_h,a_h)}\bigg[\bigg(\EE\Big[\tilde{R}_h(s_h,b_h)-R_h^{*}(s_h,b_h) \given s_h,a_h\Big]\bigg)^2\bigg]}\notag\\&\qquad \lesssim C_{\pi^*}\tau_{K,r,h}L_{K,\beta,1}\delta_{h,\cR,\delta}.\notag
\end{align}
Here, the first and second inequalities follow from Jensen's inequality and Cauchy-Schwartz inequalities, respectively. The third and fourth inequalities   follow from Assumptions \ref{assume:concentrability} and  \ref{ill-pose}, respectively.  The last inequality holds by a similar derivation of \eqref{radius_r}.
Finally, we have
\begin{align*}
   \textrm{SubOpt}(\hat\pi)\le C_{\pi^*}\Bigg[H\sum_{h=1}^{H}\tau_{K,G,h}\sqrt{d_1}L_{K,\beta,d_1}\delta_{h,\cG,\delta}+\sum_{h=1}^{H}\tau_{K,r,h}L_{K,\beta,1}\delta_{h,\cR,\delta}\Bigg].
\end{align*}
Thus, we conclude our proof for Proposition \ref{sub_optimality_app3}.

\section{Proof of Technical Lemmas}\label{auxiliary}
In this subsection, we will prove some technical Lemmas which are used by us in proving our main theorems. First, we prove Lemma \ref{conc_phi}.

\subsection{Proof of Lemma \ref{conc_phi}}
First, we prove $\ell(R_h,f)=(R_h(x_h^{(k)})-R_h^{*}(x_h^{(k)})+g_{h}^{(k)})f(z_h^{(k)})$ 
is $L_{K,\beta,1}$-Lipschitz in $f,\forall f\in \cF$, with probability $1-1/K^{\beta}$. Here we define $L_{K,\beta,d_1}:=L+\sigma\sqrt{(\beta+1)\log Kd_1}$. 

First, we have $R_h(\cdot),R_h^{*}(\cdot)$ are bounded functions with upper bound $L$. Since we assume $g_h$ is a { subGaussian random variable} with variance proxy $\sigma.$ Then we have $\max_{k\in [K]}|g_h^{(k)}|\le \sigma\sqrt{(\beta+1)\log(K)}$ with probability $1-1/K^{\beta}$ by union bound of subGaussian variables \citep{vershynin2018high}. We next provide a upper bound for $|\frac{1}{K}\sum_{k=1}^{K}g_h^{(k)}f(z_h^{(k)})|. $ For any $t_f\ge 0,$ we obtain 
\begin{align*}
   &\PP\Bigg( \exists f\in\cF, \Bigg|\frac{1}{K}\sum_{k=1}^{K}g_h^{(k)}f(z_h^{(k)})-0\Bigg|\ge t_f 
   \Bigg)\\&\qquad\le\PP\Bigg(\exists f\in \cF,\Bigg|\frac{1}{K}\sum_{k=1}^{K}g_h^{(k)}\II_{\{|g_h^{(k)}|\le \sigma\sqrt{(\beta+1)\log(K)}\}}f\Big(z_h^{(k)}\Big)-\EE\Big[g_hf(z_h)\II_{\{|g_h|\le \sigma\sqrt{(\beta+1)\log(K)}\}}\Big]\Bigg|\ge t_f/2 \Bigg)\\&\qquad\qquad+\PP\Bigg(\frac{1}{K}\sum_{k=1}^{K}g_h^{(k)}\II_{\{|g_h^{(k)}|\le \sigma\sqrt{(\beta+1)\log(K)}\}}f\Big(z_h^{(k)}\Big)\neq \frac{1}{K}\sum_{k=1}^{K}g_h^{(k)}f\Big(z_h^{(k)}\Big)\Bigg) \\&\qquad\qquad+\PP\bigg(\exists f\in \cF, \Big|\EE\Big[g_hf(z_h)\II_{\{|g_h|\ge \sigma\sqrt{(\beta+1)\log(K)}\}}\Big]\Big|\ge t_f/2\bigg):=\mathbf{(I)}+\mathbf{(II)}+\mathbf{(III)}.
\end{align*}
We have $\mathbf{(II)}\le 1/K^{\beta}$. For term $\mathbf{(III)},$ by Cauchy-Schartz inequality, we have
\begin{align*}
 \bigg|\EE\Big[g_hf(z_h)\II_{\{|g_h|\ge \sigma\sqrt{(\beta+1)\log(K)}\}}\Big]\bigg|\le \sigma\|f\|_2\sqrt{\PP(|g_h|\ge \sigma\sqrt{(\beta+1)\log(K)})}=\sigma\|f\|_2/K^{\beta+1}.
\end{align*}
Thus, when we choose $t_f=2C_1\sigma\sqrt{(\beta+1)\log K}(\delta_{\cR,h,\delta}\|f\|_2+\delta_{\cR,h,\delta}^2),$ we have $\mathbf{(III)}=0$ when we choose $\beta>0$ properly such that $1/K^{\beta}=o(\delta_{\cR,h,\delta}^2).$

Finally, by Lemma 11 in \cite{Foster2019}, we have $\mathbf{(I)}\le \delta$ for such $t_f$.
Thus, with probability $1-\delta-1/K^{\beta},$ we have
\begin{align}\label{foster_1}
   \forall f\in \cF, \Bigg|\frac{1}{K}\sum_{k=1}^{K}g_h^{(k)}f(z_h^{(k)})\Bigg|\lesssim \sigma\sqrt{(\beta+1)\log K}(\delta_{\cR,h,\delta}\|f\|_2+\delta_{\cR,h,\delta}^2).
\end{align}
Here $\delta_{\cR,h,\delta}=\delta_{\cR,h}+c_1\sqrt{\log(c_0/\delta)/K},$ with $c_0,c_1$ being absolute constants and $\delta_{\cR,h}$ being an upper bound of the maximum critical radii of $\cF.$
Next, we provide an upper bound
In addition, since $R_h,R_h^*$ are bounded functions, we obtain with probability $1-\delta,$
\begin{align}\label{foster_2}
  \forall f\in \cF,  \Bigg|\frac{1}{K}\sum_{k=1}^{K}\Big(R_h(x_h^{(k)})-R_h^*(x_h^{(k)})\Big)f\Big(z_h^{(k)}\Big)-\Phi(R_h^*,f)\Bigg|\lesssim L(\delta_{\cR,h,\delta}\|f\|_2+\delta^2_{\cR,h,\delta}),
\end{align} following Lemma 11 in \cite{Foster2019}.

Combining \eqref{foster_1} and \eqref{foster_2}, we conclude our proof of Lemma \ref{conc_phi}.

\subsection{Proof of Lemma \ref{conc_rf}}
 Here our loss function $\ell(a,b)=a,$ with $a=(R_h-R_h^*)(x)\mathbb{T}(R_h-R_h^*)(z)\in \cR_h^*$. In this scenario, we have $\ell(a,b)$ is Lipschitz continuous in $a$. In addition, $\cR_h^*$ is star-shaped and contains bounded functions. Then the conclusion of Lemma \ref{conc_rf} follows directly from Lemma 11 of \cite{Foster2019}.

 \section{Auxillary Lemma}\label{auxillary_lem}
 
 \begin{lemma}[Corollary 14.5 of \cite{wainwright_2019}] Let $H=\{h\in \cH:\|f\|_{\cH}\le C\}$ be a bounded ball of an RKHS with eigenvalues $\{\lambda_i\}_{i=1}^{\infty}.$ Then the localized population Rademacher complexity $\cR_K(\delta,H)$ is bounded as
 \begin{align*}
     \cR_K(\delta,H)\le \sqrt{\frac{2}{K}\cdot \sum_{j=1}^{\infty}\min\{\lambda_j,C^2\delta^2\} }.
 \end{align*}
 Here $C$ is an absolute constant.
 \end{lemma}
 \begin{proof}
 See the proof of Corollary 14.5 in \cite{wainwright_2019} for more details.
 \end{proof}
 
 \begin{lemma}[Theorem 1 of \cite{yarotsky2017error}]
There exist a class of ReLU neural networks with depth at most $\cO(\ln(1/\epsilon))$ and $\cO(\epsilon^{-d/\alpha}\ln(1/\epsilon))$ weights and computation units, that approximate all $f$ in Sobolev ball with order $\alpha$ and input dimension $d$ in $\ell_{\infty}$-norm within error $\epsilon.$
 \end{lemma}
 \begin{proof}
 See the proof of Theorem 1 in \cite{yarotsky2017error} for more details.
 \end{proof}
 
 \begin{lemma}[Lemma 11 in \cite{Foster2019}] Assume $\sup_{f\in\cF}\|f\|_{\infty}\le c$ for a constant $c$ and for any $f^*\in \cF$. We define $\delta_K$ as the solution to 
 \begin{align*}
     \cR_K(\delta,\textrm{star}({\cF-f^*}))\le \delta^2/c.
 \end{align*}
 Here the star$(\cdot)$ denotes the star-hull of a function class. Moreover, we assume the loss function $\ell(\cdot,\cdot)$ is $L$-Lipschitz in the first argument. The with probability $1-x,$ for all $f\in\cF$, we have
 \begin{align*}
     \bigg|\frac{1}{K}\sum_{k=1}^K\ell(f(x_k),z_k)-\frac{1}{K}\sum_{k=1}^K\ell(f^*(x_k),z_k)-(\EE[\ell(f(x),z)]-\EE[\ell(f^*(x),z)])\bigg|\le L\delta_{K,x}(\|f-f^*\|_2+\delta_{K,x}).
 \end{align*}
 Here $\delta_{K,x}=\delta_K+\sqrt{\log(1/x)/K}.$
 \end{lemma}
 \begin{proof}
 See the proof of Lemma 11 in \cite{Foster2019} for more details.
 \end{proof}

\newpage
\bibliographystyle{ims}
\bibliography{dynamic}
\end{document}